\documentclass[11pt]{article}

\usepackage{natbib}
\bibpunct{(}{)}{;}{a}{,}{,}

\newcommand{\BlackBox}{\rule{1.5ex}{1.5ex}}  
\newenvironment{proof}{\par\noindent{\bf Proof\ }}{\hfill\BlackBox\\[2mm]}
 
\newtheorem{theorem}{Theorem}
\newtheorem{lemma}[theorem]{Lemma} 
\newtheorem{proposition}[theorem]{Proposition} 
\newtheorem{remark}[theorem]{Remark}
\newtheorem{corollary}[theorem]{Corollary}
\newtheorem{definition}[theorem]{Definition}

\long\def\acks#1{\vskip 0.3in\noindent{\large\bf Acknowledgments}\vskip 0.2in
\noindent #1}

\usepackage[margin=1in]{geometry}

\usepackage{amsmath}
\usepackage{amsfonts}
\usepackage{amssymb}
\usepackage{graphicx}
\usepackage[T1]{fontenc}
\usepackage{color}
\usepackage{algorithmic, algorithm}
\usepackage{upgreek}
\usepackage[colorlinks]{hyperref}
\usepackage{hypernat}
\definecolor{MyDarkBlue}{rgb}{0.04,0.12,0.60}
\definecolor{MyRed}{rgb}{0.8,0.4,0.2}
\definecolor{Sienna}{RGB}{160,82,45}
\definecolor{RawSienna}{RGB}{199,97,20}
\definecolor{Coral3}{RGB}{205,91,69}
\definecolor{PaleBlue}{RGB}{102,178,255}
\hypersetup{
    colorlinks,%
    citecolor=MyDarkBlue,%
    filecolor=MyDarkBlue,%
    linkcolor=RawSienna,%
    urlcolor=MyDarkBlue
}

\def\BEAS{\begin{eqnarray*}}
\def\EEAS{\end{eqnarray*}}
\def\BEA{\begin{eqnarray}}
\def\EEA{\end{eqnarray}}

\newcommand{\OMIT}[1]{}
\newcommand{\er}[1]{\textcolor{orange}{\em [ER :  #1]}}
\newcommand{\GO}[1]{\textcolor{blue}{\em [GO :  #1]}}

\def\st{\text{s.t.}}
\def\ie{{\it i.e.}}

\def\etaij{{\eta^{\scriptscriptstyle (IJ)}}} 
\def\aij{{Z^{\scriptscriptstyle (IJ)}}} 

\newcommand {\br}[1]{\left(#1\right)}
\newcommand {\sqb}[1]{\left[#1\right]}
\newcommand {\cbr}[1]{\left\{#1 \right\}}

\newcommand {\nm}[1]{\left\|#1 \right\|}
\newcommand {\nmop}[1]{\left\|#1 \right\|_{\mathrm{op}}} 
\newcommand {\nmtr}[1]{\left\|#1 \right\|_*} 
\newcommand {\nmF}[1]{\left\|#1 \right\|_{\mathrm{Fro}}} 

\global\long\def\argmin{\operatorname*{arg\, min}}

\newcommand{\rank}{\text{rank}}
\newcommand{\kqrank}{(k,q)\mbox{-}\text{rank}}

\newcommand{\kqsvd}{(k,q)\mbox{-}\text{SVD}}

\newcommand{\kqcut}{(k,q)\mbox{-}\text{CUT}}
\newcommand{\kqcutn}{(k,q)\mbox{-}\text{CUT norm}}

\newcommand{\kqtn}{(k,q)\mbox{-}\text{trace norm}}
\newcommand{\ksuppn}{k\mbox{-}\text{support norm}}

\global\long\def\argmin{\operatorname*{arg\, min}}

\global\long\def\sign{\operatorname{sign}}
\global\long\def\sgn{\operatorname{\textrm{sgn}}}

\global\long\def\trace{\operatorname{Tr}}
\newcommand{\RR}{\mathbb{R}}
\newcommand{\R}{\mathbb{R}}
\newcommand{\NN}{\mathbb{N}}

\newcommand{\Span}{\text{span}}
\newcommand{\Supp}{\text{supp}}
\newcommand{\op}{\mathrm{op}}
\newcommand{\PP}{\mathcal{P}}

\newcommand{\GG}{\mathcal{G}}

\newcommand{\inr}[1]{\langle #1 \rangle}

\newcommand{\trans}{^{\scriptscriptstyle \top}}
\newcommand{\norm}[1]{\|#1\|}

\newcommand{\cX}{\mathcal X}

\DeclareMathOperator{\tr}{tr}
\newcommand{\E}{\mathbb E}

\newcommand{\cA}{\mathcal A}
\newcommand{\cAz}{\widetilde{\cA}}

\renewcommand{\P}{\mathbb P}

\usepackage{array}
\newcolumntype{M}[1]{>{\raggedright}m{#1}}
\usepackage{amssymb}
\usepackage{mathrsfs}
\usepackage{algorithmic, algorithm}
\newcommand{\conv}{\text{conv}}

\newcommand{\dist}{\text{dist}}

\def\RR{\mathbb{R}}

\def\tr{{\rm tr}}
\def\op{{\rm op}}

\def\Gcal{\mathcal{G}}
\def\Dcal{\mathcal{D}}
\def\ones{\mathbf{1}}

\def\Xk{X_{(k)}}

\def\BEAS{\begin{eqnarray*}}
\def\EEAS{\end{eqnarray*}}
\def\BEA{\begin{eqnarray}}
\def\EEA{\end{eqnarray}}
\def\BIT{\begin{itemize}}
\def\EIT{\end{itemize}}

\def\supp{{\rm supp}}
\def\and{\quad \text{and} \quad }
\def\BSM{\left ( \begin{smallmatrix}}
\def\ESM{\end{smallmatrix} \right )}

\def\risk{ \mathcal L}

\def\Scal{\mathcal{S}}

\def\Okq{\Omega_{k,q}}
\def\Ozkq{\widetilde{\Omega}_{k,q}}
\def\O{\mathcal{O}}

\def\Akq{\mathcal{A}_{k,q}}
\def\Azkq{\widetilde{\mathcal{A}}_{k,q}}

\def\sdim{\mathfrak{S}}

\newcommand{\pspan}[3]{\Pi_{#1,#2,#3}}
\newcommand{\pspz}{\pspan{A}{I_0}{J_0}}
\newcommand{\pspzorth}{\pspz^{\bot}}
\newcommand{\pspXIJ}{\pspan{A}{I}{J}}
\newcommand{\pspXIJorth}{\pspan{A}{I}{J}^{\bot}}

\newcommand{\PA}{\mathcal{P}_A}

\def\tG{\tilde{G}}
\def\eG{\epsilon(G)}
\def\id{{\rm Id}}

\def\io{{I_0}}
\def\jo{{J_0}}
\def\iio{{I \cap I_0}}
\def\jjo{{J \cap J_0}}
\def\idio{\id_\io}
\def\idjo{\id_\jo}
\def\iomi{{\io \backslash I}}
\def\jomj{{\jo \backslash J}}
\def\imio{{I \backslash \io}}
\def\jmjo{{J \backslash \jo}}

\newcommand{\Gk}{\mathcal{G}_k^{m_1}}
\newcommand{\Gq}{\mathcal{G}_q^{m_2}}

\def\aij{{Z^{\scriptscriptstyle (IJ)}}} 
\def\uij{{U^{\scriptscriptstyle (IJ)}}} 
\def\vij{{V^{\scriptscriptstyle (IJ)}}} 
\def\sij{{\Sigma^{\scriptscriptstyle (IJ)}}} 

\def\eqdef{{\,:=\,}}

\begin{document}

\title{Tight convex relaxations for sparse matrix factorization}



\author{Emile Richard$^{1}$, Guillaume Obozinski$^{2}$  and Jean-Philippe Vert$^{3,4,5}$\\ \\
$^1$Department of Electrical Engineering, Stanford University \\
$^2$Universit\'e Paris-Est, Laboratoire d'Informatique Gaspard Monge,\\ Groupe Imagine, Ecole des Ponts - ParisTech, 77455 Marne-la-Vall\'ee, France\\
$^3$Mines ParisTech, PSL Research University, \\CBIO-Centre for Computational Biology, 77300 Fontainebleau, France\\
$^4$Institut Curie, 75248 Paris Cedex ,France\\
$^5$INSERM U900, 75248 Paris Cedex ,France
}

\date{}
\maketitle

\begin{abstract}
Based on a new atomic norm, we propose a new convex formulation for sparse matrix factorization problems in which the number of nonzero elements of the factors is assumed fixed and known. The formulation
counts sparse PCA with multiple factors, subspace clustering and low-rank sparse bilinear regression as potential applications. We compute slow rates and an upper bound on the statistical dimension \citep{Amelunxen13} of the suggested norm for rank 1 matrices, showing that its statistical dimension is an order of magnitude smaller than the usual $\ell_1$-norm, trace norm and their combinations.
Even though our convex formulation is in theory hard and does not lead to provably polynomial time algorithmic schemes, we propose an active set algorithm leveraging the structure of the convex problem to solve it and show promising numerical results. 
\end{abstract}


\section{Introduction}

A range of machine learning problems such as link prediction in graphs containing community structure \citep{Richard2014Link}, phase retrieval \citep{candes11}, subspace clustering \citep{wang13} or dictionary learning for sparse coding \citep{mairal2010online} amount to solve sparse matrix factorization problems, \ie, to infer a low-rank matrix that can be factorized as the product of two sparse matrices with few columns (left factor) and few rows (right factor).  Such a factorization allows for more efficient storage, faster computation, more interpretable solutions, and, last but not least, it leads to more accurate estimates in many situations. In the case of interaction networks for example, the assumption that the network is organized as a collection of highly connected communities which can overlap implies that the adjacency matrix admits such a factorization. More generally, considering sparse low-rank matrices combines two natural forms of sparsity, in the spectrum and in the support, which can be motivated by the need to explain systems behaviors by a superposition of latent processes which only involve a few parameters. Landmark applications of sparse matrix factorization are sparse principal components analysis \citep[SPCA,][]{Ghaoui2004, zou2006sparse} or sparse canonical correlation analysis \citep[SCCA,][]{witten2009penalized}, which are widely used to analyze high-dimensional data such as genomic data.

From a computational point of view, however, sparse matrix factorization is challenging since it typically leads to non-convex, NP-hard problems \citep{moghaddam2006spectral}. For instance, \citet{BerRig13} noted that solving sparse PCA with a single component is equivalent to the planted clique problem \citep{Jerrum1992Large}, a notoriously hard problem when the size of the support is smaller than the square root of size of the matrix. Many heuristics and relaxations have therefore been proposed, with and without theoretical guaranties, to approximatively solve the problems leading to sparse low-rank matrices. A popular procedure is to alternatively optimize over the left and right factors in the factorization, formulating each step as a convex optimization problem \citep{Lee2007Efficient,mairal2010online}. Despite these worst case computational hardness,  simple generalizations of the power method have been proposed by \citet{journee10, luss12, yuan13} for the sparse PCA problem with a single component. These algorithms perform well empirically and have been proved to be efficient theoretically under mild conditions by \citet{yuan13}. Several semidefinite programming (SDP) convex relaxations of the same problem have also been proposed \citep{Ghaoui2004,Aspremont2008optimal,Amini09}. Based on the rank one approximate solutions, computing multiple principal components of the data is commonly done though successive deflations \citep{Mackey08} of the input matrix. 


Recently, several authors have investigated the possibility to formulate sparse matrix factorization as a convex optimization problem. \citet{Bach2008Convex} showed that the convex relaxation of a number of natural sparse factorization are too coarse too succeed, while \citet{Bach2013Convex} investigated several convex formulations involving nuclear norms \citep{Jameson1987Summing}, similar to the ones we investigate in this paper, and their SDP relaxations. Several authors also investigated the performance of regularizing a convex loss with linear combinations of the $\ell_1$ norm and the trace norm, naturally leading to a matrix which is both sparse and low-rank \citep{Richard12,Richard2014Link,Richard13,doan2010finding,Oymak12}. This penalty term can be related to the SDP relaxations of \citet{Ghaoui2004,Aspremont2008optimal} that penalize the trace and the element-wise $\ell_1$ norm of the positive semi-definite unknown. The statistical performance of these basic combinations of the two convex criteria has however been questioned by \citet{Oymak12, Krauthgamer13}. \citet{Oymak12} showed that for compressed sensing applications, no convex combination of the two norms improves over each norm taken alone. \citet{Krauthgamer13} prove that the SDP relaxations fail at finding the sparse principal component outside the favorable regime where a simple diagonal thresholding algorithm \citep{Amini09} works. 
Moreover, these existing convex formulations either aim at finding only a rank one matrix, or a low rank matrix whose factors themselves are not necessarily guaranteed to be sparse. 

In this work, we propose two new matrix norms which, when used as regularizer for various optimization problems, do yield estimates for low-rank matrices with multiple sparse factors that are provably more efficient statistically than the $\ell_1$ and trace norms. The price to pay for this statistical efficiency is that, although convex, the resulting optimization problems are NP-hard, and we must resort to heuristic procedures to solve them. 
Our numerical experiments however confirm that we obtain the desired theoretical gain to estimate low-rank sparse matrices.

\subsection{Contributions and organization of the paper}
More precisely, our contributions are: 
\begin{itemize}
\item {\bf Two new matrix norms (Section~\ref{sec:tighRelaxations}). }In order to properly define matrix factorization, given sparsity levels of the factors denoted by integers $k$ and $q$, we first introduce in Section~\ref{sec:kqrank} the $\kqrank$ of a matrix as the minimum number of left and right factors, having respectively $k$ and $q$ nonzeros, required to reconstruct a matrix. This index is a more involved complexity measure for matrices than the rank in that it conditions on the number of nonzero elements of the left and right factors of a matrix. Using this index, we propose in Section~\ref{sec:convexRelaxkqrank} two new \emph{atomic norms} for matrices \citep{Chandrasekaran12}. ($i$) Considering the convex hull unit operator norm matrices with $\kqrank = 1$, we build a convex surrogate to low $\kqrank$ matrix estimation problem. ($ii$) We introduce a polyhedral norm built upon $\kqrank = 1$ matrices with all non-zero entries of absolute value equal to $1$. We provide in Section~\ref{sec:nuclear} an equivalent characterization of the norms as nuclear norms, in the sense of~\citet{Jameson1987Summing}, highlighting in particular a link to the $k$-support norm of \citet{argyriou12}.
\item {\bf Using these norms to estimate sparse low-rank matrices (Section~\ref{sec:applications}).} We show how several problems such as bilinear regression or sparse PCA can be formulated as convex optimization problems with our new norms, and clarify that the resulting problems can however be NP-hard.
\item {\bf Statistical Analysis (Section~\ref{sec:statisticalproperties}). } We study the statistical performance of the new norms and compare them with existing penalties. Our analysis goes first in Section~\ref{sec:denoisingStatistics} using {\it slow rate} type of upper bounds on the denoising error, which despite sub-optimality gives a first insight on the gap between the statistical performance of our $(k,q)$-trace norm and that of the $\ell_1$ and trace norms. Next we show in Section~\ref{sec:statisticalDimension}, using cone inclusions and estimates of statistical dimension, that our norms are superior to any convex combination of the trace norm and the $\ell_1$ norm in a number of different tasks. However, our analysis also shows that the factors gained over the rivals to estimate sparse low-rank matrices vanishes when we use our norm to estimate sparse vectors.
 \item {\bf A working set algorithm (Section~\ref{sec:algos}). }  While in the vector case the computation remains feasible in polynomial time, the norms we introduce for matrices can not be evaluated in polynomial time. We propose algorithmic schemes to approximately learn with the new norms. The same norms and meta-algorithms can be used as a regularizer in supervised problems such as bilinear and quadratic regression. Our algorithmic contribution does not consist in providing more efficient solutions to the rank-1 SPCA problem, but to combine atoms found by the rank-1 solvers in a principled way. 
 \item {\bf Numerical experiments (Section~\ref{sec:num}).} We numerically evaluate the performance of our new norms on simulated data, and confirm the theoretical results. While our theoretical analysis only focuses on the estimation of sparse matrices with $\kqrank$ one, our simulations allow us to conjecture that the statistical dimension scales linearly with the $\kqrank$ and decays with the overlap between blocks. We also show that our model is competitive with the state-of-the-art on the problem of sparse PCA.
 \end{itemize}
Due to their length and technicality, all proofs are postponed to the appendices.

\subsection{Notations}

For any integers $1\leq k \leq p$, $[1,p]= \cbr{1,\ldots,p}$ is the set of integers from $1$ to $p$ and  $\GG_k^p$ denotes the set of subsets of $k$ indices in $[1,p]$.
For a vector $w\in\RR^p$, $\nm{w}_0$ is the number of non-zero coefficients in $w$, $\nm{w}_1 = \sum_{i=1}^p |w_i|$ is its $\ell_1$ norm, $\nm{w}_2 = \br{\sum_{i=1}^p w_i^2}^{\frac{1}{2}}$ is its Euclidean norm, $\nm{w}_\infty = \max_i |w_i|$ is its $\ell_\infty$ norm and $\Supp(w) \in \GG_{\nm{w}_0}^p$ is its support, \ie, the set of indices of the nonzero entries of $w$. For any $I\subset\sqb{1,p}$, $w_I \in \RR^p$ is the vector that is equal to $w$ on $I$, and has $0$ entries elsewhere. 
Given matrices $A$ and $B$ of the same size, $\inr{A, B} = \tr(A\trans B)$ is the standard inner product of matrices. For any matrix $Z \in \RR^{m_1 \times m_2}$  the notations  $\nm{Z}_0$, $\nm{Z}_1$, $\nm{Z}_\infty$, $\nmF{Z}$, $\nmtr{Z}$, $\nmop{Z}$ and $\rank(Z)$ stand respectively for the number of nonzeros, entry-wise $\ell_1$ and $\ell_\infty$ norms, the standard $\ell_2$ (or Frobenius) norm, the trace-norm (or nuclear norm, the sum of the singular values), the operator norm (the largest singular value) and the rank of $Z$, while $\Supp(Z) \subset  [ 1 ,m_1 ]\times  [ 1 ,m_2 ]$ is the support of $Z$, \ie, the set of indices of nonzero elements of $Z$. When dealing with a matrix $Z$ whose nonzero elements form a block of size $k\times q$, $\Supp(Z)$ takes the form $I \times J$ where $(I,J) \in \GG_k^{m_1} \times \GG_q^{m_2}$.
For a matrix $Z$ and two subsets of indices $I \subset [ 1, m_1 ]$ and $J \subset [ 1,  m_2 ]$, $Z_{I,J}$ is the matrix having the same entries as $Z$ inside the index subset $I \times J$, and $0$ entries outside. This notation should not be confused with the notation $\aij$ which we will sometimes use to denote a general matrix with support contained in $I \times J$. 

\section{Tight convex relaxations of sparse factorization constraints}\label{sec:tighRelaxations}

In this section we propose two new matrix norms allowing to formulate various sparse matrix factorization problems as convex optimization problems. We start by defining the $\kqrank$ of a matrix in Section~\ref{sec:kqrank}, a useful generalization of the rank which also quantifies the sparseness of a matrix factorization. We then introduce two atomic norms defined as tight convex relaxations of the $\kqrank$ in Section \ref{sec:convexRelaxkqrank}: the $\kqtn$, obtained by relaxing the $\kqrank$ over the operator norm ball, and the $\kqcutn$, obtained by a similar construction with extra-constraints on the element-wise $\ell_\infty$ of factors. In Section~\ref{sec:nuclear} we relate these matrix norms to vector norms using the concept of nuclear norms, establishing in particular a connection of the $\kqtn$ for matrices with the $\ksuppn$ of~\citet{argyriou12}, and the $\kqcutn$ to the vector $k$-norm, defined as the sum of the $k$ largest components in absolute value of a vector \citep[Exercise II.1.15]{bhatia97}. 

\subsection{The $\kqrank$ of a matrix}\label{sec:kqrank}

The rank of a matrix $Z\in\RR^{m_1 \times m_2}$ is the minimum number of rank-1 matrices (\ie, outer products of vectors of the form $ab\trans$ for $a\in\RR^{m_1}$ and $b\in\RR^{m_2}$) needed to express $Z$ as a linear combination of the form $Z=\sum_{i=1}^r a_i b_i\trans$. It is a versatile concept in linear algebra, central in particular to solve matrix factorization problems and low-rank approximations. The following definition generalizes this notion to incorporate constraints on the sparseness of the rank-1 elements:
\begin{definition}[$\kqrank$]\label{def:kqrank}
For a matrix $Z \in \RR^{m_1 \times m_2}$, we define its $\kqrank$
as the optimal value  of the optimization problem:
\begin{equation}\label{eq:kqrank}
\min \|c\|_0 \quad \st \quad Z=\sum_{i=1}^\infty c_i a_i b_i\trans, \qquad (a_i,b_i,c_i) \in \cA^{m_1}_k \times \cA^{m_2}_q \times \RR_+\,,
\end{equation}
where for any $1\leq j \leq n$, 
$
\cA^{n}_j:=\cbr{a \in \RR^{n} ~:~ \|a\|_0 \leq j, \|a\|_2=1 }, 
$
that is
$\cA^{n}_j$ is the set of $n$-dimensional unit vectors with at most $j$ non-zero components. 
\end{definition}

When $k=m_1$ and $q=m_2$, we recover the usual notion of rank of a matrix, and a particular solution to \eqref{eq:kqrank} is provided by the SVD, for which the vectors $(a_i)_{1\leq i \leq r}$ and $(b_i)_{1\leq i \leq r}$ form each a collection of orthonormal vectors.

In general, however, the $\kqrank$ does not share several important properties of the usual rank, as the following proposition shows:
\begin{proposition}(Properties of the $\kqrank$ and associated decompositions)\label{prop:kqproperties}\\\vspace{-5mm}
\begin{enumerate}
\item The $(k,q)$-rank of a matrix $Z\in\RR^{m_1 \times m_2}$ can be strictly larger than $m_1$ and $m_2$.
\item There might be no solution of \eqref{eq:kqrank} such that $(a_i)_{1\leq i \leq r}$ or $(b_i)_{1\leq i \leq r}$ form a collection of orthonormal vectors.
\end{enumerate}
\end{proposition}
For $k=q=1$, the $(1,1)$-SVD decomposes $Z$ as a sum of matrices with only one non-zero element, showing that $(1,1)\mbox{-}\rank(Z) = \norm{Z}_0$. Since $\cA^n_i \subset \cA^n_j$ when $i\leq j$, we deduce from the expression of the $\kqrank$ as the optimal value of (\ref{eq:kqrank}) that the following tight inequalities hold:
$$
\forall (k,q) \in[1,m_1] \times [1,m_2]\,,\quad \rank(Z)\:\leq\:  \kqrank(Z) \:\leq\: \nm{Z}_0\,.
$$
The $\kqrank$ is useful to formulate problems in which a matrix should be modeled as or approximated by a matrix with sparse low rank factors, with the assumption that the sparsity level of the factors is fixed and known. For example, the standard rank-1 SPCA problem consists in finding the symmetric matrix with $(k,k)\mbox{-}\rank$ equal to $1$ and providing the best approximation of the sample covariance matrix \citep{zou2006sparse}.

\subsection{Two convex relaxations for the $\kqrank$}\label{sec:convexRelaxkqrank}
The $\kqrank$ is obviously a discrete, nonconvex index, like the rank or the cardinality, leading to computational difficulties when one wants to estimate matrices with small $\kqrank$. In this section, we propose two convex relaxations of the $\kqrank$ aimed at mitigating these difficulties. They are both instances of the atomic norms introduced by \citet{Chandrasekaran12}, which we first review. 

\begin{definition}[Atomic norm]\label{def:atomicorm}
Given a centrally symmetric compact subset $\cA \subset \RR^p$ of elements called atoms, the \emph{atomic norm} induced by $\cA$ on $\RR^p$ is the gauge function\footnote{see \citet{Rockafellar1997Convex}, p.~28, for a precise definition of gauge functions.} of $\cA$, defined by
\begin{equation}\label{eq:gauge}
\norm{x}_{\cA} = \inf\cbr{t>0~:~x\in t\, \conv\, (\cA)}\, ,
\end{equation}
where $\conv(\cA)$ denotes the convex hull of $\cA$.
\end{definition}
\citet{Chandrasekaran12} show that the atomic norm induced by $\cA$ is indeed a norm, which can be rewritten as
\begin{equation}\label{eq:gauge2}
\nm{x}_{\cA} = \inf\cbr{ \sum_{a\in\cA} c_a~:~x=\sum_{a\in\cA} c_a a,\,~c_a\geq 0,~\forall a\in\cA  }\,,
\end{equation}
and whose dual norm satisfies
\begin{equation}\label{eq:dualAtomicNormGeneral}
\begin{split}
\nm{x}_{\cA}^* & \eqdef \sup\cbr{ \inr{x,z} ~:~ \nm{z}_{\cA} \leq 1 } \\
& = \sup\cbr{ \inr{x,a} ~:~ a \in \cA} \,.
\end{split}
\end{equation}
We can now define our first convex relaxation of the $\kqrank$:
\begin{definition}[$\kqtn$]\label{def:kqtn}
For a matrix $Z \in \RR^{m_1 \times m_2}$, the $\kqtn$ $\Okq(Z)$ is the atomic norm induced by the set of atoms:
\begin{equation}\label{eq:atoms}
\Akq = \cbr{ a b\trans ~:~a \in \cA_k^{m_1},~b \in \cA_q^{m_2} } \,.
\end{equation}
\end{definition}
In words, $\Akq$ is the set of matrices $Z \in \RR^{m_1 \times m_2}$ such that $\kqrank(Z) = 1 $ and $\nmop{Z}= 1$. Plugging (\ref{eq:atoms}) into (\ref{eq:gauge2}), we obtain an equivalent definition of the $\kqtn$ as the optimal value of the following optimization problem:
\begin{equation}
\label{eq:relaxed_kqrank}
\Okq(Z) = \min \cbr{ \nm{c}_1 ~:~  Z=\sum_{i=1}^\infty c_i a_i b_i\trans, ~ (a_i,b_i,c_i) \in \cA^{m_1}_k \times \cA^{m_2}_q \times \RR_+} \,.
\end{equation}
Comparing (\ref{eq:relaxed_kqrank}) to (\ref{eq:kqrank}) shows that the $\kqtn$ is derived from the $\kqrank$ by replacing the non-convex $\ell_0$ pseudo-norm of $c$ by its convex $\ell_1$ norm in the optimization problem. In particular, in the case $k=m_1$ and $q=m_2$, the $\kqtn$ is the usual trace norm (equal to the $\ell_1$-norm of singular values), i.e. the usual relaxation of the rank (which is the $\ell_0$-norm of the singular values). Similarly, when $k=q=1$, the $\kqtn$ is simply the $\ell_1$ norm. Just like the $\kqrank$ interpolates between the $\ell_0$ pseudo-norm and the rank, the $\kqtn$ interpolates between the $\ell_1$ norm and the trace norm. Indeed,  since $\cA^n_i \subset \cA^n_j$ when $i\leq j$, we deduce from the expression of $\Okq$ as the optimal value of (\ref{eq:relaxed_kqrank}) that the following tight inequalities hold for any $1\leq k\leq m_1$ and $1 \leq q \leq m_2$:
\begin{equation}\label{eq:inequalityOmegaL1trace}
\Omega_{m_1,m_2}(Z) = \|Z\|_* \leq  \Okq(Z) \leq  \|Z\|_1 = \Omega_{1,1}(Z) \,.
\end{equation}
In the case of the trace norm, the optimal decomposition solving \eqref{eq:relaxed_kqrank} is unique and is in fact the singular value decomposition of the matrix $Z$ with $a_i$ and $b_i$ being respectively the left and right singular vectors and $c_i$ the singular values. This suggest that we can use the $\kqtn$ to generalize the definition of the SVD to sparse SVDs as follows 
\begin{definition}[$\kqsvd$]
For a matrix $Z\in\RR^{m_1\times m_2}$, we call $(k,q)$-sparse singular value decomposition (or $\kqsvd$) any decomposition $Z=\sum_{i=1}^r c_i a_i b_i\trans$ that solves (\ref{eq:relaxed_kqrank}) with $c_1\geq c_2 \geq \ldots \geq c_r >0$. In such a decomposition, we refer to vectors $(a_i,b_i)_{1\leq i\leq r}$ as a set of left and right $(k,q)$-sparse singular vectors of $Z$, and to $(c_i)_{1\leq i \leq r}$ as the corresponding collection of $(k,q)$-sparse singular values.

\end{definition}
Without surprise, the $\kqsvd$ does not share a number of usual properties of the SVD, when $k<m_1$ and $q<m_2$:
\begin{proposition}\label{prop:softkqproperties}
\begin{enumerate}
\item The $\kqsvd$ is not necessarily unique.
\item The $\kqsvd$s do not necessarily solve \eqref{eq:kqrank}: the number of non-zero $(k,q)$-sparse singular values of a matrix can be strictly larger than its $\kqrank$.
\item The $(k,q)$-sparse left or right singular vectors are not necessarily orthogonal to each other.
\end{enumerate}
\end{proposition}
In addition to (\ref{eq:relaxed_kqrank}), the next lemma provides another explicit formulation for the $\kqtn$, its dual and its sub differential:
\begin{lemma}\label{lem:OmegakqAndDual} 
For any $Z \in \RR^{m_1 \times m_2}$ we have 
\begin{equation} \label{eq:defOmegaInfTraceNorm}
\Okq(Z) = \inf \cbr{ \sum_{(I,J)\in\Gk\times \Gq } \nmtr{Z^{(I,J)} }~:~Z = \sum_{(I,J) } Z^{(I,J)}~,~\Supp(Z^{(I,J)}) \subset I\times J  }\,,
\end{equation}
and
\begin{equation}\label{eq:omegadual}
\Okq^*(Z)  = \max \cbr{ \nmop{ Z_{I,J} }~:~I\in\Gk~,~J\in\Gq } \,.
\end{equation}
The subdifferential of $\Okq$ at an atom $A=a b\trans \in \Akq$ with $I_0=\supp(a)$ and $J_0=\supp(b)$ is
\begin{equation}\label{eq:subdiff}
\partial \Okq (A) =  \cbr{  A+ Z ~:~ A Z_{I_0,J_0}\trans = 0,~  A \trans Z_{I_0,J_0}= 0, ~\forall (I,J)\in \Gk \times \Gq\,~ \nmop{A_{I,J} + Z_{I,J}} \leq 1 } \,.
\end{equation}
\end{lemma}

Our second norm is again an atomic norm, but is obtained by focusing on a more restricted set of atoms. It is motivated by applications where we want to estimate matrices which, in addition to being sparse and low-rank, are constant over blocks, such as adjacency matrices of graphs with non-overlapping communities. For that purpose, consider first the subset of $\cA_k^m$ made of vectors whose nonzero entries are all equal in absolute value:
$$
\cAz^m_k=\Big \{a \in \RR^{m},~\|a\|_0 = k~,~ \forall i \in \text{supp}(a),\:|a_i| = {\textstyle \frac{1}{\sqrt k}}\Big \} \,.
$$
We can then define our second convex relaxation of the $\kqrank$:
\begin{definition}[$\kqcut$ norm]\label{def:kqcn}
We define the $\kqcutn$ $\Ozkq(Z)$ as the atomic norm induced by the set of atoms
\begin{equation}\label{eq:atoms0}
\Azkq = \cbr{ a b\trans ~:~ a \in \widetilde{\cA}^{m_1}_k, \: b \in \widetilde{\cA}^{m_2}_q  } \,.
\end{equation}
\end{definition}
In other words, the atoms in $\Azkq$ are the atoms of $\Akq$ whose nonzero elements all have the same amplitude.

Our choice of terminology is motivated by the following relation of our norm to the CUT-polytope: in the case $k=m_1$ and $q=m_2$, the unit ball of $\Ozkq$ coincides (up to a scaling factor of $\sqrt{m_1m_2}$) with the polytope known as the CUT polytope of the complete graph on $n$ vertices~\citep{deza}, defined by 
$$
\text{CUT} = \text{conv}\cbr{ ab\trans~,~a \in \{ \pm 1\}^{m_1} ~,~ b \in \{\pm 1\}^{m_2} }\,.
$$
The norm obtained as the gauge of the CUT polytope is therefore to the trace norm as $\Ozkq$ is to $\Okq$.

\subsection{Equivalent nuclear norms built upon vector norms}\label{sec:nuclear}

In this section we show that the $\kqtn$ (Definition~\ref{def:kqtn}) and the $\kqcutn$ (Definition~\ref{def:kqcn}), which we defined as atomic norms induced by specific atom sets, can alternatively be seen as instances of \emph{nuclear norms} considered by \citet{Jameson1987Summing}.  For that purpose it is useful to recall the general definition of nuclear norms and the characterization of the corresponding dual norms as formulated in \citet[Propositions 1.9 and 1.11]{Jameson1987Summing}:

\begin{proposition}[nuclear norm]\label{prop:jameson}
Let $\nm{\cdot}_\alpha$ and $\nm{\cdot}_\beta$ denote any vector norms on $\RR^{m_1}$ and $\RR^{m_2}$, respectively, then 
\[
\nu(Z) \eqdef \inf \cbr{ \sum_i \nm{a_i}_\alpha \nm{b_i}_\beta ~:~Z = \sum_i a_ib_i\trans} \,,
\] 
where the infimum is taken over all summations of finite length, is a norm over $\RR^{m_1 \times m_2}$ called the \emph{nuclear norm} induced by $\nm{ \cdot }_\alpha$ and $\nm{\cdot}_\beta$. Its dual is given by 
\begin{equation}\label{eq:dualNuclearNormGeneral} 
\nu^*(Z) = \sup \cbr{  a\trans Z b ~:~ \|a\|_\alpha \leq 1 ~,~\|b\|_\beta\leq 1 }  \,. 
\end{equation}
\end{proposition}

The following lemma shows that the nuclear norm induced by two atomic norms is itself an atomic norm. 
\begin{lemma}\label{lem:nuclearIsAtomic} 
If $\nm{\cdot}_\alpha$ and $\nm{\cdot}_\beta$ are two atomic norms on $\RR^{m_1}$ and $\RR^{m_2}$ induced respectively by two atom sets $\cA_1$ and $\cA_2$, then the nuclear norm on $\RR^{m_1 \times m_2}$ induced by $\nm{\cdot}_\alpha$ and $\nm{\cdot}_\beta$ is an atomic norm induced by the atom set:
$$
\cA = \cbr{ ab\trans~:~a\in\cA_1 ~,~ b\in\cA_2 }\,.
$$
\end{lemma}
We can deduce from it that the $\kqtn$ and $\kqcut$ are nuclear norms, associated to particular vector norms:
\begin{theorem}
\label{th:nuclearNorms}
\begin{enumerate}
\item  The $\kqtn$ is the nuclear norm induced by $\theta_k$ on $\RR^{m_1}$ and $\theta_q$ on $\RR^{m_2}$, where for any $j\geq 1$, $\theta_j$ is the $j$-support norm introduced by~\citet{argyriou12}.
\item The $\kqcutn$ is the nuclear norm induced by $\kappa_k$ on $\RR^{m_1}$ and $\kappa_q$ on $\RR^{m_2}$, where
for any $j\geq 1$:
\begin{equation}\label{eq:kappa}
 \kappa_j(w)=\frac{1}{\sqrt j} \max  \br{ \|w\|_\infty , \frac 1j \|w\|_1 } \,.
\end{equation}
 \end{enumerate}
\end{theorem}
For the sake of completeness, let us recall the closed-form expression of the $k$-support norm $\theta_k$ shown by \citet{argyriou12}. For any vector $w \in \RR^p$, let $\bar{w}\in\RR^p$ be the vector obtained by sorting the entries of $w$ by decreasing order of absolute values. Then it holds that
\begin{equation} \label{eq:ksuppsq} 
\theta_k(w) = \cbr{  \sum_{i=1}^{k-r-1} | \bar{w}_i |^2 + \frac{1}{r+1} \br{ \sum_{i=k-r}^p |\bar{w}_i| }^2 }^{\frac 12} \,,
\end{equation}
where $r \in \cbr{ 0, \cdots , k-1}$ is the unique integer such that $ |\bar{w}_{k-r-1}| > \frac{1}{r+1} \sum_{i=k-r}^p |\bar{w}_i| \geq |\bar{w}_{k-r}|$, and where by convention $|\bar{w}_0| = \infty$. 

Of course, Theorem~\ref{th:nuclearNorms} implies that in the vector case ($m_2=1$), the $\kqtn$ is simply equal to $\theta_k$ and the $\kqcut$ norm is equal to $\kappa_k$. A representation of the ``sharp edges" of unit balls of $\theta_k, \kappa_k$ and a appropriately scaled $\ell_1$ norm can be found in Figure~\ref{fig:unitatomballs} for the case $m_1=3$ and $k=2$.
\begin{figure}
\begin{center}
\includegraphics[trim=1cm 6cm 3cm 3.5cm, clip=true, width = 12cm]{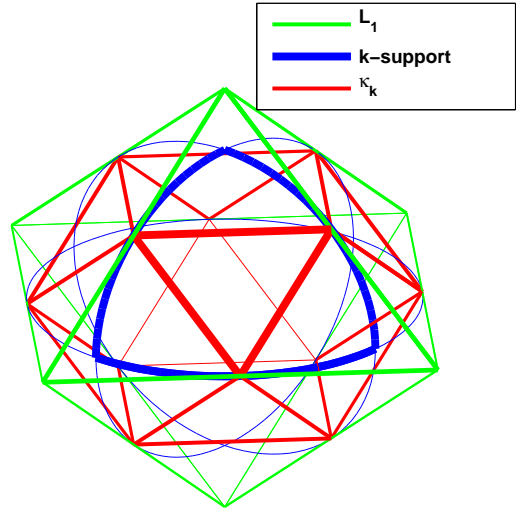}
\end{center}
\caption{Unit balls of 3 norms of interest for vectors of $\RR^3$ materialized by their sets of extreme points at which the norm is non-differentiable. Each unit ball is the convex hull of the corresponding sets. In green, the usual $\ell_1$-norm scaled by the factor $1/\sqrt{k}=1/\sqrt{2}$, in blue the norm $\theta_2$ (a.k.a.\ $2$-support norm), in red the norm $\kappa_2$ (see theorem~\ref{th:nuclearNorms}). Vertices of the $\kappa_2$ unit ball constitute the $\cAz_{2,1}$ set (see definition~\ref{def:kqcn}). The set $\cAz_{2,1}$ belongs to the unit spheres of all three norms (see proposition~\ref{prop:norm_ineqs}).}\label{fig:unitatomballs}
\end{figure}
In addition, the following results shows that the dual norms of $\theta_k$ and $\kappa_k$ have simple explicit forms:
\begin{proposition}
The dual norms of $\theta_k$ and $\kappa_k$ satisfy respectively:
\label{lem:dual_vector_norms}
$$
\theta_k^*(s)=\max_{I: |I|=k} \|s_I\|_2 \qquad \text{and} \qquad \kappa_k^*(s)=\frac{1}{\sqrt{k}}\max_{I: |I|=k}\|s_I\|_1 \,.
$$
\end{proposition}

To conclude this section, let us observe that nuclear norms provide a natural framework to construct matrix norms from vector norms, and that other choices beyond $\theta_k$ and $\kappa_k$ may lead to interesting norms for sparse matrix factorization. 
It is however known since \cite{Jameson1987Summing} \citep[see also][]{Bach11,Bach2013Convex} that the nuclear norm induced by vector $\ell_1$-norm is simply the $\ell_1$ of the matrix which fails to induce low rank (except in the very sparse case). However
\citet{Bach11} proposed nuclear norms associated with vectors norms that are similar to the elastic net penalty. 

\section{Learning matrices with sparse factors}\label{sec:applications}

In this section, we briefly discuss how the $\kqtn$ and $\kqcut$ norm can be used to attack various problems involving estimation of sparse low-rank matrices.

\subsection{Denoising}\label{sec:denoising}

Suppose $X\in\RR^{m_1 \times m_2}$ is a noisy observation of a low-rank matrix with sparse factors, assumed to have low $\kqrank$. A natural convex formulation to recover the noiseless matrix is to solve:
\begin{equation}\label{eq:proxprojection}
\min_{Z} \frac 12  \nmF{Z-X}^2 + \lambda \Okq(Z)\,,
\end{equation}
where $\lambda$ is a parameter to be tuned. Note that in the limit when $\lambda \rightarrow 0$, one simply obtains a $\kqsvd$ of $X$.

\subsection{Bilinear regression}

More generally, given some empirical risk $\risk(Z)$, it is natural to consider formulations of the form
$$
\min_Z \risk(Z) + \lambda \Okq(Z)
$$
to learn matrices that are a priori assumed to have a low $\kqrank$. A particular example is bilinear regression, where, given two inputs $x\in\RR^{m_1}$ and $x'\in\RR^{m_2}$, one observes as output a noisy version of $y=x\trans Z x'$. Assuming that $Z$ has low $\kqrank$ means that the noiseless response is a sum of a small number of terms, each involving only a small number of features from either of the input vectors. To estimate such a model from observations $\br{x_i,x'_i,y_i}_{i=1,\ldots,n}$, one can consider the following convex formulation:
\begin{equation}
\label{bilinlearn}
\min_Z \sum_{i=1}^n \ell\br{x_{i}\trans Z x_{i}',y_{i} }+\lambda \Okq(Z) \,,
\end{equation}
where $\ell$ is a loss function. A particular instance of (\ref{bilinlearn}) of interest is the quadratic regression problem, where $m_1=m_2$ and $x_i=x'_i$ for $i=1,\ldots,n$. Quadratic regression combined with additional constraints on $Z$ is closely related to phase retrieval \citep{candes11}. It should be noted that if $\ell$ is the least-square loss, (\ref{bilinlearn}) can be rewritten in the form 
$$
\min_Z \frac{1}{2} \| \mathcal X(Z) -y\|_2^2+\lambda \Okq(Z)\,,
$$
where $\mathcal X(Z)$ is a linear transformation of $Z$, so that the problem is from the point of view of the parameter $Z$ a linear regression with a well chosen feature map.

\subsection{Subspace clustering}

In subspace clustering, one assumes that the data can be clustered in such a way that the points in each cluster belong to a low dimensional space. If we have a design matrix $X\in \RR^{n\times p}$ with each row corresponding to an observation, then the previous assumption means that if $X^{(j)} \in \RR^{n_j \times p}$ is a matrix formed by the rows of cluster $j$, there exist a low rank matrix $Z^{(j)} \in \RR^{n_j \times n_j}$ such that $Z^{(j)}X^{(j)}=X^{(j)}$. This means that there exists a block-diagonal matrix $Z$ such that $ZX=X$ with low-rank diagonal blocks. This idea, exploited recently by \cite{wang13} implies that $Z$ is a sum of low rank sparse matrices; and this property still holds if the clustering is unknown. 
We therefore suggest that if all subspaces are of dimension $k$,  $Z$ may be estimated via 
\[ \min_Z \Omega_{k,k}(Z)\quad\st\quad ZX =X~.\]

\subsection{Sparse PCA}\label{sec:applicationSparsePCA}

In sparse PCA \citep{zou2006sparse,Ghaoui2004,witten2009penalized}, one tries to approximate an empirical covariance matrix $\hat{\Sigma}_n$ by a low-rank matrix with sparse factors. Although this is similar to the denoising problem discussed in Section~\ref{sec:denoising}, one may wish in addition that the estimated sparse low-rank matrix be symmetric and positive semi-definite (PSD), in order to represent a plausible covariance matrix. This suggests to formulate sparse PCA as follows:
\begin{equation}
\label{eq:projRankKQ}
  \min_Z \cbr{ \nmF{ \hat{\Sigma}_n-Z }~:~(k,k)\mbox{-}\rank(Z) \leq r  ~\text{and}~Z \succeq 0  } \,,
\end{equation}
where $k$ is the maximum number of non-zero coefficient allowed in each principal direction. In contrast to sequential approaches that estimate the principal components one by one \citep{Mackey08}, this formulation requires to find simultaneously a set of factors which are complementary to one another in order to explain as much variance as possible. A natural convex relaxation of (\ref{eq:projRankKQ}) is 
\begin{equation}\label{eq:proxOmegaK_w_PSD_constraints}
  \min_Z \cbr{\frac 12  \nmF{ \hat{\Sigma}_n-Z }^2 + \lambda \Omega_{k,k}(Z) ~:~Z \succeq 0  }\,,
\end{equation}
where $\lambda$ is a parameter that controls in particular the rank of the approximation.

However, although the solution of (\ref{eq:proxOmegaK_w_PSD_constraints}) is always PSD, its $(k,k)$-SVD leading to $Z=\sum_{i=1}^r c_i a_i b_i\trans$ may not be composed of symmetric matrices (if $a_i \neq b_i)$, and even if $a_i=b_j$ the corresponding $c_i$ may be negative, as the following proposition shows:
\begin{proposition}\label{prop:symsoftSVD}
\begin{enumerate}
\item There might be no decomposition of a PSD matrix attaining its $(k,q)$-rank (i.e.~no solution of \eqref{eq:kqrank}) which decomposes it as a sum of symmetric terms.
\item The $(k,k)$-SVD of a PSD matrix is itself not necessarily a sum of symmetric terms.
\item Some PSD matrices cannot be written as a positive combination of rank one $(k,k)$-sparse matrices, even for $k > 1$.
\end{enumerate}
\end{proposition}
This may be unappealing, as one would like to interpret the successive rank-1 matrices as covariance matrices over a subspace that explain some of the total variance. One may therefore prefer a decomposition with less sparse or more factors, potentially capturing less variance.

One solution is to replace $\Omega_{k,k}$ in (\ref{eq:proxOmegaK_w_PSD_constraints}) by another penalty which directly imposes symmetric factors with non-negative weights. This is easily obtained by replacing the set of atoms $\cA_{k,k}$ in Definition~\ref{def:kqtn} by $\cA_{k,\succeq}=\cbr{aa\trans, a \in \cA_k}$, and considering the corresponding atomic norm which we denote by $\Omega_{k, \succeq}$. To be precise, $\Omega_{k, \succeq}$ is not a norm but only a gauge because the set $\cA_{k,\succeq}$ is not centrally symmetric.  Instead of (\ref{eq:proxOmegaK_w_PSD_constraints}), it possible to use the following convex formulation of sparse PCA:
\begin{equation}
\label{eq:proxOmegaKPSD}
  \min_Z \frac 12  \nmF{ \hat{\Sigma}_n-Z }^2 + \lambda \Omega_{k, \succeq}(Z)\,.
\end{equation}
By construction, the solution of (\ref{eq:proxOmegaKPSD}) is not only PSD, but can be expanded as a sum of matrices $Z=\sum_{i=1}^r c_i a_i a_i\trans$, where for all $i=1,\ldots,r$, the factor $a_i$ is $k$-sparse and the coefficient $c_i$ is positive. This formulation is therefore particularly relevant if $\hat{\Sigma}_n$ is believed to be a noisy matrix of this form. It should be noted however that, by Proposition~\ref{prop:symsoftSVD}, $\Omega_{k, \succeq}$ is infinite for some PSD matrices\footnote{This is possible because $\Omega_{k, \succeq}$ is only a gauge and not a norm.}, which implies that some PSD matrices cannot be approximated well with this formulation.

\subsection{NP-hard convex problems}\label{sec:nphard}

Although the $\kqtn$ and related norms allow us to formulate several problems of sparse low-rank matrix estimation as convex optimization problems, it should be pointed out that this does not guarantee the existence of efficient computational procedures to solve them. Here we illustrate this with the special case of the best $(k,q)$-sparse and rank $1$ approximation to a matrix, which turns out to be a NP-hard problem. Indeed, let us consider the three following optimization problems, which are equivalent since they return the same rank one subspace spanned by $ab\trans$: 
\begin{equation}\label{eq:3}
\min_{(a,b,c)\in \cA_k \times \cA_q \times \RR^+}  \nmF{X-cab\trans}^2 ~; \qquad \quad
\max_{(a,b) \in \cA_k \times \cA_q} a\trans X b ~; \qquad \quad
\max_{Z:\: \Okq(Z) \leq 1}  \tr(XZ\trans)\,.
\end{equation}
In particular, if $k=q$ and $X=\widehat{\Sigma}_n$ is an empirical covariance matrix, then the symmetric solutions of the problem considered are the solution to the following rank 1 SPCA problem
\begin{equation}\label{eq:rank1SPCA}
\max_z \cbr{ z\trans \hat{\Sigma}_n  z ~:~\|z\|_2 = 1 ~,~ \|z\|_0 \leq k~},
\end{equation}
which it is known to be NP-hard \citep{moghaddam2008sparse}. This shows that, in spite of being a convex formulation involving the $\kqtn$, the third formulation in (\ref{eq:3}) is actually NP-hard. In practice, we will propose heuristics in Section~\ref{sec:num} to approximate the solution of convex optimization problems involving the $\kqtn$.


\section{Statistical properties of the $\kqtn$ and the $\kqcut$ norm}\label{sec:statisticalproperties}

In this section we study theoretically the benefits of using the new penalties $\Okq$ and $\Ozkq$ to infer low-rank matrices with sparse factors, as suggested in Section~\ref{sec:applications}, postponing the discussion of how to do it in practice to Section~\ref{sec:algos}. Building upon techniques proposed recently to analyze the statistical properties of sparsity-inducing penalties, such as the $\ell_1$ penalty or more general atomic norms, we investigate two approaches to derive statistical guarantees.  In Section \ref{sec:denoisingStatistics} we study the expected dual norm of some noise process, from which we can deduce upper bounds on the learning rate for least squares regression and a simple denoising task. In Section \ref{sec:statisticalDimension} we estimate the statistical dimension of objects of interest both in the matrix and vector cases and compare the asymptotic rates, which shed light on the power of the norms we study when used as convex penalties. The results in Section \ref{sec:denoisingStatistics} are technically easier to derive and contain bounds for a matrix of arbitrary $\kqrank$. The results provided in Section \ref{sec:statisticalDimension} rely on a more involved set of tools, they provide more powerful bounds but we do not derive results for matrices of arbitrary $\kqrank$.

\subsection{Performance of the $\kqtn$ in denoising}\label{sec:denoisingStatistics}

In this Section we consider the simple denoising setting (Section~\ref{sec:denoising}) where we wish to recover a low-rank matrix with sparse factors $Z^\star \in \RR^{m_1 \times m_2}$ from a noisy observation $Y\in \RR^{m_1 \times m_2}$ corrupted by additive Gaussian noise:
$$
Y = Z^\star + \sigma G\,,
$$
where $\sigma>0$ and $G$ is a random matrix with entries i.i.d. from $\mathcal N(0,1)$. Given a convex penalty $\Omega : \RR^{m_1 \times m_2} \rightarrow \RR$, we consider, for any $\lambda>0$, the estimator
$$
\hat{Z}^{\lambda}_{\Omega} \in \arg \min_{Z} \frac 12  \nmF{Z-Y}^2 + \lambda \Omega(Z)\,.
$$
The following result, valid for any norm $\Omega$, provides a general control of the estimation error in this setting, involving the dual norm of the noise:
\begin{lemma}\label{lem:denoisingbound}
If $\lambda\geq \sigma\Omega^*(G)$ then
$$
\nmF{\hat{Z}^\lambda_\Omega - Z^\star}^2 \leq 4 \lambda \Omega(Z^\star)\,.
$$
\end{lemma}
This suggests to study the dual norm of a random noise matrix $\Omega^*(G)$ in order to derive a upper bound on the estimation error. The following result provides such upper bounds, in expectation, for the $\kqtn$ as well as the standard $\ell_1$ and trace norms:
\begin{proposition}\label{prop:noisedualbound}
Let $G\in \RR^{m_1 \times m_2}$ be a random matrix with entries i.i.d. from $\mathcal N(0,1)$. The expected dual norm of $G$ for the $\kqtn$, the $\ell_1$ norm and the trace norm is respectively bounded by:
\begin{equation}
\begin{split}
 \E\, \Okq^*(G) & \leq 4 \br{\sqrt{ k \log \frac {m_1}k + 2k} + \sqrt{q \log \frac {m_2}q + 2q} } \,,\\
 \E\, \nm{G}_1^* & \leq \sqrt{2 \log(m_1 m_2)} \,, \\
 \E\, \nmtr{G}^* & \leq \sqrt{m_1} + \sqrt{m_2} \,.
\end{split}
\end{equation}
\end{proposition}
To derive an upper bound in estimation errors from these inequalities, we consider for simplicity\footnote{Similar bounds could be derived with large probability for the non-oracle estimator by controlling the deviations of $\Omega^*(G)$ from its expectation.} the oracle estimate $\hat{Z}_\Omega^\text{Oracle}$ equal to $\hat{Z}_\Omega^\lambda$ where $\lambda=\sigma\Omega^*(G)$. From Lemma~\ref{lem:denoisingbound} we immediately get the following control of the mean estimation error of the oracle estimator, for any penalty $\Omega$:
\begin{equation}\label{eq:oracledenoising}
\E\, \nmF{\hat{Z}^\text{Oracle}_\Omega - Z^\star}^2 \leq 4 \sigma \Omega(Z^\star)  \E~ \Omega^*(G)\,.
\end{equation}
We can now derive upper bounds in estimation errors for the different penalties in the so-called single spike model, where the signal $Z^\star$ consists of an atom $ab\trans \in \Akq$, and we observed a noisy matrix $Y = ab\trans + \sigma G$.  Since for an atom $a b\trans \in \Akq$ while  $\|ab\trans\|_1 \leq kq / \sqrt {(kq)} = \sqrt{kq}$, $\Okq(ab\trans) = \|ab\trans\|_* = 1$, we immediately get the following by plugging the upper bounds of Proposition~\ref{prop:noisedualbound} into (\ref{eq:oracledenoising}): 
\begin{corollary}\label{cor:slowrate}
When $Z^\star \in \Akq$ is an atom, the expected errors of the oracle estimators using respectively the $\kqtn$, the $\ell_1$ norm and the trace norm are respectively upper bounded by:
\begin{equation}
\begin{split}
\E\, \nmF{\hat Z_{\Okq}^\text{Oracle}- Z^\star }^2 & \leq   8 ~\sigma~ \br{\sqrt{ k \log \frac {m_1}k + 2k} + \sqrt{q \log \frac {m_2}q + 2q} } \,,\\
\E\, \nmF{ \hat Z_1^\text{Oracle}- Z^\star }^2 &\leq 2 \sigma \| Z^\star \|_1 \sqrt{2 \log (m_1 m_2)} \leq 2 \sigma \sqrt{2 kq \log (m_1 m_2)} \,,\\
\E\, \nmF{ \hat Z_*^\text{Oracle}- Z^\star }^2 & \leq   2 \sigma  (\sqrt {m_1} + \sqrt{m_2})\,.
\end{split}
\end{equation}
\end{corollary}

\begin{remark} It is straightforward to see that if in the latter Corollary \ref{cor:slowrate} the matrix $Z^\star$ is the convex combination of $r > 1$ atoms, a factor $r \geq \kqrank(Z^\star)$ appears in the upper bounds. This suggests a (sub-)linear dependence of the denoising error in the $\kqrank$.
\end{remark}

To make the comparison easy, orders of magnitudes of these upper bounds are gathered in Table \ref{tab:compareDenoisingMMSE} for the case where $Z^\star \in \Azkq$, and for the case where $m_1=m_2=m$ and $k=q=\sqrt{m}$.
\begin{table}
\begin{center}
\begin{tabular}{|c||c|c|c|}
\hline
Matrix norm & $(k,q)$-trace & trace  & $\ell_1$ \\
\hline
\hline
$\Omega(Z^\star) \E~ \Omega^*(G)$ &$\sqrt {k  \log \frac {m_1}k} + \sqrt {q  \log \frac {m_2}q } $&$ \sqrt {m_1} + \sqrt{m_2}$& $\sqrt{ kq \log (m_1 m_2)}$\\
\hline
 $k = \sqrt m$ & $m^{1/4}\sqrt{ \log m}$&$ \sqrt m $&$\sqrt{ m \log m} $ \\
\hline
\end{tabular}
\caption{Various norms mean square error in denoising an atom $ab \trans \in \Azkq$ corrupted with unit variance Gaussian noise. The column ``$k = \sqrt m$'' corresponds to the order of magnitudes in the regime where $m = m_1 = m_2$ and  $k = q = \sqrt m$. }\label{tab:compareDenoisingMMSE}
\end{center}
\end{table}
In the later case, we see in particular that the $\kqtn$ has a better rate than the $\ell_1$ and trace norms, in $m^{\frac{1}{4}}$ instead of $m^{\frac{1}{2}}$ (up to logarithmic terms).
Note that the largest value of $\| Z^\star \|_1$ is reached when $Z^\star \in \Azkq$ and equals $\sqrt{kq}$. By contrast, when $Z^\star \in \Akq$ gets far from $\Azkq$ elements
then the expected error norm diminishes for the $\ell_1$-penalized denoiser $\hat Z_1^\text{Oracle}$ reaching $\sigma \sqrt{2~\log(m_1 m_2)}$ on $e_1e_1\trans$ while not changing for the two other norms.\vspace{2mm}

Obviously the comparison of upper bounds is not enough to conclude to the superiority of $\kqtn$ and, admittedly, the problem of denoising considered here is a special instance of linear regression in which the design matrix is the identity, and, since this is a case in which the design is trivially incoherent, it is possible to obtain fast rates for decomposable norms such as the $\ell_1$ or trace norm \citep{negahban2012unified}; however, slow rates are still valid in the presence of an incoherent design, or when the signal to recover is only weakly sparse, which is not the case for the fast rates. Moreover, the result proved here is valid for matrices of rank greater than $1$.
We present in the next section more involved results, based on lower and upper bounds on the so-called statistical dimension of the different norms \citep{Amelunxen13}, a measure which is closely related to Gaussian widths.


\subsection{Performance through the statistical dimension}\label{sec:statisticalDimension}

Powerful results from asymptotic geometry have recently been used by \citet{Chandrasekaran12, Oymak12relation, Amelunxen13,Foygel2014Corrupted} to quantify the statistical power of a convex nonsmooth regularizer used as a constraint or penalty. These results rely essentially on the fact that if the tangent cone\footnote{As detailed later, the tangent cone is the closure of the cone of descent directions.} of the regularizer at a point of interest $Z$ is thiner, then the regularizer is more efficient at solving problems of denoising, demixing and compressed sensing of $Z$. The gain in efficiency can be quantified by appropriate measures of width of the tangent cone such as the Gaussian width of its intersection with a unit Euclidean ball \citep{Chandrasekaran12}, or the closely related concept of \textit{statistical dimension} of the cone, proposed by \citet{Amelunxen13}. In this section, we study the statistical dimensions induced by different matrix norms in order to compare their theoretical properties for exact or approximate recovery of sparse low-rank matrices. In particular, we will consider the norms $\Okq$, $\widetilde{\Omega}_{k,q}$ and linear combinations of the $\ell_1$ and trace norms, which have been used in the literature to infer sparse low-rank matrices \citep{Richard12,Oymak12}. For convenience we therefore introduce the notation $\Gamma_\mu$ for the norm that linearly interpolates between the trace norm and the (scaled) $\ell_1$ norm:
\begin{equation}
\label{def:gamma_mu}
\forall \mu \in [0,1],\: \forall Z \in \RR^{m_1\times m_2},  \quad \Gamma_\mu(Z) \eqdef \frac{\mu}{\sqrt{kq}} \nm{Z}_1 + (1-\mu) \nmtr{Z}\,,
\end{equation}
so that $\Gamma_0$ is the trace norm and $\Gamma_1$ is the $\ell_1$ norm up to a constant\footnote{Note that the scaling ensures that $\Gamma_\mu(A)=1$ for $\mu\in[0,1]$ and $A\in  \Azkq$.}.

\subsubsection{The statistical dimension and its properties}\label{sec:statdim}

Let us first briefly recall what the statistical dimension of a convex regularizer $\Omega:\RR^{m_1 \times m_2} \rightarrow \RR$ refers to, and how it is related to efficiency of the regularizer to recover a matrix $Z\in\RR^{m_1 \times m_2}$. For that purpose, we first define the tangent cone $T_\Omega(Z)$ of $\Omega$ at $Z$ as the closure of the cone of descent directions, \ie, 
\begin{equation}\label{eq:tcone}
T_\Omega(Z) \eqdef \overline{\bigcup_{\tau > 0} \cbr{H\in\RR^{m_1 \times m_2}~:~\Omega(Z+\tau H) \leq \Omega(Z)}}\, .
\end{equation}
The statistical dimension $\sdim (Z,\Omega) $ of $\Omega$ at $Z$ can then be formally defined as
\begin{equation}\label{eq:sdim}
\sdim (Z,\Omega) \eqdef \E \sqb{ \nmF{\Pi_{T_\Omega(Z)}(G)}^2}\,,
\end{equation}
where $G$ is a random matrix with i.i.d. standard normal entries and $\Pi_{T_\Omega(Z)}(G)$ is the orthogonal projection of $G$ onto the cone $T_\Omega(Z)$. 
 The statistical dimension is a powerful tool to quantify the statistical performance of a regularizer in various contexts, as the following non-exhaustive list of results shows.
\begin{itemize}
\item {\bf Exact recovery with random measurements. } Suppose we observe $y = {\mathcal X}(Z^\star)$ where $\cX : \RR^{m_1 \times m_2}\to \R^n$ is a random linear map represented by random design matrices $X_i$ $i = 1, \ldots, n$ having iid entries drawn from $\mathcal N(0,1/n)$. Then \citet[Corollary 3.3]{Chandrasekaran12} shows that
\begin{equation}
\label{eq:exact_rec}
\hat Z =  \argmin_Z \Omega(Z) \quad\text{ s.th.}\quad \cX(Z) = y
\end{equation}
is equal to $Z^\star$ with overwhelming probability as soon as $n \geq \sdim (Z^\star,\Omega)$. In addition \citet[Theorem II]{Amelunxen13} show that a phase transition occurs at $n=\sdim (Z^\star,\Omega)$ between a situation where recovery fails with large probability (for $n \leq \sdim (Z^\star,\Omega) - \gamma \sqrt{m_1 m_2}$, for some $\gamma>0$) to a situation where recovery works with large probability (for $n \geq \sdim (Z^\star,\Omega) + \gamma \sqrt{m_1 m_2}$).
\item {\bf Robust recovery with random measurements. } Suppose we observe  $y = \cX(Z^\star) + \epsilon$ where $\mathcal{X}$ is again a random linear map, and in addition the observation is corrupted by a random noise $\epsilon \in \RR^n$. If the noise is bounded as $\|\epsilon\|_2\leq \delta$, then  \citet[Corollary 3.3]{Chandrasekaran12} show that
\begin{equation}
\label{eq:rec_delta}
\hat Z =  \argmin_Z \Omega(Z) \quad\text{ s.th.}\quad\|\cX(Z) - y\|_2\leq \delta 
\end{equation}
satisfies $\nmF{\hat Z - Z^\star} \leq 2\delta / \eta$ with overwhelming probability as soon as $n \geq (\sdim (Z^\star,\Omega) + \frac 32) / (1-\eta)^2$.
\item {\bf Denoising. } Assume  a collection of noisy observations $X_i = Z^\star +\sigma \epsilon_i$ for $i = 1,\cdots, n$ is available where $\epsilon_i \in\R^{m_1 \times m_2}$ has i.i.d. $\mathcal N(0,1)$ entries, and let $Y = \frac 1n \sum_{i=1}^nX_i $ denote their average. \citet[Proposition 4]{Chandrasekaran13} prove that   
\begin{equation}
\label{eq:denoiser}
\hat Z = \argmin_Z \nmF{Z- Y} \quad\text{ s.th.}\quad \Omega(Z)\leq \Omega(Z^\star)
\end{equation}
satisfies $\E \nmF{ \hat Z - Z^\star}^2\leq\frac {\sigma^2}{n} \sdim (Z^\star,\Omega)$. \\

\item {\bf Demixing. } Given two matrices  $Z^\star, V^\star \in \RR^{m_1 \times m_2}$, suppose we observe $y = {\mathcal U}(Z^\star) + V^\star$ where ${\mathcal U} : \R^{m_1 \times m_2} \mapsto  \R^{m_1 \times m_2}$ is a random orthogonal operator. Given two convex functions $\Gamma, \Omega : \RR^{m_1 \times m_2} \rightarrow \RR$, \citet[Theorem III]{Amelunxen13} show that
$$
(\hat Z , \hat V) = \argmin_{(Z,V)} \Omega(Z) \quad \text{s.th.} \quad \Gamma(V)\leq \Gamma(V^\star)\quad\text{and}\quad y = {\mathcal U} (Z) + V
$$
is equal to $(Z^\star , V^\star)$ with probability at least $1-\eta$ provided that $$\sdim (Z^\star,\Omega) + \sdim (V^\star,\Gamma) \leq m_1 m_2 - 4\sqrt{ m_1 m_2\,\log \frac{4}{\eta}}.$$ Conversely if $\sdim (Z^\star,\Omega) + \sdim (V^\star,\Gamma) \geq m_1 m_2 + 4\sqrt{m_1 m_2\, \log \frac{4}{\eta}}$, the demixing fails with probability at least $1-\eta$. 
\end{itemize}

\subsubsection{Some cone inclusions and their consequences}\label{subset:statdim}

In this and subsequent sections, we wish to compare the behavior of $\Okq$ and $\Ozkq$ and $\Gamma_{\mu}$, as defined in (\ref{def:gamma_mu}). Before estimating and comparing the statistical dimensions of these norms, which requires rather technical proofs, let us first show through simple geometric arguments that for a number of matrices, the tangent cones of the different norms are actually nested. This will allow us to derive deterministic improvement in performance when a norm is used as regularizer instead of another, which should be contrasted with the kind of guarantees that will be derived from bounds on the statistical dimension and which are typically statements holding with very high probability. The results in this section are proved in Appendix~\ref{app:geom}.
\begin{proposition} The norms considered satisfy the following equalities and inequalities:
\label{prop:norm_ineqs}
\begin{align*}
&\forall \mu \in [0,1],\: \forall Z \in \RR^{m_1\times m_2},  \quad 
& &\Gamma_\mu(Z) \leq \Okq(Z) \leq \Ozkq(Z),\\
&\forall \mu \in [0,1],\:  \forall A \in \Azkq,  \quad 
& &\Gamma_\mu(A)= \Okq(A)= \Ozkq(A)=1.
\end{align*}
\end{proposition}
Put informally, the unit balls of $\Ozkq$, $\Okq$ and of all convex combinations of the trace norm and the scaled $\ell_1$-norm are nested and meet for matrices in $\Azkq$. This property is illustrated in the vector case (for $\mu=1$) on Figure~\ref{fig:unitatomballs}. In fact $\Azkq$ is a subset of the extreme points of the unit norms of all those norms except for the scaled $\ell_1$-norm (corresponding to the case $\mu=1$). Given that the unit balls meet on $\Azkq$ and are nested, their tangent cones on $\Azkq$ must also be nested:
\begin{corollary}\label{prop:TangentConeInclusion}
The following nested inclusions of tangent cones hold:
\begin{equation}\label{eq:ranktangcone}
\forall \mu \in [0,1],\:  \forall A \in \Azkq, \qquad T_{\Gamma_\mu}(A) \supset T_{\Okq}(A) \supset T_{\Ozkq}(A)\,.
\end{equation}
As a consequence, for any $A \in \Azkq$, the statistical dimensions of the different norms satisfy:
\begin{equation}\label{eq:rankstatdim}
\sdim(A,\Ozkq) \leq \sdim(A,\Okq) \leq \sdim(A,\Gamma_\mu)\,.
\end{equation}
\end{corollary}

\OMIT{
A weaker version of the previous result, for the $\ell_1$ norm rather than $\Gamma_\mu$, can be formulated beyond $\Azkq$, for the positive orthant (for any other fixed orthant the result holds as well):
\begin{corollary}
For any $A$ in the convex hull of $\Azkq \cap \R_+^{m_1\times m_2}$, we have 
$$\frac{1}{\sqrt{kq}}\|A\|_1=\Okq(A)=\Ozkq(A)=1 \quad  \text{and thus} \quad T_{\|\cdot\|_1}(A) \supset T_{\Okq}(A) \supset T_{\Ozkq}(A).$$
\end{corollary}
\begin{proof}
Let $c_1A_1+\ldots+c_r A_r$ be an optimal decomposition of $A$ for ${\Ozkq}$ with $A_i \in \cAz \cap \R_+^{m_1\times m_2}$. Then $\|A\|_1=\sum_{i=1}^r c_i \|A_i\|_1=\sqrt{kq} \sum_{i=1}^r c_i=\sqrt{kq} \, {\Ozkq}(A)$. The fact that ${\Ozkq}(A)=\Ozkq(A)$ follows from the first inequality of proposition~\ref{prop:norm_ineqs}. The inclusions of the cones follow from the same argument as in corollary~\ref{prop:TangentConeInclusion}.
\end{proof}
Note that in the interior of the convex hull of $\Azkq \cap\R_+^{m_1m_2}$, the fact that norms are equal implies that their tangent cones are equal. But the result is essentially interesting on the lower dimensional faces of this convex hull. In the vector case, the above result corresponds to the fact that the red frontal triangle in figure~\ref{fig:unitatomballs} belongs to the boundary of all three unit balls.
} 

As reviewed in Section~\ref{sec:statdim}, statistical dimensions provide estimates for the performance of the different norms in different contexts. Plugging (\ref{eq:rankstatdim}) in these results shows that to estimate an atom in $\Azkq$, using $\Ozkq$ is at least as good as using $\Okq$ which itself is at least as good as using any convex combination of the $\ell_1$ and trace norms.

Note that the various statements in Section~\ref{sec:statdim} provide upper bounds on the performance of the different norms, with are guarantees that are either probabilistic or hold in expectation. In fact, the inclusion of the tangent cones (\ref{eq:ranktangcone}) and a fortiori the tangential inclusion of the unit balls imply much stronger results since it can also lead some deterministic statements, such as the following:
\begin{corollary}[Improvement in exact recovery]
\label{cor:improvement_in_exact_recovery}
Consider the problem of exact recovery of a matrix $Z^*\in \Azkq$ from random measurements $y=\cX(Z^*)$ by solving (\ref{eq:exact_rec}) with the different norms. For any realization of the random measurements, exact recovery with $\Gamma_{\mu}$ for any $0\leq \mu \leq 1$ implies exact recovery with $\Okq$ which itself implies exact recovery with $\Ozkq$.
\end{corollary}
Note that in the vector case ($m_2=1$), where the $\kqtn$ $\Omega_{k,1}$ boils down to the $k$-support norm $\theta_k$, the tangent cone inclusion (\ref{eq:ranktangcone}) is not always strict:
\begin{proposition}\label{lem:tanconeequal}
For any $a\in\cAz_{k}^m$, $T_{\Gamma_1}(a)=T_{\theta_k}(a)$.
\end{proposition}
In words, the tangent cone of the $\ell_1$ norm and of the the $k$-support norm are equal on $k$-sparse vectors with constant non-zero entries, which can be observed in Figure~\ref{fig:unitatomballs}. This suggests that, in the vector case, the $k$-support norm is not better than the $\ell_1$ norm to recover such constant sparse $k$-vectors.

\OMIT{
, at any point $a \in \cAz_{k}^p$, it is easy to show (by a direct computation of the subgradients, substitute in equation \ref{eq:subdiffksupport} $r = 1$\GO{Should we add this as a lemma in the appendix and prove it?} \er{I find it interesting, I added a reference to the equation in Appendix corresponding to this with $r = 1$ if we need more ok we can add a lemma}) that the tangent cones of the $\ell_1$-norm and of $\Omega_{k,1}=\theta_k$ (the $k$-support norm) are equal, which suggests that the $k$-support norm might not provide better statistical performance than the $\ell_1$-norm. In particular, given the previous discussion, it does not provide improvement in terms of exact recovery, and when $\delta \rightarrow 0$ the denoising performances are exactly equivalent. This begs the question of whether in the matrix case $\Okq$ provides a worthwhile improvement over $\Gamma_{\mu}$.
} 

\subsubsection{Bounds on the statistical dimensions}\label{sec:boundsdim}

The results presented in Section~\ref{subset:statdim} apply only to a very specific set of matrices ($\Azkq$), and do not characterize quantitatively the relative performance of the different norms. In this Section, we turn to more explicit estimations of the statistical dimension of the different norms at atoms in $\Azkq$ and $\Akq$.

We consider first the statistical dimension of the $\kqcut$ norm $\Ozkq$ on its atoms $\Azkq$. 
The unit ball of $\widetilde{\Omega}_{k,q}$ is a vertex-transitive polytope with $2^{k+q}{m_1 \choose k} {m_2 \choose q}$ vertices. As a consequence, it follows immediately from Corollary 3.14 in \citet{Chandrasekaran12} and from the upper bound $\log{m \choose k} \leq k(1+ \log (m/k))$, that\footnote{This result is actually stated informally for the special case of $k=q=\sqrt{m}=$ with $m=m_1=m_2$ in the context of a discussion of the planted clique problem in \citet{Chandrasekaran13}.}
\begin{proposition}\label{prop:kqcut}
For  any $A \in \Azkq$, we have 
$$
\sdim(A,\Ozkq) \:\: \leq \:\:  16 (k+q) + 9 \left (k \log \frac {m_1} k +  q \log \frac {m_2} q\right) \,.
$$
\end{proposition}

Upper bounding the statistical dimension of the $\kqtn$ on its atoms $\Akq$ requires more work. First, atoms with very small coefficients are likely to be more difficult to estimate than atoms with large coefficients only. In the vector case, for example, it is known that the recovery of a sparse vector $\beta$ with support $I_0$ depends on its smallest coefficient $\beta_{\min}=\min_{i \in I_0} \beta_i^2$ \citep{wainwright2009information}. The ratio between $\beta_{\min}$ and the noise level can be thought of as the worst signal-to-noise ratio for the signal $\beta$. We generalize this idea to atoms in $\Akq$ as follows.
\begin{definition}[Atom strength]\label{def:dispersion} Let  $A = ab\trans \in \Akq$ with $I_0  = \text{supp}(a)$ and $J_0 = \text{supp}(b)$. Denote $a_{\min}^2=\min_{i \in I_0} a_i^2$ and $b_{\min}^2=\min_{j \in J_0} b_j^2$. The atom strength $\gamma(a,b) \in (0,1]$ is 
\[ \gamma(a,b) \eqdef  \, (k\,a_{\min}^2)\, \wedge \, (q\,b_{\min}^2).\]
\end{definition}
Note that the atoms with maximal strength value $1$ are the elements of $\Azkq$. With this notion in hand we can now formulate an upper bound on the statistical dimension of $\Okq$:
\begin{proposition} \label{prop:gaussianWidthAtom} For  $A = ab\trans \in \Akq$ with strength $\gamma = \gamma(a,b)$, we have  
\begin{equation}\label{eq:mainbound} 
\sdim(A,\Okq) \leq \frac{322}{\gamma^2} (k+q + 1) + \frac{160}{\gamma}  (k\vee q) \log \left(m_1\vee m_2 \right ) \,.
\end{equation}
  \end{proposition}
Note that the upper bounds obtained on atoms of $\Azkq$ for $\Ozkq$ (Proposition \ref{prop:kqcut}) and $\Okq$ (Proposition \ref{prop:gaussianWidthAtom}, with $\gamma=1$) have the same rate up to $k \log k + q \log q$ which is negligible compared to $k\log m_1 +q \log m_2$ when $k \ll m_1$ and $q \ll m_2$. Note that once the support is specified, the number of degrees of freedom for elements of $\Azkq$ is $k+q-1$, which is matched up to logarithmic terms.

It is interesting to compare these estimates to the statistical dimension of the $\ell_1$ norm, the trace norm, and their combinations $\Gamma_\mu$. Table~\ref{tab:rates} summarizes the main results. 
\begin{table}
\begin{center}
\begin{tabular}{|c|c|c||c|c|}
\hline
Matrix norm & $\sdim$ & $k=\sqrt{m}$ & Vector norm &   $\sdim$ \\
\hline
\hline
$(k,q)$-trace  & $\O( (k\vee q) \log \left(m_1\vee m_2 \right ))$ & $\O(\sqrt{m} \log m)$ & $k$-support &  $\Theta(k \log \frac p k)$ \\
\hline
$(k,q)$-cut  & $\O(k \log \frac {m_1}{k} + q \log \frac {m_2}{q})$ &  $\O(\sqrt m \log m)$ & $\kappa_k$ & $\Theta(k \log \frac p k)$ \\
\hline
$\ell_1$  & $\Theta(kq~\log \frac{m_1 m_2}{k q})$ & $\Theta(m \log m)$&  $\ell_1$ & $\Theta(k \log \frac p k)$ \\
\hline
trace-norm  &$\Theta(m_1 + m_2)$ &  $\Theta(m)$& $\ell_2$ & $p$ \\
\hline
$\ell_1 +\text{trace-n.}$  & $\Upomega \big (k q  \wedge  (m_1 + m_2) \big )$ & $\Theta(m)$& elastic net & $\Theta(k \log \frac p k)$ \\
\hline
``cut-norm''  & $\O(m_1  + m_2)$ &  $\O(m)$ & $\ell_\infty$ & $p$ \\
\hline
\end{tabular}
\caption{Order of magnitude of the statistical dimension of different matrix norms for elements of $\Azkq$ (left) and of their vector norms counterpart for elements of $\cAz_k^p$ (right). 
The $\ell_1$ norm here is the element-wise $\ell_1$ norm.
The column ``$k = \sqrt m$'' corresponds to the case of the planted clique problem where $m = m_1 = m_2$ and $k = q = \sqrt m$. We use usual Landau notation with $f=\Theta(g)$ for $(f=\O(g)) \& (g=\O(f))$ and $f=\Upomega(g)$ for $g=\O(f)$. The absence of Landau notation means that the computation is exact.}\label{tab:rates}
\end{center}
\end{table}
The statistical dimension the $\ell_1$ norm on atoms in $\Azkq$ is of order $kq \log (m_1m_2/(kq))$, which is worse than the statistical dimensions of $\Okq$ and $\Ozkq$ by a factor $k \wedge q$. On $\Akq$, though, the statistical dimension of $\Okq$ increases when the atom strength decreases, while the statistical dimension of the $\ell_1$ norm is independent of it and even decreases when the size of the support decreases. As for the trace norm alone, its statistical dimension is at least of order $m_1+m_2$, which is unsurprisingly much worse that the statistical dimensions of $\Okq$ and $\Ozkq$ since it does not exploit the sparsity of the atoms. Finally, regarding the combination $\Gamma_\mu$ of the $\ell_1$ norm and of the trace norm, \citet{Oymak12} has shown that it does not improve rates up to constants over the best of the two norms. More precisely, we can derive from \citet[Theorem 3.2]{Oymak12} the following result
\begin{proposition}\label{prop:from_oymak}
There exists $M>0$ and $C>0$ such that for any $m_1, m_2, k, q\geq M$ with $m_1/k \geq M$ and $m_2/q \geq M$, for any $A \in \Akq$ and for any $\mu \in [0,1]$, the following holds:
$$
\sdim\br{A,\Gamma_\mu} \geq C\, \zeta(a,b)\, \big ( \, (kq) \wedge (m_1+m_2-1) \big ) -2\,,
$$
with
$$
\zeta(a,b)=1-\Big (1-\frac{\|a\|_1^2}{k}\Big) \Big (1-\frac{\|b\|_1^2}{q}\Big ) \,.
$$
\end{proposition}
Note that $\zeta(a,b)\leq 1$ with equality if either $a \in \cAz_k^{m_1}$ or $b \in \cAz_q^{m_2}$, so in particular $\zeta(a,b)=1$ for $ab\trans \in \Azkq$. In that case, we see that, as stated by \citet{Oymak12}, $\Gamma_\mu$ does not bring any improvement over the $\ell_1$ and trace norms taken imdividually, and in particular has a worse statistical dimension than $\Okq$ and $\Ozkq$.

\subsubsection{The vector case}\label{sec:vectorcase}
We have seen in Section~\ref{sec:boundsdim} that the statistical dimension of the $\kqtn$ and of the $\kqcut$ norm were smaller than that of the $\ell_1$ and the trace norms, and of their combinations, meaning that theoretically they are more efficient regularizers to recover rank-one sparse matrices. In this section, we look more precisely at these properties in the vector case ($m_2=q=1$), and show that, surprisingly, the benefits are lost in this case.

Remember that, in the vector case, $\Okq$ boils down to the $k$-support norm $\theta_k$ (\ref{eq:ksuppsq}), while $\Ozkq$ boils down to the norm $\kappa_k$ (\ref{eq:kappa}). For the later, we can upper bound the statistical dimension at a $k$-sparse vector by specializing Proposition~\ref{prop:kqcut} to the vector case, and also derive a specific lower bound as follows:
\begin{proposition}\label{prop:kappakStatDim}
For any $k$-sparse vector $a \in \cAz_k^p$, 
$$
\frac{k}{2\pi} \log \br{\frac{p-k}{k+1} } \leq \sdim(a,\kappa_k) \leq 9 k  \log \frac {p}{k} + 16 (k+1) \,.
$$
\end{proposition}
From the explicit formulation of $\theta_k$ (\ref{eq:ksuppsq}) we can derive an upper bound of the statistical dimension of $\theta_k$ on any sparse vector with at least $k$ non-zero coefficients:
\begin{proposition}\label{prop:statisticaldimksupport}
For any $s\geq k$, the statistical dimension of the $k$-support norm $\theta_k$ at an $s$-sparse vector $w \in \RR^p$  is bounded by
\begin{equation}\label{eq:StatisticalDimKsupportNorm}
\sdim( w,  \theta_k)\leq  \frac{5}{4}s + 2 \left \{  \frac{(r+1)^2 \nm{\tilde{w}_{I_2}}_2^2}{\nm{\tilde{w}_{I_1}}_1^2}  + |I_1| \right \} \log \frac{p}{s}\,,
\end{equation}
where $\tilde{w}\in\RR^p$ denotes the vector with the same entries as $w$ sorted by decreasing absolute values, $r$ is as defined in equation~(\ref{eq:ksuppsq}), $I_2=\sqb{1,k-r-1}$ and $I_1=\sqb{k-r,s}$. In particular, when $s=k$, the following holds for any atom $a \in \cA_k^p$ with strength $\gamma = k a_{\min}^2$:
\begin{equation}\label{eq:sdimthetak}
\sdim( a,  \theta_k)\leq  \frac 54 k + \frac{2k}{\gamma} \log\frac{p}{k}\,.
\end{equation}
\end{proposition}
We note that (\ref{eq:sdimthetak}) has the same rate but tighter constants than the general upper bound (\ref{eq:mainbound}) specialized to the vector case. In particular, this suggests that the $\gamma^{-2}$ term in (\ref{eq:sdimthetak}) may not be required. In the lasso case ($k=1$), we recover the standard bound \citep{Chandrasekaran12}:
\begin{equation}\label{eq:boundlasso}
\sdim(w,\theta_k) \leq \frac{5}{4}s + 2 s \log \frac{p}{s} \,,
\end{equation}
which is also reached by $\theta_k$ on an atom $a\in\cAz_k^p$ because in that case $\gamma=1$ in (\ref{eq:sdimthetak}). On the other hand, for general atoms in $\cA_k^p$ the upper bound (\ref{eq:sdimthetak}) is always worse than the upper bound for the standard Lasso (\ref{eq:boundlasso}), and more generally the upper bound for general sparse vectors (\ref{eq:StatisticalDimKsupportNorm}) is also never better than the one for the Lasso. Although these are only upper bounds, this raises questions on the utility of the $k$-support norm compared to the lasso to recover sparse vectors.

The statistical complexities of the different regularizers in the vector case are summarized in Table~\ref{tab:rates}. We note that, contrary to the low-rank sparse matrix case, the $\ell_1$-norm, the $k$-support norm, and the norm $\kappa_k$ all have the same statistical dimension up to constants. Note that the tangent cone of the elastic net equals the tangent cone of the $\ell_1$-norm in any point (because the tangent cone of the $\ell_2$ norm is a half space that always contains the tangent cone of the $\ell_1$-norm) so that the elastic net has always the exact same statistical dimension as the $\ell_1$-norm.

\section{Algorithms}\label{sec:algos}

As seen in Section~\ref{sec:applications}, many problems involving sparse low-rank matrix estimation can be formulated as optimization problems of the form:
\BEA
\label{eq:minregrisk}
\min_{Z \in \RR^{m_1 \times m_2}} \risk(Z)+ \lambda \Okq(Z).
\EEA
Unfortunately, although convex, this problem may be computationally challenging (Section~\ref{sec:nphard}). In this section, we present a working set algorithm to approximately solve such problems in practice when $\risk$ is differentiable.

\subsection{A working set algorithm}\label{sec:acset}
Given a set $\Scal \subset \Gk \times \Gq$ of pairs of row and column subsets, let us consider the optimization problem:
\begin{equation}\tag{$\mathcal{P}_{\Scal}$}
\min_{\br{\aij}_{(I,J) \in \Scal}} \cbr{ \risk \Bigg ( \sum_{(I,J) \in \Scal} \aij \Bigg )+ \lambda\!\!\! \sum_{(I,J) \in \Scal} \!\! \nmtr{\aij} ~:~ \forall (I,J)\in\Scal,\, \Supp(\aij) \subset I \times J}\,.
\label{restricted_pb}
\end{equation}
Let $(\widehat\aij)_{(I,J) \in \Scal}$ be a solution of this optimization problem.
Then, by the characterization of $\Okq(Z)$ in (\ref{eq:defOmegaInfTraceNorm}), $Z=\sum_{(I,J) \in \Scal} \widehat\aij$ is the solution of (\ref{eq:minregrisk}) when $\Scal=\Gk \times \Gq$.
Clearly, it is still the solution of (\ref{eq:minregrisk}) if $\Scal$ is reduced to the set of non-zero matrices $\widehat\aij$ at optimality often called {\it active} components.

We propose to solve problem (\ref{eq:minregrisk}) using a so-called working set algorithm which solves a sequence of problems of the form (\ref{restricted_pb}) for a growing sequence of working sets $\Scal$, so as to keep a small number of non-zero matrices $\aij$ throughout. Working set algorithms \citep[][Chap.~6]{Bach2011Optimization} are typically useful to speed up algorithm for sparsity inducing regularizer; they have been used notably in the case of the overlapping group Lasso of 
\citet{Jacob2009Group} which is also naturally formulated \emph{via} latent components.

To derive the algorithm we write the optimality condition for (\ref{restricted_pb}):
$$
\forall (I,J)\in\Scal\,,\quad  \nabla \risk(Z)_{IJ} \in -\lambda \partial \nmtr{Z^{(IJ)}}\,.
$$
From the characterization of the subdifferential of the trace norm \citep{Watson1992Characterization}, writing $\aij=\uij \sij \vij$ the SVD of $\aij$, this is equivalent to, for all $(I,J)$ in $\Scal$,
\begin{align}
\text{either} \quad &\aij \! \neq \!  0 \quad \text{and}    && \nabla \risk(Z)_{IJ} = -\lambda \br{ \uij \vij^\top + A}  \nonumber \\
& && \qquad\text{ with } \nmop{A}\leq 1\text{ and }A \uij = A\trans \vij=0 \,,\label{KKT1}\\
\text{or} \quad  &\aij \! = \!  0 \quad \text{and}  &&\nmop{\nabla \risk(Z)]_{IJ}} \leq \lambda\,.\label{KKT2}
\end{align}
The principle of the working set algorithm is to solve problem (\ref{restricted_pb}) for the current set $\Scal$ so that (\ref{KKT1}) and  (\ref{KKT2}) are (approximately) satisfied for $(I,J)$ in $\Scal$, and to check subsequently if there are any components not in $\Scal$ which violate (\ref{KKT2}). If not, this guarantees that we have found a solution to problem (\ref{eq:minregrisk}), otherwise the new pair $(I,J)$ corresponding to the most violated constraint is added to $\Scal$ and problem (\ref{restricted_pb}) is initialized with the previous solution and solved again. The resulting algorithm is Algorithm~\ref{aaalgo} (where the routine SSVDTPI is described in the next section).
Problem (\ref{restricted_pb}) is solved easily using the approximate block coordinate descent of \citet{tseng2009coordinate} \citep[see also][Chap.~4]{Bach2011Optimization}, which consists in iterating proximal operators. The modifications to the algorithm to solve problems regularized by the norm $\Omega_{k, \succeq}$ are relatively minor (they amount to replace the trace norms by penalization of the trace of the matrices $\aij$ and by positive definite cone constraints) and we therefore do not describe them here.

Determining efficiently which pair $(I,J)$ possibly violates condition (\ref{KKT2}) is in contrast a more difficult problem that we discuss next. 

\begin{algorithm}
\caption{Active set algorithm}
\begin{algorithmic}
\REQUIRE $\risk$, tolerance $\epsilon>0$, parameters $\lambda,k,q$
\STATE Set $\Scal=\varnothing, Z=0$
\WHILE{$c=\text{\texttt{true}}$}
\STATE Recompute optimal values of $Z$, $(\aij)_{(I,J)\in \Scal}$ for (\ref{restricted_pb}) using warm start
\STATE $(I,J) \leftarrow \texttt{SSVDTPI}(\nabla \risk(Z),k,q,\epsilon)$
\IF{$\|[\nabla \risk(Z)]_{I,J}\|_{\op}>\lambda$}
\STATE $\Scal \leftarrow \Scal \cup \{(I,J)\}$
\ELSE
\STATE $c \leftarrow \text{\texttt{false}}$
\ENDIF
\ENDWHILE
\RETURN $Z$, $\Scal$, $(\aij)_{(I,J)\in \Scal}$
\end{algorithmic}
\label{aaalgo}
\end{algorithm}

\subsection{Finding new active components}
Once (\ref{restricted_pb}) is solved for a given set $\Scal$, (\ref{KKT1}) and (\ref{KKT2}) are satisfied for all $(I,J)\in\Scal$. Note that (\ref{KKT1}) implies in particular that $\nmop{\nabla \risk(Z)]_{IJ}} = \lambda$ when $\aij\neq 0$ at optimality. Therefore, (\ref{KKT2}) is also satisfied for all $(I,J) \notin \Scal$ if and only if
\begin{equation}\label{eq:optnewcomponent}
\max_{(I,J) \in \Gk \times \Gq} \|[\nabla \risk(Z)]_{IJ}\|_\op \leq \lambda \,,
\end{equation}
and if this is not the case then any $(I,J)$ that violates this condition is a candidate to be included in $\Scal$. This corresponds to solving the following sparse singular value problem
\begin{equation}
\tag*{$(k,q)$-linRank-1} \label{linRank1}
\max_{a,b} \quad a\trans \nabla \risk(Z) b \quad \st \quad ab\trans \in \Akq\,.
\end{equation}
This problem is unfortunately NP-hard since rank $1$ sparse PCA problem is a particular instance of it (when $\nabla \risk(Z)$ is replaced by a covariance matrix), and we therefore cannot hope to solve it exactly with efficient algorithms. Still, sparse PCA has been the object of a significant amount of research, and several relaxations and other heuristics have been proposed to solve it approximately.
In our numerical experiments we use a truncated power iteration (TPI) method, also called TPower, GPower or CongradU in the PSD case \citep{journee10, yuan13, luss12}, which has been proved recently by \citet{yuan13} to provide accurate solution in reasonable computational time under RIP type of conditions. Algorithm \ref{alg:TPI} provides a natural generalization of this algorithm to the non-PSD case. The algorithm follows the steps of a power method, the standard method for computing leading singular vectors of a matrix, with the difference that at each iteration a truncation step is use.  We denote the truncation operator by $T_k$. It consists of keeping the $k$ largest components (in absolute value) and setting the others to $0$. 
\begin{algorithm}
\caption{SSVDTPI: Bi-truncated power iteration for \ref{linRank1}}
\begin{algorithmic}
\REQUIRE $A \in \RR^{m_1 \times m_2}$, $k,q$ and tolerance $\epsilon >0$
\STATE Pick a random initial point $b^{(0)} \sim \mathcal N(0,I_{m_2})$ and let  
\WHILE{$| a^{(t)\top} Ab^{(t)} - a^{(t-1)\top} Ab^{(t-1)} | / |a^{(t-1)~{\scriptscriptstyle \top}} Ab^{(t-1)} | > \epsilon$}
\STATE  $a \leftarrow  A b^{(t)}$   \quad $\backslash \backslash$ {\tt Power}
\STATE $a \leftarrow  T_k(a)$       \quad $\backslash \backslash$   {\tt Truncate}
\STATE $b \leftarrow  A\trans a$   \quad $\backslash \backslash$ {\tt Power}
\STATE $b \leftarrow  T_q(b)$       \quad $\backslash \backslash$   {\tt Truncate}
\STATE   $a^{(t+1)} \leftarrow  a / \|a\|_2$ and  $b^{(t+1)} \leftarrow  b / \|b\|_2$  $\backslash \backslash$ {\tt Normalize} 
\STATE $t \leftarrow t+1$
\ENDWHILE
\STATE $I \leftarrow \text{Supp}(a^{(t)})$ and $J \leftarrow \text{Supp}(b^{(t)})$
\RETURN $(I,J)$
\end{algorithmic}
\label{alg:TPI}
\end{algorithm}
Note that Algorithm \ref{alg:TPI} may fail to find a new active component for Algorithm~\ref{aaalgo} if it finds a local maximum of (\ref{linRank1}) smaller than $\lambda$, and therefore result in the termination of Algorithm~\ref{aaalgo} on a suboptimal solution. On the positive side, note that Algorithm~\ref{aaalgo} is robust to some errors of Algorithm \ref{alg:TPI}. For instance, if an incorrect component is added to 
$\Scal$ at some iteration, but the correct components are identified later, 
Algorithm~\ref{aaalgo}  will eventually shrink the incorrect components to $0$.
One of the causes of failure of TPI type of methods is the presence of a 
large local maximum in the sparse PCA problem corresponding to a suboptimal component; 
incorporating this component in $\Scal$ will reduce the size of that local maximum,
thereby increasing the chance of selecting a correct component the next time around.
  
\subsection{Computational cost}
Note that when $m_1,m_2$ are large, solving \ref{restricted_pb} involves the minimizations of trace norms of matrices of size $k \times q$ which, when $k$ and $q$ are small compared to $m_1$ and $m_2$ have low computational cost.  
The bottleneck for providing a computational complexity of the algorithm is the \ref{linRank1} step. It has been proved by \cite{yuan13} that under some conditions the problem can be solved in linear time. If the conditions hold at every step of gradient, the overall cost of an iteration can be cast into the cost of evaluating the gradient and the evaluation of thin SVDs: $O(k^2q)$. Evaluating the gradient has a cost dependent on the risk function $\risk$. This cost for usual applications is $O(m_1m_2)$. So assuming the RIP conditions required by \citet{yuan13} hold, the cost of Algorithm \ref{alg:TPI} is dominated by matrix-vector multiplications so of the order $O(m_1m_2)$. The total cost of the algorithm for reaching a $\delta$-accurate solution is therefore $O((m_1m_2 + k^2 q)/\delta)$. However the worst case complexity of the algorithm is non-polynomial as \ref{linRank1} is non-polynomial in general. We would like to point out that in our numerical experiments a warm start with singular vectors and multiple runs of the algorithm \ref{linRank1} keeping track of the highest found variance has provided us a very fast and reliable solver. Further discussion on this step go beyond the scope of this work. 

%
%

%

\section{Numerical experiments}\label{sec:num}
In this section we report experimental results to assess the performance of sparse low-rank matrix estimation using different techniques. We start in Section~\ref{sec:numtoy1} with simulations aiming at validating the theoretical results on statistical dimension of $\Okq$ and assessing how they generalize to matrices with $\kqrank$ larger than $1$. In Section~\ref{sec:numtoy2} we compare several techniques for sparse PCA on simulated data.

\subsection{Empirical estimates of the statistical dimension.}\label{sec:numtoy1}

In order to numerically estimate the statistical dimension $\sdim(Z,\Omega)$ of a regularizer $\Omega$ at a matrix $Z$, we add to $Z$ a random Gaussian noise matrix and observe $Y=Z+\sigma G$ where $G$ has normal i.i.d. entries following $\mathcal N(0,1)$. We then denoise $Y$ using (\ref{eq:denoiser}) to form an estimate $\hat{Z}$ of $Z$. For small $\sigma$, the normalized mean-squared error (${\sf NMSE}$) defined as
$$
{\sf NMSE}(\sigma) \eqdef \frac{\E \nmF{\hat{Z} - Z}^2}{\sigma^2}
$$
is a good estimate of the statistical dimension, since \citet{Oymak2013Sharp} show that
$$
\sdim(Z,\Omega) = \lim_{\sigma \rightarrow 0} {\sf NMSE}(\sigma)\,.
$$
Numerically, we therefore estimate $\sdim(Z,\Omega)$ by taking $\sigma = 10^{-4}$ and measuring the empirical ${\sf NMSE}$ averaged over 20 repeats. We consider square matrices with $m_1=m_2=1000$, and estimate the statistical dimension of $\Okq$, the $\ell_1$ and the trace norms at different matrices $Z$. The constrained denoiser (\ref{eq:denoiser}) has a simple close-form for the $\ell_1$ and the trace norm. For $\Okq$, it can be obtained by a series of proximal projections (\ref{eq:proxprojection}) with different parameters $\lambda$ until $\Okq(\hat{Z})$ has the correct value $\Okq(Z)$. Since the noise is small, we found that it was sufficient and faster to perform a $\kqsvd$ of $Y$ by solving  (\ref{eq:proxprojection}) with a small $\lambda$, and then apply the $\ell_1$ constrained denoiser to the set of $(k,q)$-sparse singular values.

We first estimate the statistical dimensions of the three norms at an atom $Z\in\Azkq$, for different values of $k=q$. Figure~\ref{fig:scalingK} (top left) shows the results, which confirm the theoretical bounds summarized in Table~\ref{tab:rates}. The statistical dimension of the trace norm does not depend on $k$, while that of the $\ell_1$ norm increases almost quadratically with $k$ and that of $\Okq$ increases linearly with $k$. As expected, $\Okq$ interpolates between the $\ell_1$ norm (for $k=1$) and the trace norm (for $k=m_1$), and outperforms both norms for intermediary values of $k$. This experiments therefore confirms that our upper bound (\ref{eq:mainbound}) on $\sdim(Z,\Okq)$ captures the correct order in $k$, although the constants can certainly be much improved, and that Algorithm~\ref{aaalgo} manages, in this simple setting, to correctly approximate the solution of the convex minimization problem.
 \begin{figure}
\hspace{1.5cm}\includegraphics[width = 7cm]{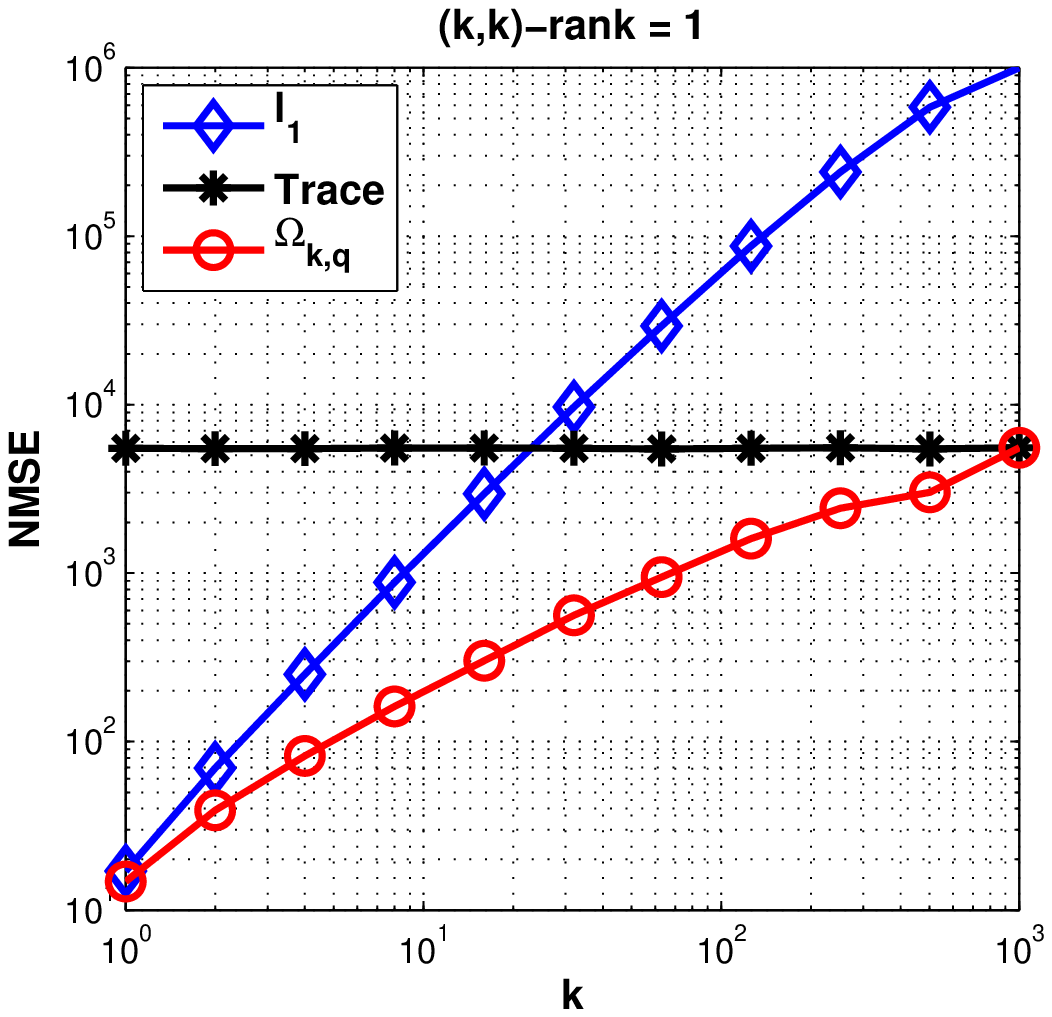}\includegraphics[width = 7cm]{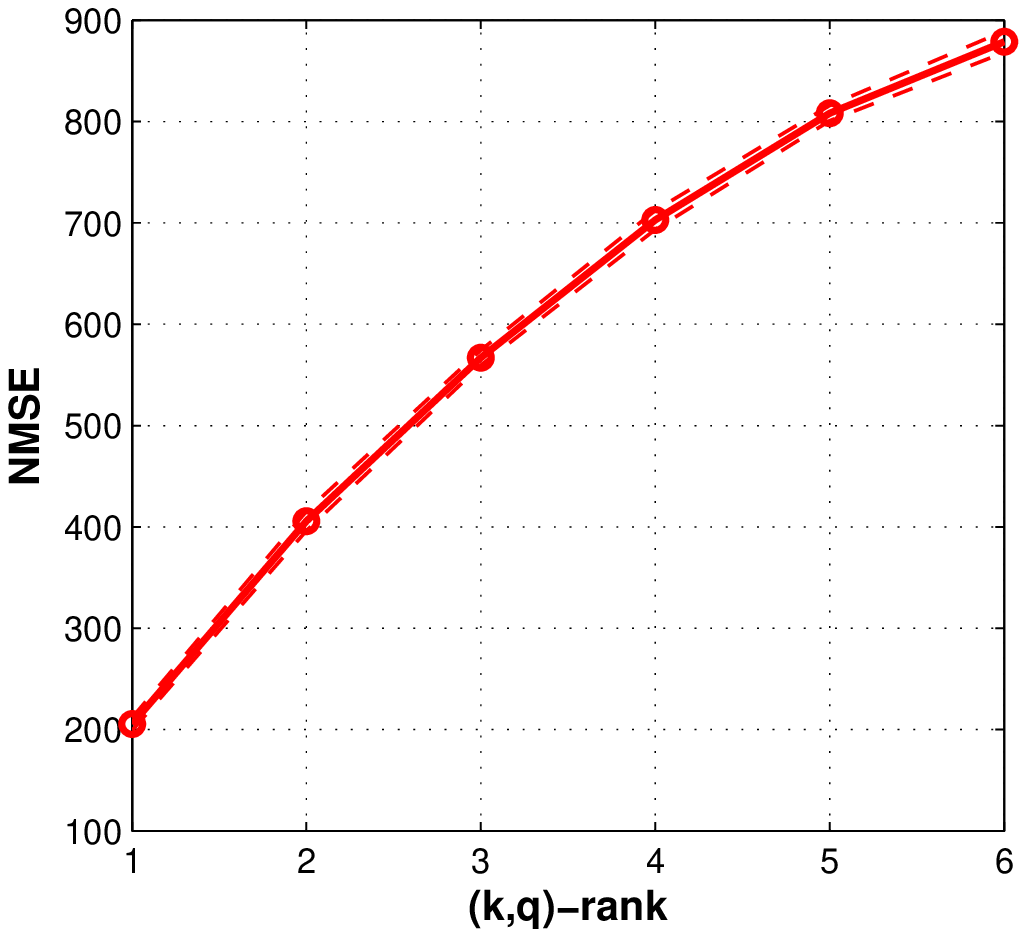}\\

\hspace{1.5cm}\includegraphics[width = 7cm]{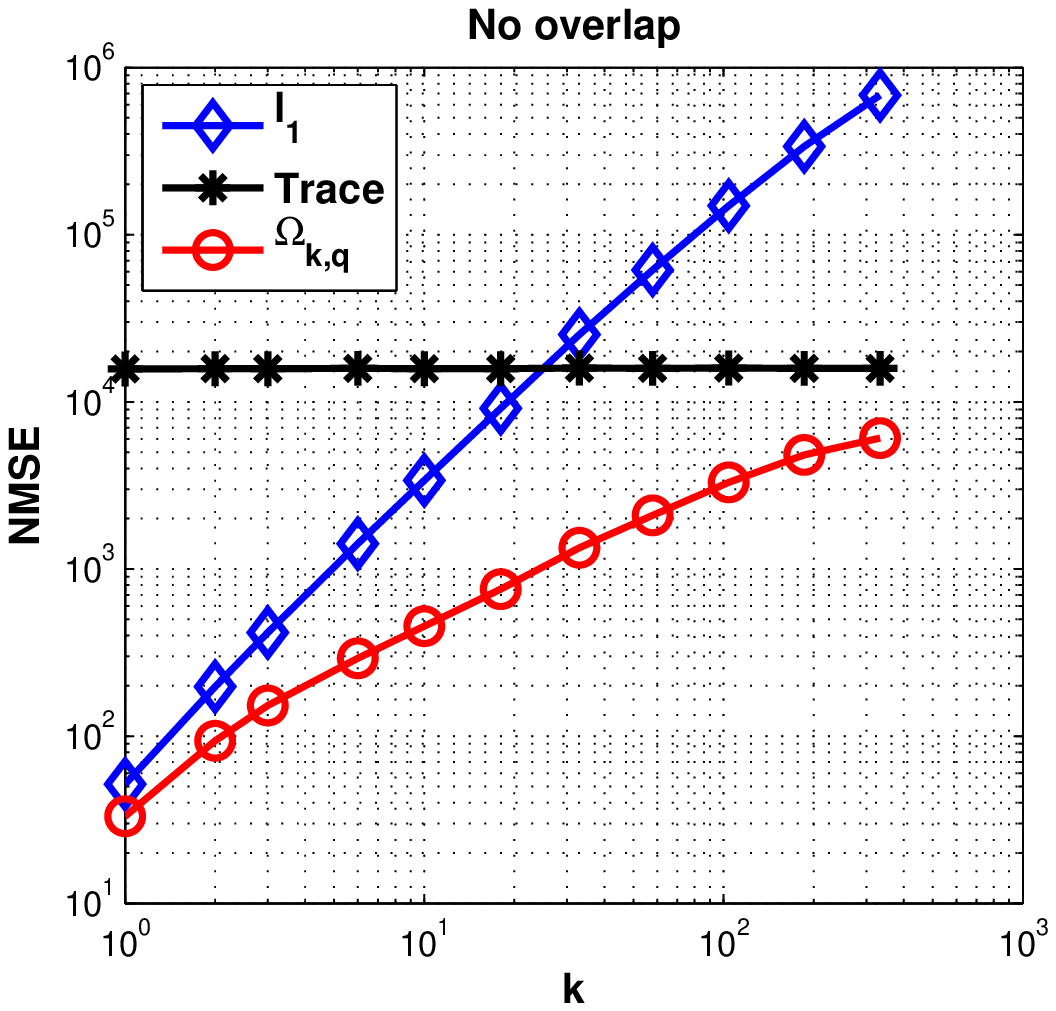}\includegraphics[width = 7cm]{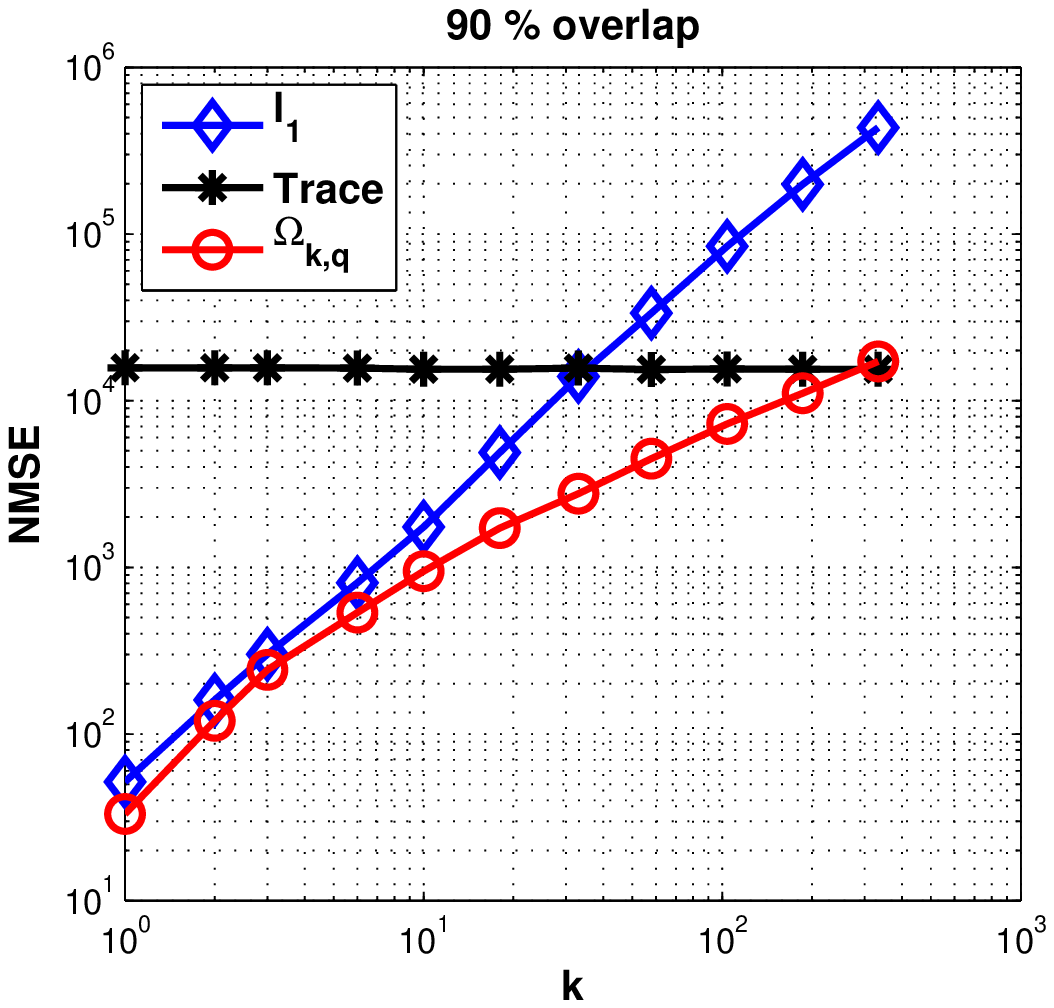}\\
\caption{Estimates of the statistical dimensions of the $\ell_1$, trace and $\Okq$ norms at a matrix $Z\in\RR^{1000\times 1000}$ in different setting. Top left: $Z$ is an atom in $\cAz_{k,k}$ for different values of $k$. Top right: $Z$ is a sum of $r$ atoms in $\cAz_{k,k}$ with non-overlapping support, with $k=10$ and varying $r$. Bottom left: $Z$ is a sum of $3$ atoms in $\cAz_{k,k}$ with non-overlapping support, for varying $k$. Bottom right: $Z$ is a sum of $3$ atoms in $\cAz_{k,k}$ with overlapping support, for varying $k$.}
\label{fig:scalingK}
\end{figure}

Second, we estimate the statistical dimension of $\Okq$ on matrices with $\kqrank$ larger than $1$, a setting for which we proved no theoretical result. Figure~\ref{fig:scalingK} (top left) shows the numerical estimate of $\sdim(Z,\Okq)$ for matrices $Z$ which are sums of $r$ atoms in $\cAz_{k,k}$ with non-overlapping support, for $k=10$ and varying $r$. We observe that the increase in statistical dimension is roughly linear in the $\kqrank$. For a fixed $\kqrank$ of $3$, the bottom plots of Figure~\ref{fig:scalingK} compare the estimated statistical dimensions of the three regularizers on matrices $Z$ which are sums of $3$ atoms in $\cAz_{k,k}$ with non-overlapping (bottom left) or  overlapping (bottom right) supports. The shapes of the different curves are overall similar to the rank $1$ case, although the performance of $\Okq$ degrades as the supports of atoms overlap. In both cases, $\Okq$ consistently outperforms the two other norms. Overall these experiments suggest that the statistical dimension of $\Okq$ at a  linear combination of $r$ atoms increases as $C r \left (k \log m_1 + q \log m_2 \right ) $ where the coefficient $C$ increases with the overlap among the supports of the atoms.

\subsection{Comparison of algorithms for sparse PCA}\label{sec:numtoy2}

 \begin{figure}
\begin{center}
\includegraphics[width=15cm,trim=0cm 3.5cm 1.5cm 4cm]{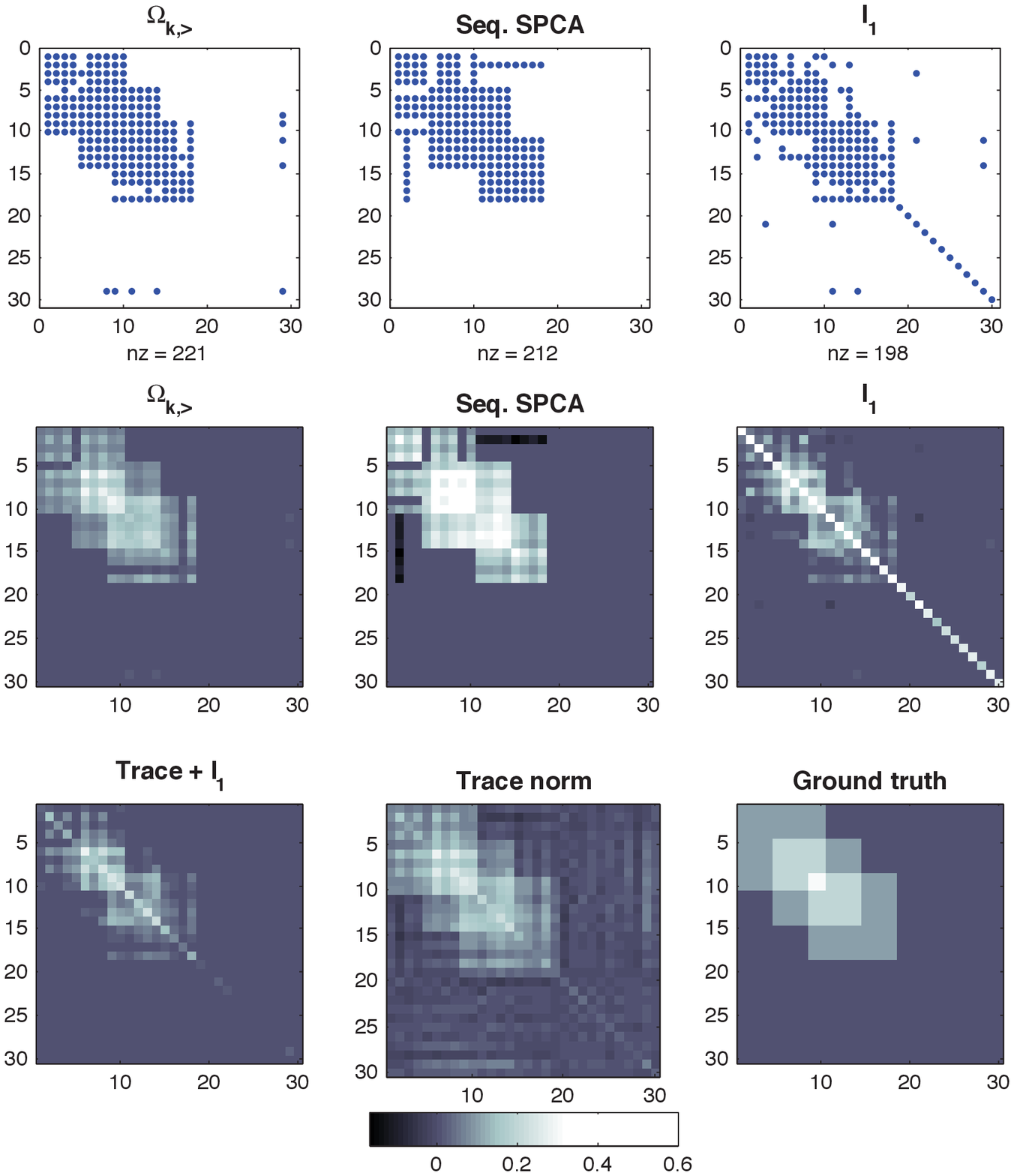}\caption{Sparse PCA example. The first row shows the supports found by our method (left) by sequential sparse PCA (middle) and element wise thresholding of the sample covariance matrix. Other plots contain heatmaps of the estimated covariance matrix using different methods, and the ground truth $\Sigma^\star$ in the lower right hand side. }
\label{fig:sparsePCAexample}
\end{center}
\end{figure}

In this section we compare the performance of different algorithms in estimating a sparsely factored covariance matrix that we denote $\Sigma^\star$. The observed sample consists of $n$ random vector vectors generated i.i.d. according to $\mathcal N(0,\Sigma^\star + \sigma^2 \id_p)$, where $(k,k)\mbox{-}\rank(\Sigma^\star) = 3$.
 The matrix $\Sigma^\star$ is formed by adding 3 blocks of rank 1, $\Sigma^\star = a_1 a_1\trans + a_2 a_2\trans  +  a_3 a_3\trans $, having all the same sparsity $\|a_i\|_0 = k = 10$, $3 \times 3$ overlaps and nonzero entries equal to $1/\sqrt k$. See Figure \ref{fig:sparsePCAexample}, bottom right plot for a representation of the ground truth $\Sigma^\star$.
 The noise level $\sigma = 0.8$ is set in order to make the signal to noise ratio below the level $\sigma = 1$ where a spectral gap appears and makes the spectral baseline (penalizing the trace of the PSD matrix) work. In our experiments the number of variables is $p = 200$ and $n = 80$ points are observed. To estimate the true covariance matrix from the noisy observation, first the sample covariance matrix is formed as
\[ \hat{\Sigma}_n = \frac 1n \sum_{i=1}^n x_i x_i\trans\,,\]
and given as input to various algorithms which provide a new estimate $\hat\Sigma$. The methods we compared are the following:
\begin{itemize}
\item {\bf Raw sample covariance. } The most basic is to output $ \hat{\Sigma}_n$ as the estimate of the covariance, which is not accurate due to presence of noise and underdeterminedness $n<p$.
\item {\bf Trace penalty on the PSD cone.}  This spectral algorithm solves the following optimization problem in the cone of PSD matrices:
\[ \min_{Z \succeq 0} \frac 12 \nmF{Z-\hat{\Sigma}_n}^2 + \lambda \trace Z\,.\]
\item {\bf $\ell_1$ penalty. } In order to approximate the  sample covariance $\hat{\Sigma}_n$ by a sparse matrix a basic idea is to soft-threshold it element-by-element. This is equivalent to solving the following convex optimization problem:
\[ \min_{Z } \frac 12 \nmF{Z-\hat{\Sigma}_n}^2 + \lambda \| Z\|_1\,.\]
\item {\bf Trace + $\ell_1$ penalty. } The restriction of $\Gamma_\mu$ to the PSD cone, which is equivalent to solving the following SDP
\[ \min_{Z \succeq 0} \frac 12 \nmF{Z-\hat{\Sigma}_n}^2 + \lambda \Gamma_\mu( Z)\,.\]
This approach needs to tune two parameters $\lambda>0,\mu\in [0,1]$.
\item {\bf Sequential sparse PCA. } This is the standard way of estimating multiple sparse principal components which consists of solving the problem for a single component at each step $t = 1\cdots r$, and deflate to switch to the next $(t+1)$st component.  The deflation step used in this algorithm is the orthogonal projection \[Z_{t+1} = \left ( \id_p - u_tu_t\trans \right )Z_t  \left ( \id_p - u_tu_t\trans \right )\,.\] The tuning parameters for this approach are the sparsity level $k$ and the number of principal components $r$.
\item {\bf  $\Omega_{k,\succeq}$ penalty.} The following optimization problem, which is a proximal operator computation, is solved using the active set algorithm:
\[ \min_{Z \succeq 0} \frac 12 \nmF{Z-\hat{\Sigma}_n}^2 + \lambda \Omega_{k,\succeq}( Z)\,,\] with $\Omega_{k,\succeq}$ the  gauge associated with $\cA_{k,\succ}$ already introduced in Section \ref{sec:applicationSparsePCA}. The two parameters of this method are $\lambda>0$ and $k\in \NN\backslash \{ 0\}$.
\end{itemize}
\begin{table}
\begin{center}
\begin{tabular}{|c|c|c|c|c|c|c|}
\hline
Sample covariance  & Trace & $\ell_1$ & Trace + $\ell_1$ & Sequential & $\Omega_{k,\succeq}$  \\
\hline
4.20 $\pm$  0.02 & 0.98 $\pm$ 0.01 & 2.07  $\pm$  0.01 &   0.96  $\pm$ 0.01 &  0.93 $\pm$ 0.08  & \bf 0.59  $\pm$ 0.03 \\
 \hline
  \end{tabular}
  \end{center}
  \caption{Relative error of covariance estimation with different methods.}\label{tab:algoComparison}
  \end{table}
We report the relative errors $\nmF{\hat \Sigma - \Sigma^\star }/\nmF{\Sigma^\star} $ over 10 runs of our experiments in Table \ref{tab:algoComparison}, and a representation of the estimated matrices can be found in Figure \ref{fig:sparsePCAexample}. We observe that sparse PCA methods using $\Omega_{k,\succeq}$ and also the sequential method using deflation steps outperform spectral and $\ell_1$ baselines. In addition, penalizing $\Omega_{k,\succeq}$ is superior to the sequential approach. This was expected since our algorithm minimizes a loss function that is close to the test errors reported, whereas the sequential scheme does not optimize a well-defined objective.


\section{Conclusion}
In this work, we proposed two new convex penalties, the $\kqtn$ and the $\kqcut$ norm, specifically tailored to the estimation of low-rank matrices with sparse factors. Our motivation for proposing such convex formulations for sparse low-rank matrix inference was twofold. First, it allowed us to consider algorithmic schemes that are better understood when a problem is formulated as a convex optimization problem, even though the complexity of solving the problem exactly remains super-polynomial. Second, using convex geometry allowed us to provide sample complexity and statistical guarantees, and notably to show that the proposed estimators have much better statistical dimension than more standard convex combinations of the $\ell_1$ and trace norms. We observed that the improvement exists only for matrices: for sparse vectors, using our penalty (which boils down to the $k$-support norm in this case) does not improve over the standard $\ell_1$ norm, in terms of statistical dimension increase rate. 

One limitation of this work is that we assume that the sparsity of the factors is known and fixed. Lifting this constraint and investigating procedures that can adapt to the size of the blocks (like the $\ell_1$ norm adapts to the size of the support) is an interesting direction for future research. Another interesting direction is to use the nuclear norm formulation of the $(k,q)$-trace norm as in Lemma \ref{lem:nuclearIsAtomic}  to optimize the regularized problem. 

\acks{We would like to thank Francis Bach for interesting discussions related to this work. This work was supported by the European Research Council (SMAC-ERC-280032) and by
by Agence Nationale de la Recherche (ANR-13-MONU-005-10)}

\bibliography{biblio}

\appendix
\section{Proofs of results in Sections~\ref{sec:tighRelaxations} and \ref{sec:applications}.}\label{sec:OmegakqBasicProperties}

\begin{proof}[Proposition~\ref{prop:kqproperties}]

To prove the first claim, note that a matrix of the form $a b\trans$ for $a\in\cA_{k}^{m_1}$ and $b\in\cA_q^{m_2}$ has at most $kq$ non-zero terms. Therefore, the decomposition of a matrix with no null entries as a linear combination of such sparse matrices must count at least $\frac{m_1m_2}{kq}$ terms, which is larger than $m_1 \vee m_2$ when $kq \leq m_1\wedge m_2$.

To prove second claim, consider for the matrix $Z=\ones \ones\trans \in \RR^3$ the problem of finding a decomposition of $Z$ which attains the $(2,2)$-rank.
 It is impossible to write $Z$ as the sum of two $(2,2)$-sparse matrices, because it would then have at most $8$ non-zero coefficients. But we have the decomposition.
 $$
\BSM
\rule{0mm}{.6pc} 1 & 1 & 1\\
\rule{0mm}{.6pc} 1 & 1 & 1\\
\rule{0mm}{.6pc} 1 & 1 & 1
\ESM
=
\BSM
2 & 1 & 0\\
1 & \frac12 & 0\\
0 & 0 & 0
\ESM
+
\BSM
0 & 0 & 1\\
0 & \frac12 & 1\\
0 & 1 & 2
\ESM
-
\BSM
\:\:1 & \:\:0 & -1\\
\:\:0 \rule{0mm}{.5pc} & \:\:0 & \:\:0\\
-1 \rule{0mm}{.5pc} & \:\:0 & \:\:1
\ESM,
$$
which shows that the $(2,2)$-rank of $Z$ is $3$. Note that this decomposition is not unique: given that $Z$ is invariant by any of the 6 permutations of the rows and any of the 6 permutations of the columns, $Z$ admits at least 36 different decompositions attaining the $(2,2)$-rank. 

Now, observe that the decomposition proposed above for $Z=\ones \ones\trans \in \RR^3$ yields $3$ left- and right-$(2,2)$-sparse factors that are obviously not orthogonal. It can actually be shown by systematic enumeration of all possible cases that it is impossible to find any $(2,2)$-sparse decomposition of $Z$ with left or right factors that are orthogonal.

\end{proof}

\begin{proof}[Proposition~\ref{prop:softkqproperties}]

To prove the first claim, let us consider the matrix $Z=\ones \ones\trans \in \RR^3$. We showed in the proof of Proposition~\ref{prop:kqproperties} above that its $(2,2)$-rank is equal to $3$. We now show that the number of its $(k,q)$-sparse singular value is 9, and thus much larger than 3. For that purpose, we express any $(2,2)$-SVD of $Z$ as a minimizer of~(\ref{eq:defOmegaInfTraceNorm}), and write the corresponding Lagrangian:
$$
\mathcal{L}((\aij)_{I,J},K)=\sum_{I, J \in \Gcal_2 } \nmtr{\aij}+\tr\Big (K\trans \Big (Z- \sum_{I, J \in \Gcal_2} \aij \Big) \Big ) \,,
$$ 
where $(\aij)_{I,J}$ and $K$ are the primal and dual variables. It is easy to check that the dual solution is the unique subgradient of $\Omega_{2,2}$ at $Z$ which is equal to $K^*=\frac{1}{2}Z$.
But any primal solution must satisfy $\tr({K^*}\trans\aij)=\nmtr{\aij}$. This implies that any primal solution $(\aij)_{I,J}$  satisfies $\aij \propto \ones_I \ones_J\trans$. Then, one can check that  $((\frac{1}{2})\ones_I \ones_J\trans)_{I,J \in \Gcal_2}$ forms a basis of $\RR^{3 \times 3}$ so that any matrix $Z$ admits a unique set of decomposition coefficients on that basis. This proves that the unique solution of~(\ref{eq:defOmegaInfTraceNorm}) is the one such that $\aij=\frac{1}{4} \ones_I \ones_J\trans$ for all pairs $(I,J)\in \Gcal_2 \times \Gcal_2$.
This unique $\kqsvd$ is composed of $9$ terms which is strictly larger than its $\kqrank$, the latter being equal to 3.

To prove the second claim, let us consider the $(2,2)$-SVDs of $Z=\frac{1}{2}\ones \ones\trans \in \RR^4$. By proposition~\ref{prop:norm_ineqs}, $\frac{1}{2}\|Z\|_1\leq \Omega_{2,2}(Z)$, but $\frac{1}{2}\|Z\|_1=4$ and $2Z=(\ones_{\{1,2\}}+\ones_{\{3,4\}})(\ones_{\{1,2\}}+\ones_{\{3,4\}})\trans$  which shows that $\Omega_{2,2}(Z)\leq 4$. So $\Omega_{2,2}(Z)=4$. Considering that there are 3 ways to partition $\{1,2,3,4\}$ into sets of cardinality $2$, $Z$ admits at least $9$ different optimal decompositions in the sense of the $(2,2)$-SVD since $Z$ can be written in $9$ different ways as the sum of four matrices of $\cAz_{2,2}$ with disjoint supports. Each of these decompositions attains the $(2,2)$-rank which is equal to $4$. Note also that by convexity any convex combination of these decompositions is also an optimal decomposition in the sense of the $(2,2)$-SVD, but can contain up to $36$ terms!

To prove the third claim, let us consider 
$$
Z_1 =  \BSM 1 & 1 & 0 \\ 1 & 1 & 0 \\ 0 & 0 & 0 \ESM ,~  Z_2 =  \BSM 0 & 0 & 0 \\ 0 & 1 & 1 \\ 0 & 1 & 1 \ESM ,  ~ Z = Z_1+Z_2 =  \BSM 1 & 1 & 0 \\ 1 & 2 & 1 \\ 0 & 1 & 1 \ESM\,.
$$
As $Z_1,Z_2,Z$ are all positive semidefinite we have  $ \|Z_1\|_* = 2$, $ \|Z_2\|_* = 2$,  and $ \|Z\|_* = 4$. By inequality (\ref{eq:inequalityOmegaL1trace}), $\Omega_{2,2}(Z) \geq \|Z\|_* = 4$ which proves that the decomposition $Z = Z_1  + Z_2$  is optimal: $\Omega_{2,2}(Z) = 4$. But $\inr{Z_1, Z_2} = 1$. So this decomposition is a decomposition of $Z$ onto linear combination of atoms $\frac 12 Z_1, \frac 12 Z_2 \in \cA_{2,2}$ which are not orthogonal.
\end{proof}

\begin{proof}[Lemma \ref{lem:OmegakqAndDual}]

We first show (\ref{eq:omegadual}) from the definition of the dual norm $\Okq^*$:
\begin{align*}
\Okq^*(Z) &=  \max_K \cbr{ \inr{ K,Z } ~:~ \Okq(K) \leq 1 } \\
&= \max_{a,b} \cbr{ \inr{ Z,ab\trans } ~:~  ab\trans \in \Akq } \\
&= \max_{a,b} \cbr{ a\trans Z b ~:~  \nm{a}_0\leq k~,~\nm{b}_0 \leq q~,~\nm{a}_2 = \nm{b}_2 = 1 } \\
&= \max_{I,J} \cbr{ \nmop{ Z_{I,J} }~:~ I\in\Gk~,~J\in\Gq}\,, 
\end{align*}
where the second equality follows from the fact that the maximization of a linear form over a bounded convex set is attained at one of the extreme points of the set. Given this closed-form expression of the dual norm, we prove the variational formulation~(\ref{eq:defOmegaInfTraceNorm}) for the primal norm $\Okq$. Consider the function $\check\Okq$ defined by
$$
\check\Okq(Z)  = \inf  \cbr{ \sum_{(I,J)\in\Gk\times \Gq } \nmtr{Z^{(I,J)}} ~:~  Z = \sum_{(I,J) } Z^{(I,J)}~,~\text{supp}(Z^{(I,J)}) \subset I\times J  } \,. 
$$
 Since $\check\Okq(Z)$ is defined as the infimum of a jointly convex function of $Z$ and $(Z^{(I,J)})_{I \in \Gk, \: J \in \Gq}$ obtained by minimizing w.r.t. to the latter variables, it is a an elementary fact from convex analysis that $\check\Okq$ is a convex function of $Z$. It is also symmetric and positively homogeneous, which together with convexity prove that 
 $\check\Okq$ defines a norm. We can compute its dual norm as 
\begin{eqnarray*}
\check\Okq^*(K) &=& \max_Z \cbr{ \inr{ K,Z} ~:~ \check\Okq(Z) \leq 1 }\\
&=& \max_{(\aij)_{(I,J)}} \cbr{ \inr{ K,\sum_{(I,J)} \aij } ~:~ \sum_{(I,J)} \nmtr{\aij} \leq 1 ~,~ \supp(\aij) \subset I \times J }\\
&=& \max_{(\aij)_{(I,J)},(\etaij)_{(I,J)}} \cbr{ \sum_{(I,J)} \eta^{(I,J)}\inr{ K_{I,J}, \aij } ~:~ \nmtr{\aij} \leq \etaij, \quad \sum_{(I,J)}\etaij \leq 1} \\
&=& \max_{(\etaij)_{(I,J)}} \cbr{ \sum_{(I,J)} \etaij \nmop{K_{I,J}} ~:~ \sum_{(I,J)}\etaij \leq 1 }\\
&=& \max_{(I,J)}\nmop{K_{I,J}} \\ 
& = & \Okq^*(K) \,.
\end{eqnarray*}
This proves that $\Okq(K)=\check\Okq(K)$ since a norm is uniquely characterized by its dual norm. 

Finally, to show (\ref{eq:subdiff}) we use the general characterization of the subdifferential of a norm \citep[e.g.,][]{Watson1992Characterization}:
$$
G \in \partial \Okq (A) \Leftrightarrow
\begin{cases}
\Okq(A) = \inr{G,A} \,,\\
\Okq^*(G) \leq 1\,.
\end{cases}
$$
Let us denote a subgradient by $G=A+Z$. Since $A=ab^\top$ is an atom, we have $\Okq(A)=1$. In addition, $\nmF{A}^2 = \trace(ba\trans a b\trans) = 1$, therefore the condition $\Okq(A) = \inr{G,A}$ boils down to $\inr{Z,A}=0$. Given the characterization of the dual norm (\ref{eq:omegadual}), we therefore get:
$$
\partial \Okq (A) =  \cbr{  A+ Z ~:~ \inr{A,Z}=0, ~\forall (I,J)\in \Gk \times \Gq\,~ \nmop{A_{I,J} + Z_{I,J}} \leq 1 } \,.
$$
Let now
$$
\Dcal(A) = \cbr{  A+ Z ~:~ A Z_{I_0,J_0}\trans = 0,~  A \trans Z_{I_0,J_0}= 0, ~\forall (I,J)\in \Gk \times \Gq\,~ \nmop{A_{I,J} + Z_{I,J}} \leq 1 }\,.
$$
Since $\inr{A,Z} = \inr{A,Z_{I_0,J_0}} = \trace\br{A\trans Z_{I_0,J_0}}$, it is clear that $\Dcal(A) \subset \partial \Okq (A)$. Conversely, let $G=A+Z \in \partial \Okq (A)$. Then $\inr{A,Z} = \inr{a b\trans, Z_{I_0,J_0}} = a\trans Z_{I_0,J_0} b = 0$, and therefore, by Pythagorean equality applied to the orthogonal vectors $a$ and $Z_{I_0,J_0} b$:
$$
\nm{ \br{A_{I_0,J_0}+Z_{I_0,J_0}}b}_2^2 = \nm{ab\trans b + Z_{I_0,J_0}b }_2^2 = \nm{a + Z_{I_0,J_0}b }_2^2 = 1+\nm{Z_{I_0,J_0}b }_2^2\,,
$$
but since $\nmop{A_{I_0,J_0} + Z_{I_0,J_0}} \leq 1$ and $\nm{b}_2=1$ we must have $\nm{Z_{I_0,J_0}b}_2=0$. This shows that $A Z_{I_0,J_0}\trans = a b\trans Z_{I_0,J_0}\trans = 0$. The same reasoning starting with the orthogonal vectors $b$ and $Z_{I_0,J_0}\trans a$ shows that we also have $A \trans Z_{I_0,J_0}= 0$, implying that $\partial \Okq (A) \subset \Dcal(A)$. This concludes the proof that $\partial \Okq (A) = \Dcal(A)$, as claimed in (\ref{eq:subdiff}).
\end{proof}

\begin{proof}[Lemma \ref{lem:nuclearIsAtomic}]

Let $\nu$ be the nuclear norm induced by two atomic norms $\nm{\cdot}_\alpha$ and $\nm{\cdot}_\beta$, induced themselves respectively by the two atom sets $\cA_1$ and $\cA_2$.  Let $\cA = \big\{ ab\trans~:~a\in\cA_1\,,\, b\in\cA_2\big\}\,$ and $B=\text{Conv}\big (\cA\big)$, then the key argument is to note that we have 
$$
\cbr{ab\trans ~:~   \nm{a}_\alpha \leq 1,~\nm{b}_\beta\leq 1 } \subset B \,.
$$
Indeed, if $a=\sum_i \lambda_i a_i$ and $b=\sum_j \lambda'_j b_j$ with $a_i \in \cA_1, \: b_j \in \cA_2$ and $\sum_i \lambda_i=\sum_j \lambda'_j=1$, then with $\mu_{ij}:=\lambda_i \lambda'_j$, we have $ab\trans=\sum_{i,j} \mu_{ij} a_i b_j\trans$ and $\sum_{i,j} \mu_{ij}=1$.
The inclusion is then proved by density. 
By (\ref{eq:dualNuclearNormGeneral}) the dual norm of $\nu$ satisfies
$$
\nu^*(Z) = \sup \cbr{  a\trans Z b ~:~ \nm{a}_\alpha \leq 1 ~,~\nm{b}_\beta\leq 1 } \,,
$$
so that
$$
\nu^*(Z) \leq \sup \cbr{  \inr{ Z, a b\trans } ~:~ ab\trans \in B} = \sup  \cbr{  \inr{ Z, a b\trans } ~:~ ab\trans \in \cA}\leq \nu^*(Z) \,,
$$
where the middle equality is due to the fact that the maximum of a linear function on a convex set is attained at a vertex. We therefore have $\nu^*(Z) = \sup  \cbr{  \inr{ Z, A }~:~ A \in \cA }$. Given (\ref{eq:dualAtomicNormGeneral}), this shows  that $\nu$ is the atomic norm induced by $\cA$.
\end{proof}

\begin{proof}[Theorem \ref{th:nuclearNorms}]

Since the $\kqtn$ is the atomic norm induced by the atom set (\ref{eq:atoms}), Lemma~\ref{lem:nuclearIsAtomic} tells us that it is also the nuclear norm induced by the two atomic norms with atom sets $\cA_k^{m_1}$ and $\cA_q^{m_2}$, which correspond exactly to the so-called $k$- and $q$-support norms of \citet{argyriou12}.

To prove the second statement, we proceed similarly to get that the $\kqcutn$ is the nuclear norm induced by the two atomic norms with atom sets $\cAz_k^{m_1}$ and $\cAz_k^{m_2}$. Calling $\kappa_k$ and $\kappa_q$ these norms, we obtain an explicit formulation as follows:
 \begin{align*}
\kappa_k(w) &= \max_s \cbr{ \langle s,w \rangle~:~\kappa_k^*(s) \leq 1 }  \\
& = \max \cbr{ \langle s,w \rangle~:~  \frac {1}{\sqrt k} \sum_{i=1}^k |s_{(i)}| \leq 1 }  \\
& =  \left \{   \begin{aligned} \frac{1}{k \sqrt k}\nm{w}_1~~&\text{if}~~\nm{w}_1 \geq k  \nm{w}_\infty \\ \frac 1{\sqrt k}\nm{w}_\infty ~~&\text{if }~~\nm{w}_1 \leq k \nm{w}_\infty \end{aligned}\right. \\   
& = \frac{1}{\sqrt k} \max  \br{ \nm{w}_\infty , \frac 1k \nm{w}_1 } \,.
\end{align*}
\end{proof}

\begin{proof}[Lemma~\ref{lem:dual_vector_norms}]

The form of $\theta^*_k$ follows immediately from the fact that $\theta_k^*(w)=\max \{a\trans w \: : \: a \in \cA_k\}$.
Similarly for $\kappa_k^*$, we have
\begin{align*}
\kappa_k^*(s) = \max \left \{ \langle a,s \rangle~:~a \in \cAz_{k} \right \}= \max_{I:|I|=k} \nm{s_I}_1 = \frac {1}{\sqrt k}\sum_{i=1}^k |s_{(i)}| , 
\end{align*}
where $s_{(i)}$ denotes the the $i$th largest element of $s$ in absolute value. This norm is proportional to a norm known as the vector $k$-norm or 1-$k$ symmetric norm gauge. 
\end{proof}

\begin{proof}[Proposition~\ref{prop:symsoftSVD}]

To prove the first claim, we show a counterexample for the $(2,2)$-SVD in $\RR^{4 \times 4}$. Let $I=\{1,2\}$ and $J=\{3,4\}$. The matrix $Z=\ones \ones\trans \in \RR^4$ can be written as $Z=Z_{I,I}+Z_{I,J}+Z_{J,I}+Z_{J,J}$, and so its $(2,2)$-rank is less than $4$. But its $(2,2)$-rank must be at least $4$, because the matrix has $16$ non-zeros coefficients and the sum of three $(2,2)$-sparse matrices has at most $12$ non-zero coefficients. Its $(2,2)$-rank is thus equal to $4$. 
 
However, it is not possible to write it as a sum of less than $6$ symmetric $(2,2)$-sparse matrices, because each of these matrices can only make one coefficient above the non-diagonal non-zero. 

For the second claim,  we have shown in the proof of proposition~\ref{prop:softkqproperties} that the decomposition above is a $(2,2)$-SVD.

To prove the third claim, note first that the case $k=1$ is peculiar and not representative of the general case because the span of the PSD matrices of sparsity $1$ are only the diagonal matrices, while the span of rank one PSD matrices of sparsity $k\times k$ for $k>1$ is all the symmetric matrices. Now, we claim that it is not possible to write $Z=\ones \ones\trans \in \RR^3$ as a sum of PSD matrices that are $(2,2)$-sparse and PSD. Indeed, if this was the case, this would imply the existence of a non zero vector $v$ with a support of size at most $2$ such that $Z-vv\trans \succ 0$. Since the only eigenvector of $Z$ associated with a non-zero eigenvalue is the constant vector this is impossible.

\end{proof}


\section{Proofs of results in Section~\ref{sec:denoisingStatistics}}\label{sec:slowRates}

\begin{proof}[Lemma~\ref{lem:denoisingbound}]

We prove a more general result than Lemma~\ref{lem:denoisingbound}. Let $\Omega : \R^{m_1 \times m_2} \to \R$ be any matrix norm, and $\cX : \R^{m_1 \times m_2} \to \R^n$ be a linear map. We denote by $X_i$ ($i=1,\ldots,n$) the $i$-th design matrix defined by $\cX(Z)_i = \langle Z,X_i\rangle$. For a given matrix $Z^\star \in \R^{m_1 \times m_2}$, assume we observe:
\begin{equation}\label{eq:ydef}
Y=\cX(Z^\star) +\epsilon \,,
\end{equation}
where $\epsilon \in \R^n$ is a centered random noise vector. We consider the following estimator of $Z^\star$:
\begin{equation}\label{eq:defLeastSquRegGeneralOmega}
\hat Z_\Omega \in \arg \min_{Z}  \frac{1}{2n} \nm{Y-\cX(Z)}_2^2+\lambda \Omega(Z) \,,
\end{equation}
for some value of the parameter $\lambda >0$. The following result generalizes standard results known for the $\ell_1$ and trace norms \citep[e.g.,][Theorem 1]{Koltchinskii11} to any norm $\Omega$.
\begin{theorem}\label{th:slowrates}
If $\lambda \geq   \frac 1n  \Omega^*(\sum_{i=1}^n \epsilon_i X_i)$ then
\begin{equation}\label{eq:boundSlowRate}
\frac{1}{2n} \nm{\cX(\hat Z_\Omega - Z^\star)}_2^2 \leq \inf_Z \left \{ \frac{1}{2n} \nm{\cX( Z - Z^\star)}_2^2  + 2 \lambda \Omega(Z) \right \} \, .\end{equation}
\end{theorem}
Lemma~\ref{lem:denoisingbound} is then a simple consequence of Theorem~\ref{th:slowrates} by taking for $\cX$ the identity map, upper bounding the right-hand side of (\ref{eq:boundSlowRate}) by the value $2\lambda \Omega(Z^\star)$ it takes for $Z=Z^\star$, and replacing $\lambda$ by $\lambda/n$.
\end{proof}

\begin{proof}[Theorem \ref{th:slowrates}] 

By definition of $\hat Z_\Omega$ (\ref{eq:defLeastSquRegGeneralOmega}), we have for all $Z$:
$$
\frac{1}{2n} \nm{Y-\cX(\hat Z_\Omega)}_2^2 \leq \frac{1}{2n} \nm{Y-\cX( Z)}_2^2 + \lambda \left ( \Omega(Z) - \Omega(\hat Z_\Omega) \right ) \,,
$$
which after developing the squared norm and replacing $Y$ by (\ref{eq:ydef}) gives
$$
\frac{1}{2n} \nm{\cX(\hat Z_\Omega)}_2^2 - \frac 1n \langle \cX(Z^\star) +\epsilon,\cX(\hat Z_\Omega) \rangle \leq\frac{1}{2n} \nm{\cX( Z)}_2^2 - \frac 1n \langle \cX(Z^\star) +\epsilon,\cX( Z) \rangle  + \lambda \left ( \Omega(Z) - \Omega(\hat Z_\Omega) \right )  \,,
$$
and therefore
\begin{equation}\label{eq:tech2}
\frac{1}{2n} \nm{\cX(\hat Z_\Omega - Z^\star)}_2^2 \leq  \frac{1}{2n} \nm{\cX( Z - Z^\star)}_2^2  + \frac 1n \langle \epsilon, \cX(\hat Z_\Omega-  Z) \rangle  + \lambda \left ( \Omega(Z) - \Omega(\hat Z_\Omega) \right )\,.
\end{equation}
Now, using the fact (true for any norm) that $\Omega(A)\Omega^\star(B) \geq \langle A , B \rangle$ for any vectors $A,B\in\RR^n$, and taking $\lambda \geq \frac 1n  \Omega^*(\sum_{i=1}^n \epsilon_i X_i)$, we can upper bound the second term of the right-hand side of (\ref{eq:tech2}) by:
\begin{equation*}
\begin{split}
 \frac 1n \langle \epsilon, \cX(\hat Z_\Omega-  Z) \rangle 
 & =  \frac 1n \sum_{i=1}^n \epsilon_i \cX(\hat Z_\Omega-  Z)_i \\ 
 & =  \frac 1n \sum_{i=1}^n \epsilon_i \langle X_i , \hat Z_\Omega-  Z \rangle \\ 
 & =  \frac 1n \langle \sum_{i=1}^n \epsilon_i  X_i , \hat Z_\Omega-  Z \rangle \\ 
 & \leq  \frac 1n \Omega^\star \br{ \sum_{i=1}^n \epsilon_i  X_i } \Omega\br{\hat Z_\Omega-  Z } \\ 
 & \leq \lambda \Omega\br{\hat Z_\Omega-  Z } \\ 
\end{split}
\end{equation*}
Plugging this bound back in (\ref{eq:tech2}) finally gives
\begin{equation*}
\begin{split}
\frac{1}{2n} \nm{\cX(\hat Z_\Omega - Z^\star)}_2^2 
& \leq  \frac{1}{2n} \nm{\cX( Z - Z^\star)}_2^2  + \lambda \Omega(\hat Z_\Omega - Z)  + \lambda \left ( \Omega(Z) - \Omega(\hat Z_\Omega) \right )  \\
& \leq  \frac{1}{2n} \nm{\cX( Z - Z^\star)}_2^2  + 2 \lambda \Omega(Z)  \,,
 \end{split}
 \end{equation*}
 the last inequality being due to the triangle inequality.
\end{proof}

Before proving Propositon~\ref{prop:noisedualbound}, let us first derive an intermediary results useful to obtain an upper bound on the dual $\kqtn$ of a random matrix with i.i.d. normal entries.
\begin{lemma} \label{lem:max_exp_sq_op_norm}
Let $G$ be a $m_1\times m_2$ random matrix with i.i.d. normally distributed entries. Then
$$
\E \max_{I\in\Gcal_k , J\in\Gcal_q} \nm{ G_{I,J} }_\op^2  \leq 16 \sqb{ \left ( k \log \frac {m_1}k +  q \log \frac {m_2}q \right ) + 2(k+q)}\,.
$$
\end{lemma}

\begin{proof}[Lemma \ref{lem:max_exp_sq_op_norm}]

For a random matrix  $H \in \RR^{k \times q}$ with i.i.d. standard normal entries, we have the following concentration inequality \citep[e.g.,][]{Davidson01}: for $s\geq 0$, 
\begin{equation}
\label{eq:DavSzar}
 \mathbb{P}[ \nmop{ H } > \sqrt{k} + \sqrt{q} + s] \leq \exp(-{s^2}/{2}) \,.
\end{equation} 
Denoting $R = 2\left ( \sqrt k + \sqrt q \right )$, and $f(x) = e^{tx^2}$, we have the sequence of inequalities
 \begin{align}
 \mathbb{E} \exp (t \nmop{ H }^2 )  
  & = \mathbb{E} f( \nmop{ H } )  \nonumber\\
 & =  \int_1^\infty \mathbb{P}[f(\nmop{ H })>h] ~~dh \nonumber \\
  &\leq \int_1^{1+f(R)} ~1~ dh + \int_{1+f(R)}^\infty \mathbb{P}[f(\nmop{ H })>h]  dh \nonumber \\
  & = {f(R)} +  \int_0^\infty \mathbb{P}[\nmop{ H }>f^{-1}(f(R) +1+ \zeta)]  d\zeta \nonumber \\
  & \leq f(R) +  \int_0^\infty \mathbb{P}[\nmop{ H } >\frac 12 R +\frac 12 f^{-1}(1+\zeta)]  d\zeta \label{in:propertyfinverse}\\
&  \leq  f(R)  + \int_0^\infty 8ts ~  \exp \left (-s^2/2 +4  {ts^2} \right ) ds \label{eq:boundGopChgtVariable} \\
& \leq  f(R) + 4 \frac{t}{\frac 12 - 4t}  \label{in:integralvalueFisrt}   \\
& \leq \exp(8 t(k+q) ) +   \frac{8 t}{1 - 8t} \nonumber \,,
 \end{align}
where the change of variable used in (\ref{eq:boundGopChgtVariable}) is $1+\zeta = f(2s) = e^{4ts^2}$, (\ref{in:integralvalueFisrt}) is true for any ${t}<\frac{1}{8}$, and  (\ref{in:propertyfinverse}) follows from the property of the inverse $f^{-1}(z) = \sqrt{\frac{\log(z)}{t}}$ that it is strictly increasing on $[1; \infty )$ and  sandwiched via
 \begin{equation}\label{eq:sandwichfinverse} 
 \frac 12 \left \{  f^{-1}(z) +  f^{-1}(z') \right \}  \leq  f^{-1}(z+z') \leq  f^{-1}(z) +  f^{-1}(z')\,.
 \end{equation}

Take now $t = \frac 18 - \frac {1}{8(k+q)}$. Since $k+q\geq 2$, we have $1/16 \leq t < 1/8$. Therefore, 
 \begin{align}
\mathbb{E} \max_{I,J} \nmop{ G_{I,J} }^2 & =\frac 1t \log  \left \{  \exp t\, \mathbb{E} \max_{I,J} \nmop{ G_{I,J} }^2\right \}  \nonumber \\
& \leq \frac 1t \log  \left \{  \mathbb{E}  \exp (t  \max_{I,J} \nmop{ G_{I,J} }^2 ) \right \} \nonumber\\
& \leq \frac 1t \log \bigg \{ \sum_{I,J} \mathbb{E} \exp (t \nmop{ G_{I,J} }^2 ) \bigg \}  \nonumber\\
& \leq \frac 1t \log \left \{ \begin{pmatrix} m_1\\k \end{pmatrix}  \begin{pmatrix} m_2\\ q \end{pmatrix}\mathbb{E} \exp (t \nmop{ H }^2 ) \right \} \nonumber  \\
& \leq \frac 1t \log \left \{ \left ( \frac {e~m_1} k\right )^k  \left ( \frac {e~m_2} q\right )^q \left ( e^{8 t(k+q) } + \frac{8t}{1-8t}\right ) \right \} \nonumber  \\
& = \frac 1t \left [  \left ( k \log \frac {m_1}k +  q \log \frac {m_2}q \right ) + k+q + 8 t(k+q) + \log \left ( 1 + \frac{8t}{1-8t} e^{-8 t(k+q) }\right ) \right ] \nonumber \\
& \leq 16 \left [  \left ( k \log \frac {m_1}k +  q \log \frac {m_2}q \right ) + k+q \right ]  + 8 (k+q) +   \frac{8}{1-8t} e^{-8 t(k+q) } \nonumber \\
& \leq 16 \left [  \left ( k \log \frac {m_1}k +  q \log \frac {m_2}q \right ) + 2(k+q) \right ]  \nonumber \,,
 \end{align}
where in the last inequality we simply used $8/(1-8t) = 8(k+q)$ and $\exp(-8t(k+q))\leq 1$.
  \end{proof}
  
\begin{proof}[Propositon~\ref{prop:noisedualbound}]

From Lemma~\ref{lem:max_exp_sq_op_norm} we have:
\begin{equation*}
\begin{split}
\E\, \Okq^*(G) & = \E\, \max_{I\in\Gcal_k , J\in\Gcal_q} \nmop{ G_{I,J} } \\
& \leq  \br{\E\, \max_{I\in\Gcal_k , J\in\Gcal_q} \nmop{ G_{I,J} }^2}^{\frac 12} \\
& \leq 4 \sqb{ \left ( k \log \frac {m_1}k +  q \log \frac {m_2}q \right ) + 2(k+q)}^{\frac 12}\\
& \leq 4 \br{\sqrt{ k \log \frac {m_1}k + 2k} + \sqrt{q \log \frac {m_2}q + 2q} }
\end{split}
\end{equation*}
The upper bounds for the $\ell_1$ and trace norms are standard. See  \citet[Theorem.~5.32]{Vershynin12} for the tight upper bound on the operator norm $\E \nmop{ G } \leq \sqrt{m_1} + \sqrt{m_2}$, and for the upper bound on the element-wise $\ell_\infty$ norm of $G$, use Jensen inequality followed by upper bounding the maximum of nonnegative scalars by their sum:  
\begin{equation*}
\begin{split}
 \exp \left (t~\E\, \nm{ G }_\infty\right ) & \leq \E  \exp \left (t~\, \nm{G}_\infty\right )  \\
& \leq  m_1m_2 \exp(t^2 / 2)\,.
\end{split}
\end{equation*}
Taking $t = \sqrt{2 \log(m_1 m_2)}$ in the logarithms of the last inequality gives $\E\, \nm{G}_\infty \leq \sqrt{2 m_1m_2 }$.
\end{proof}

\section{Some cone inclusions (Proofs of results in Section~\ref{subset:statdim})}
\label{app:geom}
Let us start with a simple result useful to prove inclusions of tangent cones.
\begin{lemma}\label{lem:tgt_cones}
Let $f$ and $g$ two convex functions from $\RR^d$ such that $f \leq g$ and let $x^*$ such that $f(x^*)=g(x^*)$. Then $T_g(x^*) \subset T_f(x^*)$.
\end{lemma}
\begin{proof}[Lemma \ref{lem:tgt_cones}]\\
Let $h\in\RR^d$ and $\tau>0$ such that $g(x^*+\tau h) \leq g(x^*)$. Then we also have
$$
f(x^*+\tau h) \leq g(x^*+\tau h) \leq g(x^*) = f(x^*)\,,
$$
and therefore, for any $\tau>0$,
$$
\cbr{h\in\RR^d ~:~ g(x^*+\tau h) \leq g(x^*)} \subset \cbr{h\in\RR^d ~:~ f(x^*+\tau h) \leq f(x^*)} \,.
$$
From the definition (\ref{eq:tcone}) of the tangent cone we deduce, by taking the union over $\tau>0$ and the closure of this inclusion, that $T_g(x^*) \subset T_f(x^*)$.
\end{proof}
We can now prove the results in Section~\ref{subset:statdim}

\begin{proof}[Proposition \ref{prop:norm_ineqs}]

Consider a matrix $A=ab\trans \in \Azkq$. We have $\nmtr{A}=\|a\|_2\|b\|_2=1$, and $\nm{A}_1=\nm{a}_1\nm{b}_1=\sqrt{kq}$. Since $A$ is an atom of both the norm $\Okq$ and the norm $\Ozkq$ we have $\Okq(A)=\Ozkq(A)=1$ so that, for any $\mu\in\sqb{0,1}$,
$$
\Gamma_\mu(A) = \nmtr{A}=\frac{1}{\sqrt{kq}} \nm{A}_1=\Okq(A)=\Ozkq(A)=1 \,.
$$
Besides, for any matrix $K \in \RR^{m_1\times m_2}$, for all $(I,J) \in \Gk \times \Gq$, we have $\nmop{K_{I,J}} \leq \nmop{K}$ and $\nmop{K_{I,J}}\leq \nmF{K_{I,J}}\leq \sqrt{kq} \nm{K_{I,J}}_{\infty}$ so that $\Okq^*(K)\leq \nmop{K}$ and $\Okq^*(K)\leq \sqrt{kq}\max_{I,J}\nm{K_{I,J}}_{\infty}=\sqrt{kq} \nm{K}_{\infty}$. Given that $\Azkq \subset \Akq$, we also have that
$$
\Ozkq^*(K)=\max_{A \in \Azkq} \inr{A , K} \leq \max_{A \in \Akq} \inr{A , K}=\Okq^*(K) \,.
$$
By Fenchel duality, we therefore have for any $Z\in \RR^{m_1\times m_2}$ and $\mu\in\sqb{0,1}$:
$$
\frac{\mu}{\sqrt{kq}} \nm{Z}_1 + (1-\mu) \nmtr{Z} \leq \Okq(Z) \leq \Ozkq(Z) \,.
$$
\end{proof}

\begin{proof}[Corollary~\ref{prop:TangentConeInclusion}]

Combining Proposition~\ref{prop:norm_ineqs} with Lemma~\ref{lem:tgt_cones} directly gives (\ref{eq:ranktangcone}). (\ref{eq:rankstatdim}) is then a direct consequence of the definition of the statistical dimension (\ref{eq:sdim}).
\end{proof}

\begin{proof}[Corollary~\ref{cor:improvement_in_exact_recovery}]

A necessary and sufficient condition for exact recovery is the so called null space property which is the event that $T_{\Omega}(Z^*) \cap \text{Ker}(\cX)=\{0\}$, where $\text{Ker}(\cX)$ is the kernel of the linear transformation $\cX$ \citep[Proposition 2.1]{Chandrasekaran12}. The result therefore follows from the inclusion of the cones stated in Corollary~\ref{prop:TangentConeInclusion}.
\end{proof}

\begin{proof}[Proposition~\ref{lem:tanconeequal}]

Let $a\in\cAz_k^m$ with $\supp(a)=I_0$, meaning that $|a_i|=1/\sqrt{k}$ for $i\in I_0$ and $a_i=0$ for $i\in I_0^\complement$. The sub differential of the scaled $\ell_1$ norm $\Gamma_1$ at $a$ is
$$
\partial\Gamma_1(a) = \cbr{s\in\RR^m ~:~ s_i = \sign(a_i) \text{ for }i\in I_0\,, ~|s_i| \leq 1\text{ for }i\in I_0^\complement}\,.
$$
From (\ref{eq:subdiff}), we get that the subdifferential of $\theta_k$ at $a$ is
$$
\partial \theta_k(a) = \cbr{a+z ~:~ \forall i \,, a_i z_i = 0 \text{ and } \forall I\in\GG_k^m\,, \nm{a_I+z_I}\leq 1}\,.
$$
The first condition is equivalent to $z_i=0$ for $i\in I_0$, which implies that the second is equivalent to $|z_i|\leq 1/\sqrt{k}$ for $i\in I_0^\complement$. We deduce that $s=a+z\in \partial \theta_k(a) $ if and only if $s_i=a_i$ for $i\in I_0$ and $|s_i| \leq 1/\sqrt{k}$ for $i\in I_0^\complement$, \ie, 
$$
\partial \theta_k(a) = \frac{1}{\sqrt{k}} \partial\Gamma_1(a) \,.
$$
This shows that the subdifferentials of $\Gamma_1$ and $\theta_k$ have the same conic hull, and Proposition~\ref{lem:tanconeequal} follows by noting that the tangent cone is the polar cone of the conic hull of the subdifferential \citep[Theorem 23.7]{Rockafellar1997Convex}.
\end{proof}

\section{Upper bound on the statistical dimension of $\Okq$ (proof of Proposition \ref{prop:gaussianWidthAtom})}\label{sec:stat_dim_proof}

The aim of this appendix is to prove the upper bound on the statistical dimension $\Okq$ given in Proposition \ref{prop:gaussianWidthAtom}. Given its level of technicality, we split the proof in several parts. We start with preliminaries and notations in Section~\ref{sec:statdimpreliminaries}, before proving Proposition \ref{prop:gaussianWidthAtom} in Section~\ref{sec:statdimmainproof}. The proofs of several technical results needed in Section~\ref{sec:statdimmainproof} are postponed to Section~\ref{sec:proofepsilonIsAGoodScalingFactor}, \ref{sec:proofdecomposition} and \ref{sec:boundingEpsilon}.

\subsection{Preliminaries and notations}\label{sec:statdimpreliminaries}

Let us start with some notations used throughout Appendix~\ref{sec:stat_dim_proof}. $A=ab\trans \in \Akq$ is an atom of $\Okq$, with $I_0=\supp(a)$ and $J_0=\supp(b)$. $\gamma=\gamma(a,b)$ refers to the atom strength of $A$ (Definition~\ref{def:dispersion}). For any $I\in\Gk$ and $J\in\Gq$, let $u_I=a_I/ \nm{a_I}_2$ and $v_J=b_J / \nm{b_J}_2$. Note that while $a_I$ is a subvector of $a$, the notation $u_I$ does not refer to a subvector of some vector $u$ and that therefore $[u_{I}]_{\io}\neq[u_{\io}]_I=a_I$ since $\|a_{\io}\|=\|a\|=1$.

To analyze the statistical dimension (\ref{eq:sdim}) of $\Okq$ at $A$, it is useful to express it as follows \citep[Proposition 3.6]{Chandrasekaran12}:
\begin{equation}\label{eq:statdimnormal}
\sdim (A,\Okq) := \E \sqb{ \dist \br{ G, N_{\Okq}(A) }^2 }\,,
\end{equation}
where $N_{\Okq}(A)$ is the normal cone of $\Okq$ at $A$ (\ie, the conic hull of the subdifferential of $\Okq$ at $A$) and $\dist \br{ G, N_{\Okq}(A)}$ denotes the Frobenius distance of the Gaussian matrix $G$ with i.i.d. standard normal entries to $N_{\Okq}(A)$. In order to upper bound this quantity, it is therefore important to characterize precisely the normal cone $N_{\Okq}(A)$ .

For that purpose, let us introduce further notations. We consider the following subspace of $\RR^{m_1 \times m_2}$
$$
\Span(A) = \cbr{  L A + A R ~:~  L \in \RR^{m_1\times m_1}, ~ R \in \RR^{m_2\times m_2} }\,,
$$
and denote by $\PA$ and $\PA^\perp$ the orthogonal projectors onto $\Span(A)$ and $\Span^\perp(A)$ respectively. Since $A=ab\trans$ with $\nm{a}_2=\nm{b}_2=1$, we have the closed-form expressions $\PP_A^\perp(Z) = (Id_{m_1} - aa\trans)Z(Id_{m_2} - bb\trans)$.

For any $(I,J)\in\Gk \times \Gq$, consider now the subspace
$$
\Span_{I,J}(A) = \cbr{  L_{I,I} A_{I,J} + A_{I,J}R_{J,J} ~:~  L \in \RR^{m_1\times m_1},~ R \in \RR^{m_2\times m_2} } \,,
$$
and its orthogonal
$$
\Span_{I,J}^\perp(A) = \cbr{  Z \in \RR^{m_1\times m_2} ~:~  A_{I,J} Z_{I,J}\trans = A_{I,J} \trans Z_{I,J}= 0 } \,.
$$
Note that $\Span_{I_0,J_0}^\perp(A)$ is related to the subdifferential of $\Okq$ at $A$, since according to (\ref{eq:subdiff}) we can write it as
\begin{equation}\label{eq:subdiffspan}
\partial \Okq(A) = \cbr{A+Z ~:~ Z\in \Span_{I_0,J_0}^\perp(A) \,,~ \forall (I,J)\in\Gk\times \Gq ~\nmop{A_{I,J}+Z_{I,J}} \leq 1}\,.
\end{equation}
It is possible to estimate the dimension of $\Span_{I_0,J_0}^\perp(A)$ as follows:
\begin{lemma}
\label{lem:codimension}
The dimension of $\Span_{I_0,J_0}(A)$ is $k+q-1$.
\end{lemma} 
\begin{proof}[Lemma~\ref{lem:codimension}]

For $A =  ab\trans$,   the range of $L \mapsto L_{I_0,I_0} A_{I_0,J_0}$ equals the range of $ \alpha_{I_0} \mapsto \alpha_{I_0} b\trans$ which has dimension $|{I_0}| = k$. By the same token, the range of  $R \mapsto  A_{I_0,J_0} R_{J_0,J_0}$ has dimension $q$. By definition of $\Span_{I_0,J_0}(A)$ we therefore have
\[ 
\Span_{I_0,J_0}(A) = \cbr{  \alpha_{I_0} b\trans + a\beta_{J_0} \trans ~:~  \alpha \in \RR^{m_1}, \beta \in \RR^{ m_2}  }
\]
and therefore by the inclusion-exclusion principle $\dim \br{ \Span_{I_0,J_0}(A) }= k+q-1.$
\end{proof}

Finally we denote by $\pspXIJ$ the projector onto $\Span_{I,J}(A)$, and by $\pspXIJorth$ the projector onto $\Span_{I,J}^\perp(A)$. They satisfy respectively
$$
\pspXIJ(Z) = \PP_{A_{I,J}}(Z_{I,J})\quad \text{and}\quad \pspXIJorth(Z) = Z -  \pspXIJ(Z) = Z -  \PP_{A_{I,J}}(Z_{I,J}) \,.
$$

\subsection{Proof of Proposition \ref{prop:gaussianWidthAtom}}\label{sec:statdimmainproof}

\begin{proof}[Proposition \ref{prop:gaussianWidthAtom}]

In order to upper bound the statistical dimension of $\Okq$ at $A$, we associate to any matrix $G$ a matrix $\Xi(G)$ belonging to the normal cone $N_{\Okq}(A)$, where $\Xi:\RR^{m_1\times m_2}\rightarrow \RR^{m_1\times m_2}$ is measurable. From the characterization of the statistical dimension (\ref{eq:statdimnormal}), since $\dist \br{ G, N_{\Okq}(A) } \leq \nmF{G - \Xi(G)}$, we will then get the upper bound:
\begin{equation}\label{eq:boundsdim}
\sdim (A,\Okq) = \E \sqb{ \dist \br{ G, N_{\Okq}(A) }^2 } \leq \E \nmF{G - \Xi(G)}^2 \,.
\end{equation}
The main steps in the proof are then (i) to define the mapping $\Xi$, (ii) to show that $\Xi(G) \in N_{\Okq}(A)$ for all $G$, and (iiii) to upper bound $\E \nmF{G - \Xi(G)}^2$ in order to derive an upper bound on $\sdim (A,\Okq)$ by (\ref{eq:boundsdim}).

Given a measurable function $\epsilon:\RR^{m_1\times m_2} \rightarrow \RR$, let us therefore consider the mapping $\Xi$:
\begin{equation}\label{eq:xidef}
\forall G\in\RR^{m_1\times m_2},\quad \Xi(G) \eqdef \eG A + \pspz^\perp(G)\,.
\end{equation}
The following lemma provides a mapping $\epsilon$ to ensure that $\Xi(G) \in N_{\Okq}(A)$.
\begin{lemma}\label{lem:epsilonIsAGoodScalingFactor}
Let $\eG^2$ be equal to
\begin{equation}\label{eq:epsilon}
\frac{16}{\gamma^{2}}\nmop{G_{\io,\jo}}^2 \,\vee \,\max_{I \in \Gk \atop J \in \Gq} \nmop{G_{IJ}}^2 \, \vee\: \max_{{0 \leq i < k \atop 0 \leq  j < q} \atop (i,j) \neq (0,0)}  \frac {8}{\gamma \left (\frac ik+\frac jq \right)}  \max_{ |\imio| = i  \atop |\jmjo| = j} \sqb{ \nm{ G_{\iio,\jmjo}\trans u_{I} }_2^2+ \nm{G_{\imio,\jjo} v_{J}}_2^2} \,.
\end{equation}
Then, for every $G \in \RR^{m_1\times m_2}$, the matrix $\Xi(G)$ defined in (\ref{eq:xidef}) belongs to the normal cone of $\Okq$ at $A$.
 \end{lemma}
By choosing $\eG$ as in Lemma~\ref{lem:epsilonIsAGoodScalingFactor}, the upper bound (\ref{eq:boundsdim}) because $\Xi(G)\in N_{\Okq}(A)$. Using the decomposition $G = \pspz(G) + \pspz^\perp(G)$ we deduce
\begin{align}
\sdim (A,\Okq) \leq \E  \nmF{ G - \Xi(G) }^2
& =  \E \nmF{ \eG A - \pspz(G) }^2 \nonumber \\
& \leq 2 \E \nmF{   \eG A  }^2 + 2 \E \nmF{  \pspz(G) }^2 \nonumber \\
& =  2 \E ~ \eG ^2 + 2( k+q-1) \label{ineq:abeq1},
\end{align}
where (\ref{ineq:abeq1}) is due to $\nmF{A} = 1$ and the fact that $\nmF{  \pspz(G) }^2$ follows a chi-square distribution with $k+q-1$ degrees of freedom, since by Lemma~\ref{lem:codimension} this is the dimension of $\Span_{I_0,J_0}(A)$. In order to upper bound $\E ~ \eG ^2$ we need the following two lemmata in addition to Lemma~\ref{lem:max_exp_sq_op_norm}.
\begin{lemma}\label{lem:exp_sq_op_norm} 
\begin{equation}\label{eq:termI0J0epsilon}
\E\nmop{G_{\io,\jo}}^2 \leq 4(k+q)+4 \,.
\end{equation}
\end{lemma}
 \begin{lemma}\label{lem:maxMIJ}
 \begin{multline*}
 \E  \max_{i,j}  \frac {8}{\gamma \left (\frac ik+\frac jq \right)}   \max_{|\jmjo| = j \atop |\imio| = i}  \sqb{ \|G_{\iio,\jmjo}\trans u_{I}\|_2^2+ \|G_{\imio,\jjo} v_{J}\|_2^2 } \\
\leq \frac {48} \gamma (k\vee q) \log \left((m_1-k)\vee (m_2-q) \right )  + \frac {64} \gamma (k\vee q) \,.
\end{multline*}
\end{lemma}
Combining Lemmata~\ref{lem:max_exp_sq_op_norm}, \ref{eq:termI0J0epsilon} and \ref{lem:maxMIJ} with the definition of $\eG$ in (\ref{eq:epsilon}) we deduce
\begin{equation*}
\begin{split}
\E ~ \eG ^2 & \leq \frac{16}{\gamma^2} \sqb{4(k+q)+4} +  16 \sqb{ \left ( k \log \frac {m_1}k +  q \log \frac {m_2}q \right ) + 2(k+q)} \\
& \quad \quad \quad + \frac {48} \gamma (k\vee q) \log \left((m_1-k)\vee (m_2-q) \right )  + \frac {64} \gamma (k\vee q) \\
& \leq \left ( \frac{64}{\gamma^2} +\frac{64}{\gamma}+32  \right )  (k+q + 1)+ 16 \left ( k \log \frac {m_1}k +  q \log \frac {m_2}q\right ) \\
& \quad \quad \quad+  \frac{48}{\gamma} (k\vee q) \log \left(m_1\vee m_2 \right ) \\
& \leq \frac{160}{\gamma^2} (k+q + 1) + \frac{80}{\gamma}  (k\vee q) \log \left(m_1\vee m_2 \right ) \,.
\end{split}
\end{equation*}
Plugging this upper bound into (\ref{ineq:abeq1}) finally proves Proposition~\ref{prop:gaussianWidthAtom}.
\end{proof}

\subsection{The scaling factor $\eG$ ensures that $\Xi(G) \in N_{\Okq}(A)$ (proof of Lemma~\ref{lem:epsilonIsAGoodScalingFactor})}\label{sec:proofepsilonIsAGoodScalingFactor}

\begin{proof}[Lemma~\ref{lem:epsilonIsAGoodScalingFactor}]

To simplify notations let us denote
$$
\tG \eqdef \pspzorth(G)\,,
$$
so that (\ref{eq:xidef}) becomes $\Xi(G) = \eG A + \tG$. To prove that $\Xi(G)$ belongs to the normal cone of $\Okq$ at $A$, it is sufficient to prove that $\eG^{-1} \Xi(G) = A + \eG^{-1} \tG$ is a subgradient of $\Okq$ at $A$. By the characterization of the subgradient in (\ref{eq:subdiffspan}), and since $\tG \in \Span_{I_0,J_0}^\perp(A)$, this is equivalent to $\nmop{A_{IJ}+\eG^{-1}\,\tG_{IJ} } \leq 1$ for any $(I,J)\in \Gk\times \Gq$, which itself is equivalent to
 \begin{equation}
 \nmop{A_{IJ}+\eG^{-1}\, \pspXIJ(\tG) } \leq 1 \qquad \text{and} \qquad \eG^{-1}\, \nmop{ \PP_{A}^\perp(\tG_{I,J}) } \leq 1 \,.
\label{eq:ineqs}
\end{equation} 
First, the second inequality of (\ref{eq:ineqs}) is satisfied since
$$
\nmop{ \PP_{A}^\perp(\tG_{I,J}) } \leq \nmop{ \tG_{I,J} }=\nmop{ \sqb{ \pspz^\perp(G) }_{IJ} } \leq \nmop{[G]_{IJ} } \leq \eG \,.
$$
There thus remains to prove the first inequality of (\ref{eq:ineqs}). Note that the matrix $A_{IJ}+\eG^{-1}\, \pspXIJ(\tG)$ has rank 2, so its Frobenius norm is larger than its operator norm by at most a factor of $\sqrt 2$. Working with the Frobenius norm is more convenient, so knowing that 
$$
\nmop{A_{IJ}+\eG^{-1}\, \pspXIJ(\tG)}^2 \leq  \nmF{ A_{IJ}+\eG^{-1}\, \pspXIJ(\tG)}^2 \,,
$$
we will establish an upper bound on the latter quantity which we denote by $\nu_{I,J}(G)$. Noting that $A_{IJ} = \nm{a_I}_2 \nm{b_J}_2 u_I v_J\trans$ and that
$$
\pspXIJ(\tG)=u_Iu_I\trans\tG_{IJ}+\tG_{IJ} v_Jv_J\trans - u_I u_I\trans \tG_{IJ} v_J\,  v_J\trans \,,
$$ 
we get
 \begin{equation*}
 \begin{split}
\nu_{I,J}(G)
&= \nmF{ \nm{a_I}_2 \nm{b_J}_2 u_I v_J\trans + \eG^{-1} \br{u_I u_I\trans \tG_{IJ} + \tG_{IJ} v_J v_J^ \top - u_I u_I\trans \tG_{IJ} v_J  v_J\trans} }^2  \\
&=\nm{a_I}_2^2 \nm{b_J}_2^2 + \frac{2}{\eG} \nm{a_I}_2 \nm{b_J}_2 u_I\trans \tG_{IJ} v_J \\
& \quad \quad \quad \quad + \frac{1}{\eG^2} \br{u_I\trans \tG_{IJ}\tG_{IJ}\trans u_I+ v_J\trans \tG_{IJ}\trans \tG_{IJ} v_J - 2(u_I\trans \tG_{IJ} v_J)^2} \\
&\leq \nm{a_I}_2^2 \nm{b_J}_2^2 + \frac{2}{\eG} \nm{a_I}_2 \nm{b_J}_2 u_I\trans \tG_{IJ} v_J + \frac{1}{\eG^2} \br{u_I\trans \tG_{IJ}\tG_{IJ}\trans u_I+ v_J\trans \tG_{IJ}\trans \tG_{IJ} v_J }\,.
 \end{split}
 \end{equation*}
The following Lemma provides upper bounds on the different terms.
\begin{lemma}\label{lem:decomposition}
We have
\begin{eqnarray*}
u_I\trans \tG_{IJ} v_J & \leq & \nm{a_\iomi}_2 \,\nm{b_\jomj}_2 \, \nmop{G_{\io \jo}} \,,\\
u\trans_I \tG_{IJ}\tG_{IJ}\trans u_I & \leq  & \nm{ G_{\iio,\jmjo}\trans u_{I} }_2^2+2 \, \nm{a_{\iomi} }_2^2 \, \nmop{G_{\io,\jo} }^2 \,,\\
v\trans_J \tG_{IJ}\trans \tG_{IJ} v_J & \leq  & \nm{G_{\imio,\jjo} v_{J} }_2^2+2 \, \nm{b_{\jomj} }_2^2 \, \nmop{ G_{\io,\jo} }^2 \,.
\label{eq:decomp}
\end{eqnarray*}
\end{lemma}
This yields
 \begin{equation*}
\begin{split}
 \nu_{I,J}(G) 
 & \leq  \:  \nm{a_I}_2^2 \nm{b_J}_2^2 + \frac{2}{\eG} \,\nm{a_I}_2 \nm{b_J}_2  \,  \nm{a_\iomi}_2 \,\nm{b_\jomj}_2 \, \nmop{G_{\io \jo}}  \\
 & \quad\quad\quad\quad\quad\quad+ \frac{1}{\eG^2} \br{ \nm{ G_{\iio,\jmjo}\trans u_{I} }_2^2+2 \, \nm{a_{\iomi} }_2^2 \, \nmop{G_{\io,\jo} }^2 } \\ 
 & \quad\quad\quad\quad\quad\quad+  \frac{1}{\eG^2} \br{ \nm{G_{\imio,\jjo} v_{J} }_2^2+2 \, \nm{b_{\jomj} }_2^2 \, \nmop{ G_{\io,\jo} }^2 } \\
 & \leq  \nm{a_I}_2^2 \nm{b_J}_2^2 +  \frac{\gamma}{2} \nm{a_I}_2 \nm{b_J}_2  \,  \nm{a_\iomi}_2 \,\nm{b_\jomj}_2 \\
  & \quad\quad\quad\quad\quad\quad +\frac{ \gamma }{8} \br{ \frac ik+\frac jq } + \frac{ \gamma^2} 8 \br{ \nm{a_\iomi}_2^2+\nm{b_\jomj}_2^2 } \,,
\end{split}
\end{equation*}
where we used the definition of $\eG$ (\ref{eq:epsilon}) to derive the last inequality.

Define $\alpha:=\|a_\iomi\|^2=1-\|a_I\|^2$ and $\beta:=\|b_\jomj\|^2=1-\|b_J\|^2$. With these notations and rearranging the terms, we can rewrite the above inequality as
$$
\nu_{I,J}(G)  \leq (1-\alpha)(1-\beta)+\frac{\gamma}{2} \sqrt{\alpha \beta (1-\alpha) (1-\beta)} +\frac{\gamma^2}{8}(\alpha+\beta) + \frac{\gamma}{8} \left ( \frac ik + \frac jq \right ) \,.
$$
Since $0 \leq \alpha, \beta \leq 1$ and using $\sqrt{\alpha \beta} \leq \frac{1}{2} (\alpha+\beta)$,
we have 
$$
\alpha \beta \leq \frac{1}{2} (\alpha+\beta) \quad \text{and} \quad  \sqrt{\alpha \beta (1-\alpha) (1-\beta)} \leq \frac{1}{2} (\alpha+\beta) \,.
$$ 
These inequalities yield
\begin{equation*}
\nu_{I,J}(G) \leq 1+(\alpha+\beta) \Big (-1 +\frac{1}{2} +\frac{\gamma}{4}+\frac{\gamma^2}{8} \Big )+\frac{\gamma}{8} \left ( \frac ik + \frac jq \right ) \,.
\end{equation*}

By definition of $\gamma = \min_{\iota \in I_0 \atop \iota' \in J_0} \left (k ~a_\iota^2 ,q ~b_{\iota '}^2\right )$, we have $\frac{i}{k} \leq \frac{\alpha}{\gamma}$ and $\frac{j}{q} \leq \frac{\beta}{\gamma}$. Moreover, given that $0\leq \gamma \leq 1$, we have $\frac{4}{\gamma}-2-\gamma=\frac{1}{\gamma}(4-2\gamma-\gamma^2) \geq \frac{1}{\gamma}$, so that factorizing $\frac{\gamma}{8}$ in the previous expression, we obtain
\begin{eqnarray*}
\nu_{I,J}(G) &\leq& 1+ \frac{\gamma}{8} \left [ \Big ( -\frac{4}{\gamma}+2+\gamma \Big) (\alpha+\beta)+ \left ( \frac ik + \frac jq \right ) \right ]\\
&\leq& 1+ \frac{\gamma}{8} \left [ - \frac{1}{\gamma} (\alpha+\beta)+ \left ( \frac ik + \frac jq \right ) \right ]\\
&\leq& 1 \,,
\end{eqnarray*}
which concludes the proof.
\end{proof}

\subsection{Proof of Lemma~\ref{lem:decomposition}}\label{sec:proofdecomposition}

Let us first start with a few useful technical lemmas.

 \begin{lemma}\label{lemma:decomposeTildeGIJ} The matrix
 $\tG_{IJ}=[\pspzorth(G)]_{IJ}$ is of the form $\: \tG_{IJ}=\tG_1+\tG_2$ with $$\tG_1=G_{IJ}-G_{\iio,\jjo} \quad  \text{and}  \quad \tG_2=(\id_I - a_{I}a\trans) \,  G_{\io \jo} \, (\id_J - b b_J\trans).$$
 \end{lemma}
 \begin{proof}[Lemma \ref{lemma:decomposeTildeGIJ}]\\
 \BEAS
 \pspzorth(G) 
 & = & G-\pspz(G)\\
 & = & G-a_{\io}a_{\io}\trans G_{\io \jo}- G_{\io \jo} b_{\jo}b_{\jo}\trans +a_{\io}a_{\io}\trans G_{\io \jo} b_{\jo}b_{\jo}\trans\\
 & = & G-G_{\io \jo} + (\idio - a_{\io}a_{\io}\trans) \,  G_{\io \jo} \, (\idjo - b_{\jo} b_{\jo}\trans),\\
 \text{so that} \qquad [\pspzorth(G)]_{IJ} 
 & = &  G_{IJ}-G_{\iio,\jjo} \: + \:  (\id_I - a_{I}a\trans) \,  G_{\io \jo} \, (\id_J - b b_J\trans).     
 \EEAS
\end{proof}

\begin{lemma}\label{lemma:uTransTildeG} We have
$\quad \displaystyle u_I\trans \tG_1=u_I\trans G_{\iio,\jmjo} \quad \text{and} \quad \tG_1v_J= G_{\imio,\jjo}v_J~.$
\label{lem:G1}
\end{lemma}
\begin{proof}[Lemma \ref{lemma:uTransTildeG}]\\
Given that $\supp(u_I) \subset \io$, we have
$$
u_I\trans \tG_1= u_I\trans (G_{IJ}-G_{\iio,\jjo})=u_I\trans (G_{\iio,J}-G_{\iio,\jjo})=u_I\trans G_{\iio,\jmjo}\,, 
$$
which proves the first equality. The second one is proved similarly. 
\end{proof}

\begin{lemma}
$\qquad \displaystyle \|\id-b_J b\trans\|_{\op}^2 \leq \frac{4}{3}$
\label{lem:weird_op_norm}
\end{lemma}
\begin{proof}[Lemma \ref{lem:weird_op_norm}]\\
The largest singular value is attained on the span of $b_J$ and $b_{J^c}$ both on the left and on the right.
Given that $\|b\|=1$, it is therefore also the largest eigenvalue of the matrix of the linear operator restricted to this span which is equal to
$$\begin{bmatrix}
(1-x) & -\sqrt{(1-x)x}\\
0 & 1
\end{bmatrix},$$
for $x=\|b_J\|^2$. 
Tedious but simple calculations show that the squared operator norm of this matrix is equal to
${1-x/2+1/2\sqrt{x(4-3x)}},
$
which takes its maximum value $4/3$ for $x=1/3$.
\end{proof}

\begin{proof}[Lemma \ref{lem:decomposition}]\\
Given that $\tG_{IJ}=\tG_1+\tG_2$ and $u_I \trans \tG_1=u_I \tG_{\iio,\jmjo}$, we have 
$u_I\trans \tG_{1} v_J=u_I\trans \tG_{1} v_{\jjo}=0$, so that
\begin{eqnarray*}
u_I\trans \tG_{IJ} v_J & = & u_I\trans \tG_{2} v_J\\
& =& u_I\trans (\id_I - a_{I}a\trans) \,  G_{\io \jo} \, (\id_J - b b_J\trans) v_J \\
& \leq &  \big \|u_I - \|a_{I}\| \, a \, \big  \|  \big  \|G_{\io \jo} \big  \|_{\op} \, \big  \|v_J - \|b_J\| \, b \, \big  \| \\
& \leq & \|a_{\iomi}\|\,\|b_{\jomj}\| \, \|G_{\io \jo}\|_{\op},
\end{eqnarray*}
because
$\|u_I\trans (\id_I - a_{I}a\trans)\|^2=\big \|u_I- \|a_{I}\|\, a \big \|^2=1-2\|a_{I}\|^2+\|a_I\|^2=\|a_{\iomi}\|^2,$
and symmetrically $\big  \|v_J - \|b_J\| \, b \, \big  \|=\|b_{\jomj}\|$. This shows the first inequality.

For the two next inequalities, note that $$u\trans_I \tG_{IJ}\tG_{IJ}\trans u_I=\|\tG_{IJ}\trans u_I\|^2=\|\tG_{1}\trans u_I\|^2+\|\tG_{2}\trans u_I\|^2$$ because $\langle \tG_{1}\trans u_I,\tG_{2}\trans u_I\rangle=0$ as a result of the fact that by lemma~\ref{lem:G1}, $\tG_{1}\trans u_I$ and $\tG_{2}\trans u_I$
have disjoint supports.

Now $\|\tG_{1}\trans u_I\|^2=\|G_{\iio,\jmjo}\trans u_{I}\|_2^2$ and $\|\tG_{2}\trans u_I\| \leq 2 \, \|a_{\iomi}\|^2 \, \|G_{\io,\jo}\|_{\op}^2$, because $\|\id-b_J b\trans\|_{\op}^2 \leq 2$ (see Lemma \ref{lem:weird_op_norm} for a proof). This shows the second inequality and the third follows by symmetry.
\end{proof}

\subsection{Upper bounds for $\eG^2$ (Proofs of Lemmata~\ref{lem:exp_sq_op_norm} and \ref{lem:maxMIJ})}\label{sec:boundingEpsilon} 

\begin{proof}[Lemma \ref{lem:exp_sq_op_norm}]\\

Using (\ref{eq:DavSzar}) and the fact that $\left ( \sqrt{k}+\sqrt{q}+s\right )^2 \leq 2 \left ( (\sqrt{k}+\sqrt{q})^2+s^2 \right )$ gives
$$
\P \sqb{ \nmop{G_{\io,\jo}}^2> 2 \left ( (\sqrt{k}+\sqrt{q})^2+s^2 \right ) } \leq \exp(-s^2/2) \,.
$$
Setting $t = 2s^2$ yields 
$$ 
 \P \sqb{\nmop{G_{\io,\jo}}^2> 4(k+q)+t  } \leq \exp(-t/4) \,.
 $$
It follows that 
\begin{eqnarray*} 
\E~\nmop{G_{\io,\jo}}^2 &=&\int_0^\infty \P(\nmop{G_{\io,\jo} }^2 \geq t') dt'\\
 & = &  \int_0^{4(k+q)} dt'+\int_{4(k+q)}^\infty \P(\nmop{G_{\io,\jo}}^2 \geq t') dt'\\
 &\leq&  4(k+q) + \int_0^\infty \exp(-t/4) dt \\
 & = &4(k+q) +4 \,.
\end{eqnarray*}\vspace{-8mm}

\end{proof}\vspace{-3mm}

\begin{proof}[Lemma \ref{lem:maxMIJ}]\\
As the sets $\iio \times \jmjo$ and  $\imio \times\jjo$ are disjoint, and $u_I, v_J$ of unit length, the random variable 
$$
M_{I,J} = \nm{G_{\iio,\jmjo}\trans u_{I}}_2^2+ \nm{G_{\imio,\jjo} v_{J}}_2^2
$$
follows a chi-square distribution with $i+j$ degrees of freedom, where $i = |\imio|$ and $j = |\jmjo|$. Using Chernoff's inequality and the form of the chi-square moment generating function, we have that for any fixed real number $\alpha$ and fixed index sets $I$ and $J$, for all $t \in (0,1/2)$,
\begin{equation*}
\mathbb P \Big [ M_{I,J} > \alpha  \Big ]  = \mathbb P \Big [ e^{t M_{I,J}} > e^{t \alpha} \Big] \leq e^{-t\alpha} ~\E~ e^{t M_{I,J}}= e^{-t\alpha}(1-2t)^{-\frac{i+j}{2}} \,. 
\end{equation*}
Taking the maximum over index sets $I$ and $J$ with the same intersection sizes with $I_0$ and $J_0$ respectively, and using a union bound on the independent choices of $I$ and $J$, yields
\begin{align*}
\mathbb P \left  [ \max_{ |\imio| = i  \atop  |\jmjo| = j}M_{I,J} > \alpha  \right  ] &\leq \begin{pmatrix} m_1-k\\ i \end{pmatrix} \begin{pmatrix} m_2-q\\ j \end{pmatrix} \exp \left \{-t\alpha -\frac{i+j}{2} \log (1-2t)  \right \}  \\
& \leq  \exp \left \{-t\alpha -\frac{i+j}{2} \log (1-2t) + i\,\log (m_1-k) + j\,\log (m_2-q) \right \} .
\end{align*}
Taking $\alpha = \lambda (i+j)$, we have for any $t < 1/2$ (assuming w.l.o.g. $m_1-k\geq m_2-q$) 
\begin{align*}
\mathbb P \left  [ \max_{ |\imio| = i  \atop  |\jmjo| = j}M_{I,J} >  \lambda (i+j) \right  ] & \leq  \exp \left \{-t \lambda (i+j) -\frac{i+j}{2} \log (1-2t) + i\,\log (m_1-k) + j\,\log (m_2-q) \right \} \\
&  \leq  \exp \left \{(i+j)\left ( -t \lambda  -\frac{1}{2} \log (1-2t) + \log (m_1-k)  \right )\right \} \,.
\end{align*}

Let us introduce $\mathcal M_{i,j} = \frac {1}{i+j}  \max_{  |\imio| = i  \atop  |\jmjo| = j}M_{I,J}$, and take  $t = \frac 12 \left ( 1- \frac {1}{m_1-k }\right ) < \frac 12$. Then
\begin{align*}
\mathbb P \left [ \max_{{0 \leq i < k \atop 0 \leq  j < q} \atop (i,j) \neq (0,0)}\mathcal M_{i,j} > \lambda  \right ] & \leq \sum_{{0 \leq i < k \atop 0 \leq  j < q} \atop (i,j) \neq (0,0)} \exp \left \{(i+j)\left ( - \frac 12 \left ( 1- \frac {1}{m_1-k}\right ) \lambda + \frac 32 \log (m_1-k) \right )\right \} \\
& = \sum_{i=0}^{k-1} \beta^i \sum_{j=0}^{q-1} \beta^j -1 
 = \frac{1 - \beta^{k}}{1-\beta} \frac{1 - \beta^{q}}{1-\beta} -1
 \leq  2 \beta \,,
\end{align*}
where 
$$
\beta  = \exp \left \{ - \frac 12  \left ( 1- \frac {1}{m_1-k}\right ) \lambda +  \frac 32  \log (m_1-k) \right \} \,.
$$ 
As a consequence, we have 
\begin{align*} 
& \E [  \max_{i,j} \mathcal M_{i,j}  ]  = \int_0^\infty \mathbb P [  \max_{i,j} \mathcal M_{i,j} >\lambda ] d \lambda \\
 &  \qquad \leq \int_{0}^{\frac{3(m_1-k)}{m_1-k-1}   \log k} d \lambda
 +2 \int_{\frac{3 (m_1-k) }{m_1-k-1}   \log (m_1-k)}^\infty \exp \left \{   \frac 32 \log (m_1-k) - \frac 12 \left (1- \frac 1{m_1-k}  \right )\lambda \right \}d \lambda\\
 & \qquad  \leq \frac{3(m_1-k)}{m_1-k-1}   \log k +  4 \frac{m_1-k}{m_1-k-1}\\
 &  \qquad  \leq 6 \log (m_1-k) + 8~.
\end{align*}

It follows that  \begin{align}
\E \max_{{0 \leq i < k \atop 0 \leq  j < q} \atop (i,j) \neq (0,0)}  \frac {8}{\gamma \left (\frac ik+\frac jq \right)} &  \max_{|\jmjo| = j \atop |\imio| = i}  \|G_{\iio,\jmjo}\trans u_{I}\|_2^2+ \|G_{\imio,\jjo} v_{J}\|_2^2 \nonumber \\
&\leq \frac {48} \gamma (k\vee q) \log \left((m_1-k)\vee (m_2-q) \right )  + \frac {64} \gamma (k\vee q)~.~ \label{eq:termETAepsilon}\end{align}
\end{proof}

\section{Lower bound on the statistical dimension of $\Gamma_\mu$ (Proof of Proposition~\ref{prop:from_oymak})}\label{sec:proof_oymak_prop}

Let us start with a technical lemma:
\begin{lemma}\label{lem:oymak}
Let $ab\trans \in \Akq$, $\mathcal{X}: \RR^{m_1 \times m_2} \rightarrow \RR^n$ a linear map from the standard Gaussian ensemble and $y=\cX(ab\trans)$. If $n \leq \frac{1}{9} m_1 m_2$ and further
$$n \leq n_0:=\zeta(a,b)\,\frac{1}{6^4} \big ( \, (kq) \wedge (m_1+m_2-1) \big ) -2,\quad \text{with}\quad \zeta(a,b)=1-\Big (1-\frac{\|a\|_1^2}{k}\Big) \Big (1-\frac{\|b\|_1^2}{q}\Big ),$$
then, with probability $1-c_1\exp(-c_2 n_0)$, solving formulation~(\ref{eq:exact_rec}) with the norm $\Gamma_{\mu}$ fails to recover $ab\trans$ simultaneously for any values of $\mu \in [0,1]$, where $c_1$ and $c_2$ are universal constants.
\end{lemma}
\begin{proof}[Lemma \ref{lem:oymak}]\\
The proof consists in applying theorem 3.2 in \citet{Oymak12} for the combination of the $\ell_1$-norm with the trace norm. We adapt slightly the notations of that paper to reflect the fact that we are working with matrices. Since we consider conic combinations of the $\ell_1$ and trace norms, the number of norms is therefore $\tau=2$.
To apply the theorem we need to specify $\kappa,\theta,d_{\min},\gamma$ and $\mathcal{C}^{\circ}$ in the notations of that paper.
 
For each decomposable norm $\nu_j$ for $j \in \{1,2\}$, with $\nu_1$ the $\ell_1$-norm and  and $\nu_2$ the trace norm, given a point $ab\trans$ \citep[which corresponds to the point $\mathbf{x}_0$ in][]{Oymak12}, the authors define 
 \BIT
 \item $T_j$ the supporting subspaces and $E_j$ ($\mathbf{e}_j$ in the paper), the orthogonal projection of any subgradient of the norm in $ab\trans$ (Definition 2.1),
\item $L_j$ the Lipschitz constant of $\nu_j$ with respect to the Euclidean norm (Definition 2.2),
\item $\displaystyle \kappa_j=\frac{\nmF{ E_j }^2}{L_j^2} \frac{m_1m_2}{\text{dim}(T_j)}$ (Definition 2.2).
\EIT

Let $ab\trans \in \Akq$ with support $I_0 \times J_0$ and $s_a=\sign(a)$, $s_b=\sign(b)$.
Denoting $e_{ij}$ the element of the canonical basis of $\RR^{m_1\times m_2}$, we have
\BIT
\item $T_1=\Span(\{e_{ij}\}_{(i,j) \in I_0 \times J_0})$ so that $\text{dim}(T_1)=kq$,
\item $T_2=\{av\trans+u b\trans \mid u \in \RR^{m_1}, v \in \RR^{m_2}\}$ so that $\text{dim}(T_2)=m_1+m_2-1$.
\EIT
By definition $d_{\min}=\text{dim}(T_1) \wedge \text{dim}(T_2)$. We have 
$$
E_1=s_a s_b\trans, \:\: \nmF{ E_1}^2=kq, \:\: E_2=ab\trans, \:\: \nmF{E_2}^2=1, \:\: L_1=\sqrt{kq}, \:\: L_2=\sqrt{m_1 \wedge m_2} \,,
$$
$$
\text{and thus} \quad \kappa_1=\frac{m_1 m_2}{kq}, \:\: \kappa_2=\frac{m_1m_2}{(m_1 \wedge m_2)(m_1+m_2-1)}, \:\: \text{so that} \:\: \kappa=\kappa_1 \wedge \kappa_2 \geq \frac{1}{2}\,.
$$
We then have $\theta$ defined as $\theta=\theta_1 \wedge \theta_2$ with $\theta_j=\|E_{\cap,j}\|_2/\|E_{j}\|_2$ where $E_{\cap,j}$ is the projection of $E_{j}$ on $T_1 \cap T_2$.
But $E_2 \in T_1$ so that $\theta_2=1$. The situation is less simple for $E_1$.
Indeed, $E_{\cap,1}=\|a\|_1 a s_b\trans+\|b\|_1 s_a b\trans -a b\trans \|a\|_1\|b\|_1$.
Some calculations lead to
$$
\theta_1^2=\frac{\|a\|_1^2}{k}+\frac{\|b\|_1^2}{q}-\frac{\|a\|_1^2}{k}\frac{\|b\|_1^2}{q}\,,
$$
hence the definition of $\zeta(a,b)=\theta^2=\theta_1^2 \wedge \theta_2^2$.
 Theorem 3.2 in \citet{Oymak12} offers the possibility of constraining the estimator to lie in a cone $\mathcal{C}$. In our case, $\mathcal{C}=\RR^{m_1\times m_2}$, given the definition of $\gamma$ we therefore have $\gamma\leq 2$.
The result follows from applying the theorem with $\theta^2=\zeta(a,b)$ and using $\frac{\kappa}{81\gamma^2\tau} \geq \frac{1/2} {3^4.2^2.2}=\frac{1}{6^4}$.
\end{proof}

\begin{proof}[Proposition~\ref{prop:from_oymak}]\\
Take $M$ such that when $m_1, m_2, k, q, m1/k, m_2 / q \geq M$ then $n_0$ is large enough to ensure $1-c_1\exp(-c_2 n_0) > 4 \exp \br{-32/17}$. Then, according to Lemma~\ref{lem:oymak}, solving (\ref{eq:exact_rec}) with the norm $\Gamma_{\mu}$ fails to recover $A=ab\trans$ with probability at least $4 \exp \br{-32/17}$. On the other hand, \citet[Theorem 7.1]{Amelunxen13} shows that, when $n\geq \sdim\br{A,\Gamma_\mu} + \lambda$, for any $\lambda\geq 0$, then solving (\ref{eq:exact_rec}) with the norm $\Gamma_{\mu}$ correctly recovers $A$ with probability at least 
\begin{equation}\label{eq:proba}
4 \exp\br{\frac{-\lambda^2 / 8}{\omega^2(A,\Gamma_\mu) + \lambda}}\,,
\end{equation}
where $\omega^2(A,\Gamma_\mu) = \sdim\br{A,\Gamma_\mu} \wedge \br{m_1 m_2 - \sdim\br{A,\Gamma_\mu}}$. Take $\lambda = 16 \omega(A,\Gamma_\mu)$, then using the fact that $\omega(A,\Gamma_\mu) \geq 1$ we get that the probability (\ref{eq:proba}) is smaller than $4 \exp \br{-32/17}$. This implies that 
$$
n_0 \leq \sdim\br{A,\Gamma_\mu} + \lambda \leq \sdim\br{A,\Gamma_\mu} + 16 \sqrt{\sdim\br{A,\Gamma_\mu}} \leq 17\sdim\br{A,\Gamma_\mu} \,.
$$
\end{proof}

\section{Bounds on the statistical dimension in the vector case (proofs of results of Section~\ref{sec:vectorcase})}

\subsection{Lower bound on the statistical dimension of $\kappa_k$ (Proof of Proposition~\ref{prop:kappakStatDim})}\label{sec:vec_lower_bounds}

Let us start with two technical lemmata.
\begin{lemma}\label{lem:techF1}
Let $\Xk$ denote the $k$th order statistics of an i.i.d.\ sample $X_1,\ldots, X_n$ whose common distribution has a cdf $F$. Assume that $F^{-1}$ is a convex function\footnote{Note that this implies that the essential support of the random variable is bounded below.} from $[0,1]$ to $\overline{\RR}$. Then 
$$
\E[\Xk] \geq F^{-1}\Big (\frac{k}{n+1} \Big) \,.
$$
\end{lemma}
\begin{proof}[Lemma~\ref{lem:techF1}]\\
Let $f$ denote the pdf of $X$.
We have 
\begin{eqnarray*}
\E[\Xk]& = &\frac{n!}{(k-1)!(n-k)!} \,\int_{-\infty}^\infty  u \, F(u)^{k-1} \big (1-F(u)\big )^{n-r} f(u) \, du\\
& = &\frac{\Gamma(n+1)}{\Gamma(k)\,\Gamma(n-k+1)} \,\int_{0}^1  F^{-1}(v) \, v^{k-1} (1-v )^{n-r} dv=\E[F^{-1}(V)] \,,
\end{eqnarray*}
with $V \sim \text{Beta}(k,n-k+1)$.
Assuming that $F^{-1}$ is a convex function, we have by Jensen's inequality
$$\E[\Xk]=\E[F^{-1}(V)] \geq F^{-1}(\E[V])=F^{-1}\Big (\frac{k}{n+1} \Big)\,.$$
\end{proof}

\begin{lemma}\label{lem:techF2}
Let $G \in \RR^n$ be an standard normal vector, then we have
$$\E[\kappa_k^*(G)] \geq \sqrt{\frac{2}{\pi}} \sqrt{k \,\log \Big(\frac{n+1}{k+1}\Big)}\,.$$
\end{lemma}
\begin{proof}[Lemma~\ref{lem:techF2}]\\
Denote by $F$ the cdf of the absolute value of a standard normal variable. Then,
$$
F(x)=\Phi(x)-\Phi(-x)=\text{erf}\Big (\frac{x}{\sqrt{2}} \Big)\,,
$$
where $\Phi$ is the cdf of a standard Gaussian and $\text{erf}$ denotes the error function.
We use the following inequality due to \cite{chu1954bounds}:
$$
\sqrt{1-e^{-x^2}} \leq \text{erf}(x) \leq \sqrt{1-e^{-\frac{\pi}{4}x^2}}\,,
$$
to deduce that
$$
F^{-1}(y) \geq \sqrt{-\frac{2}{\pi} \log(1-y^2)}\,.
$$

By definition, we have $\E[\kappa_k^*(G)]=\frac{1}{\sqrt{k}}\E[X_{(n)}+\ldots+X_{(n-k+1)}]$ where $X_i=|G_i|$ and $G$ is a vector of independent standard normal variables. It can easily be checked that $F^{-1}$ is a convex function. This implies, using Lemma~\ref{lem:techF1}, that
\begin{eqnarray*}
\E[\kappa_k^*(G)] &\geq & \frac{1}{\sqrt{k}} \,\sum_{j=1}^k F^{-1}\Big (1-\frac{j}{n+1} \Big )\\
 & \geq & \sqrt{k}\, F^{-1}\Big (\frac{1}{k} \sum_{j=1}^k \Big(1-\frac{j}{n+1} \Big) \Big ) \qquad \qquad \text{(again by Jensen's inequality)}\\
 & = & \sqrt{k}\, F^{-1}\Big ( 1-\frac{k+1}{2(n+1)} \Big )\\
 & \geq & \sqrt{k}\, \sqrt{-\frac{2}{\pi}\log \Big(\frac{k+1}{(n+1)}- \Big(\frac{k+1}{2(n+1)} \Big )^2 \Big)}\\
 & \geq  & \sqrt{\frac{2}{\pi}} \sqrt{k \,\log \Big(\frac{n+1}{k+1}\Big)}\,.
\end{eqnarray*}
\end{proof}

\begin{proof}[Proposition~\ref{prop:kappakStatDim}]\\
We will denote the squared Gaussian width of the tangent cone intersected with a Euclidean unit ball by
$$
w(T_{\kappa_k}(a)\cap \mathbb{S}^{m-1})=\E \Big [\max_{t \in T_{\kappa_k}(a)\cap \mathbb{S}^{m-1}} \langle t, G\rangle \Big ]\,,
$$
where $G \in \RR^m$ denotes a standard Gaussian vector. 
We have $w(T_{\kappa_k}(a)\cap \mathbb{S}^{m-1})^2 \leq \sdim(a,\kappa_k)$\citep[Proposition 3.6]{Chandrasekaran12}.
We thus seek a lower bound of $w(T_{\kappa_k}(a)\cap \mathbb{S}^{m-1}).$
Since the tangent cone is polar to the normal cone, we have that 
$$
T_{\kappa_k}(a)=\cbr{t \in \RR^m \mid \langle s, t \rangle \leq 0, \: \forall s \in \partial \kappa_k(a)} \,.
$$
Given a random Gaussian vector $G$, denote $I_0$ the support of $a$ and $I_G$ the indices of the $k$ largest coefficients of $G$ in absolute value outside of $I_0$.
Denote by $\tilde{s}_G=\text{sign}(G_{I_G})$, \ie, the vector whose entries are zero outside of $I_{G}$ and equal to the sign of the corresponding coefficient of $G$ otherwise. 
Define $t_G=\frac{1}{\sqrt{2k}} (\tilde{s}_{G}-a)$. By construction $t_G \in \mathbb{S}^{m-1}$.
Let now consider $s \in \partial \kappa_k(a)$, we have
\begin{eqnarray*}
\sqrt{2k} \, \langle s, t_G \rangle & = &  - \langle s, a \rangle+ \langle s, \tilde{s}_{G}\rangle \leq -1+ \kappa_k(\tilde{s}_{G}) \, \kappa^*_k(s)\leq -1+1=0,
\end{eqnarray*}
so that $t_G \in T_{\kappa_k}(a)$.
Therefore $w(T_{\kappa_k}(a)\cap \mathbb{S}^{m-1}) \geq \E[\langle t_G, G\rangle]=\frac{1}{2\sqrt{k}} \E[\langle \tilde{s}_G, G\rangle]=\frac{1}{2} \E[\kappa_k^*(G)]$, whence the result using Lemma~\ref{lem:techF2} and $w(T_{\kappa_k}(a)\cap \mathbb{S}^{m-1})^2 \leq \sdim(a,\kappa_k)$.
\end{proof}

\subsection{Upper bound on the statistical dimension of $\theta_k$ (Proof of Proposition~\ref{prop:statisticaldimksupport})}
\label{app:k_support_norm}

\begin{proof}[Proposition \ref{prop:statisticaldimksupport}]

Without loss of generality, let us assume that $w \in \RR^p$ is a fixed vector having nonincreasing -- in absolute value -- coordinates, the first $s$ of which are assumed to be nonzero. We compute the subdifferential of $ \theta_k(w)$ directly by using (\ref{eq:ksuppsq}). Remember that one characterization of the subdifferential is 
$$
\partial \theta_k(w) = \cbr{\alpha \in\RR^p ~:~ \theta_k^*(\alpha) \leq 1 \,,~ \alpha\trans w = \theta_k(w)}\,.
$$
Letting $r \in \cbr{ 0, \cdots , k-1}$ being the unique integer such that $ | w_{k-r-1}| > \frac{1}{r+1} \sum_{i=k-r}^p |w_i| \geq | w_{k-r}|$, let us partition the set of entries $\{1 ,\cdots , p \}$ into $I_2 = \{1, \cdots , k-r-1\}$, $I_1 = \{ k-r, \cdots , s\}$ and $I_0 = \{s+1, \cdots , p\}$ (where each set may be empty). Then we can rewrite the expression of the $k$-support norm (\ref{eq:ksuppsq}) as
 \begin{equation*}
\theta_k(w)^2 =  \|w_{I_2}\|_2^2 + \frac{1}{r+1} \|w_{I_1}\|_1^2 \,.
\end{equation*}
Then necessarily each element $\alpha\in\partial \theta_k(w)$ must satisfy
\begin{equation*}
\begin{cases}
\alpha_i = \frac{w_i}{\theta_k(w)} & \text{for }i\in I_2\,,\\
\alpha_i = \frac{\nm{w_{I_1}}_1 \sign(w_i)}{(r+1)\theta_k(w)} & \text{for }i\in I_1\,.
\end{cases}
\end{equation*}
As for $i\in I_0$, the coefficients $\alpha_i$ do not impact $\alpha^\top w$ so they should also not impact $\theta_k^*(\alpha)$. If $s<k$ this implies $\alpha_i=0$, and if $s\geq k$ this means $|\alpha_i| \leq |\alpha_k|$, and in that case $k\in I_1$. With the convention $\nm{w_{I_1}}_1 = 0$ when $I_1 = \emptyset$, we finally get the following expression for the subdifferential:
\begin{equation}\label{eq:subdiffksupport} 
\partial \theta_k(w) = \frac{1}{ \theta_k(w)}\cbr{  w_{I_2}  + \frac{1}{r+1} \nm{w_{I_1}}_1  \br{ \sgn(w_{I_1}) + h_{I_0} }  ~:~\|h\|_\infty \leq 1 } \,.
\end{equation}

In the case $s<k$, we have $s=k-r+1$, $I_2=\sqb{1,s}$, $I_1=\emptyset$ and $I_0 = \sqb{s+1,p}$. In that case $\theta_k(w) = \nm{w}_2$ and $\partial \theta_k(w) = w/\nm{w}_2$, showing that $\theta_k$ is differentiable at $w$, meaning $\theta_k$ is useless to recover $w$.

Let us therefore only consider the case $s\geq k$, in which case $I_1 \neq \emptyset$ and $\nm{w_{I_1}}_1 >0$. In order to compute the statistical dimension of $\theta_k$ at $w$, we use the characterization (\ref{eq:statdimnormal}) 
$$
\sdim(w,\theta_k) = \E \sqb{\dist\br{ g , N_{\theta_k}(A)}^2}\,,
$$
where $g$ is a $p$-dimensional random vector with i.i.d. normal entries and $N_{\theta_k}(A)$ is the conic hull of $\partial \theta_k(w)$. We then get:
\begin{align}
\sdim(w,\theta_k)
&= \E \sqb{  \inf_{t>0~\&~u \in t \partial  \theta_k(w)} \|u-g\|_2^2 } \nonumber \\
& \leq \inf_{t>0} ~~\E \sqb{ \inf_{u \in t  \partial  \theta_k(w)} \|u-g\|_2^2 }  \nonumber \\
&  \leq \inf_{t>0}~~ \E  \inf_{h\in \RR^{p},\|h\|_\infty\leq 1} \bigg \{ \left  \|g_{I_2} - t \frac{(r+1)}{\nm{w_{I_1}}_1}w_{I_2}\right \|_2^2 + \nm{g_{I_1} - t \sgn(w_{I_1}) }_2^2 \nonumber \\
&\quad\quad\quad\quad +\left \|g_{I_0} -t h_{I_0}\right\|_2^2 \bigg \} \nonumber \\
& \leq  \inf_{t>0}  ~\bigg \{~ |I_2| + \frac{(r+1)^2 \nm{w_{I_2}}_2^2}{\nm{w_{I_1}}_1^2} t^2 + |I_1| (1+t^2)  + |I_0| \frac{2}{\sqrt{2 \pi}} \frac 1t  \exp \left (-\frac{t^2}{2}  \right ) \bigg \}   \label{eq:partial1}  \\
&= \inf_{t>0}  ~~\cbr{ s + t^2 \left \{  \frac{(r+1)^2 \nm{w_{I_2}}_2^2}{\nm{w_{I_1}}_1^2}  + |I_1| \right \} + (p-s) \frac{2}{\sqrt{2 \pi}} \frac 1t   \exp \left (-\frac{t^2}{2}  \right )    } \nonumber \\
& \leq \frac{5}{4}s + 2 \left \{  \frac{(r+1)^2 \nm{w_{I_2}}_2^2}{\nm{w_{I_1}}_1^2}  + |I_1| \right \} \log \frac{p}{s} \,, \label{eq:sdimboundannex}
\end{align}
where following  \citet[Annex C]{Chandrasekaran12}, for (\ref{eq:partial1}) we used the fact that for a standard normal random variable $G \sim \mathcal N(0,1)$
\[
\E_G \inf_{| \eta | \leq 1} \left (G -t \eta \right )^2   \leq \frac {2}{\sqrt{2 \pi}} \frac 1 t e^{-\frac{t^2}{2}} \,,
\]
while (\ref{eq:sdimboundannex}) is obtained by taking $b = \sqrt{2 \log(p/s)}$ and using $\frac{s \left ( 1-s/p \right )}{\sqrt{\pi \log ( p/s)}}\leq \frac 14$.

For the lasso case ($k=1$), we have $r=0$, $I_2=\emptyset$ and $I_1=\sqb{1,s}$. Plugging this into (\ref{eq:sdimboundannex}) we recover the standard bound \citep{Chandrasekaran12}:
\begin{equation}\label{eq:boundlassoannex}
\sdim(w,\theta_k) \leq \frac{5}{4}s + 2 s \log \frac{p}{s} \,.
\end{equation}
In the general case $1\leq k \leq s$ remember that, by definition of $r$,
$$
 | w_{k-r-1}| > \frac{\nm{w_{I_1}}_1}{r+1} \geq | w_{k-r}| \,,
$$
and therefore 
\begin{equation}\label{eq:partial2}
|I_2| \leq \frac{\nm{w_{I_2}}_2^2}{|w_{k-r-1}|^2} \leq \frac{(r+1)^2 \nm{w_{I_2}}_2^2}{\nm{w_{I_1}}_1^2} \leq \frac{\nm{w_{I_2}}_2^2}{|w_{k-r}|^2} \,.
\end{equation}
Plugging the left-hand inequality of (\ref{eq:partial2}) into (\ref{eq:sdimboundannex}) and remembering that $|I_2|+|I_1|=s$ shows that the bound (\ref{eq:sdimboundannex}) obtained for $\theta_k$, for any $1\leq k \leq s$, is never better than the bound (\ref{eq:boundlassoannex}) obtained for the lasso case $k=1$. In the case $s=k$, the right-hand inequality of (\ref{eq:partial2}) applied to an atom $w\in\cA_p^k$ with atom strength $\gamma = k|w_k|^2$ and unit $\ell_2$ norm leads to
$$
\frac{(r+1)^2 \nm{w_{I_2}}_2^2}{\nm{w_{I_1}}_1^2} + |I_1|\leq \frac{\nm{w_{I_2}}_2^2}{|w_{k-r}|^2} + |I_1| \leq \frac{\nm{w_{I_2}}_2^2}{|w_{k}|^2} + \frac{\nm{w_{I_1}}_2^2}{|w_{k}|^2}  = \frac{1}{|w_k|^2} = \frac{k}{\gamma}\,,
$$
from which we deduce by (\ref{eq:sdimboundannex}) the upper bound on the statistical dimension
$$
\forall w\in\cA_p^k\,,\quad \sdim(w,\theta_k) \leq  \frac{5}{4}k + \frac{2 k}{\gamma} \log\frac{p}{k} \,.
$$
\end{proof}

\end{document}